\documentclass[11pt]{article}

\usepackage{amsmath,amscd,amsbsy,amssymb,latexsym,url,bm,amsthm}
\usepackage{array}
\usepackage{booktabs}
\usepackage{etoolbox}
\preto\subequations{\ifhmode\unskip\fi} 
\usepackage[T1]{fontenc}
\usepackage{graphicx}
\usepackage{multirow,makecell}
\usepackage{mathtools}
\usepackage{subeqnarray}
\usepackage{tabularx}
\usepackage{threeparttable}
\usepackage{times}
\usepackage{tikz}
\usepackage{xcolor}
\usepackage{subcaption}
\allowdisplaybreaks
\usepackage{bbm}
\usepackage{appendix}
\newif\ifarxivVersion
\arxivVersiontrue  
\usepackage{dsfont}

\ifarxivVersion
    \usepackage{natbib}
    \usepackage[margin=1in]{geometry}
    \usepackage[colorlinks=false,allbordercolors={1 1 1}]{hyperref}
    \usepackage[capitalize]{cleveref}
     
    \newtheorem{theorem}{Theorem}
    \newtheorem{lemma}[theorem]{Lemma} 
    \newtheorem{proposition}[theorem]{Proposition} 
    
    \newtheorem{corollary}[theorem]{Corollary}
    \newtheorem{definition}[theorem]{Definition}
    
    \newenvironment{keywords}{\bgroup\leftskip 20pt\rightskip 20pt \small\noindent{\bf Keywords:} }%
    {\par\egroup\vskip 0.25ex}
\else
    \usepackage[preprint]{jmlr2e}
\fi

\newtheorem{assumption}{Assumption}

\crefname{assumption}{Assumption}{Assumptions}
\crefname{corollary}{Corollary}{Corollaries}
\crefname{theorem}{Theorem}{Theorems}
\crefname{lemma}{Lemma}{Lemmas}

\renewcommand{\bar}[1]{\mkern 1.5mu\overline{\mkern-1.5mu#1\mkern-1.5mu}\mkern 1.5mu}

\newcommand{\RNum}[1]{\uppercase\expandafter{\romannumeral #1\relax}}

\newcommand{\xindg}{\mathbf{o}^y_{w,s}}
\newcommand{\xindm}{\mathbf{o}^y_{M,s}}
\newcommand{\xindi}{\mathbf{o}^y_{I,s}}

\newcommand{\xrisk}{\mathbf{r}^y_{s}}
\newcommand{\xindmi}{\mathbf{f}^y_{s}}

\usepackage{threeparttable}
\usepackage{hhline}

\newcolumntype{Y}{>{\centering\arraybackslash}X}

\makeatletter
\def\raisedotfill{%
  \leavevmode
  \cleaders \hb@xt@ .44em{\hss\raise0.5ex\hbox{.}\hss}\hfill
  \kern\z@}
\makeatother

\ifarxivVersion
\else

\firstpageno{1}
\fi

\begin{document}
    \ifarxivVersion
    \title{{\textbf{Anticipating Gaming to Incentivize Improvement: Guiding Agents in (Fair) Strategic Classification}}}
    \author{Sura Alhanouti\textsuperscript{1}\textsuperscript{3}, Parinaz Naghizadeh\textsuperscript{2}\\           
        \date{}
        }
    \fi

    \maketitle
    \ifarxivVersion
        
        \def\thefootnote{}\footnotetext{$^1$Department of Integrated Systems Engineering, The Ohio State University, alhanouti.1@osu.edu. $^2$Department of Electrical and Computer Engineering, University of California, San Diego, parinaz@ucsd.edu. $^3$Department of Industrial Engineering, Jordan University of Science \& Technology.}
        \setcounter{footnote}{0}
        \def\thefootnote{\arabic{footnote}}
    \fi

\begin{abstract}
As machine learning algorithms increasingly influence critical decision making in different application areas, understanding human strategic behavior in response to these systems becomes vital. We explore individuals' choice between genuinely improving their qualifications (``improvement'') vs. attempting to deceive the algorithm by manipulating their features (``manipulation'') in response to an algorithmic decision system. We further investigate an algorithm designer's ability to shape these strategic responses, and its fairness implications. Specifically, we formulate these interactions as a Stackelberg game, where a firm deploys a (fair) classifier, and individuals strategically respond. Our model incorporates both different costs and stochastic efficacy for manipulation and improvement. The analysis reveals different potential classes of agent responses, and characterizes optimal classifiers accordingly. Based on these, we identify when and why a (fair) strategic policy can not only prevent manipulation, but also incentivize agents to opt for improvement. Our findings shed light on the intertwined nature of microeconomic and ethical implications of firms' anticipation of strategic behavior when employing ML-driven decision systems. 
\end{abstract}


\section{Introduction}\label{sec:intro}

Machine learning (ML) algorithms have come to play a pivotal role in guiding decision making in many application areas, including banking, hiring, social media, and resource allocation. While the use of ML-driven systems can enhance efficiency, it can also drive the humans who are subject to algorithmic decisions to adjust their behavior accordingly. For instance, candidates in school admissions, job recruitment, and research pools tend to systematically shift their self-presentation to appear more analytical when facing algorithmic assessments \citep{goergen2025ai}. In other algorithmic contexts, 
Uber drivers coordinate their behavior to trigger surge pricing \citep{andon2023uber,mohlmann2017hands}, and users change their posting and engagement behavior to influence content curation \citep{cloudresearch2024user,eslami2016first}. Such \emph{strategic} responses motivate a game-theoretic analysis of learning algorithms with humans in the loop.

Existing work on strategic classification have primarily focused on agents who \emph{manipulate} observable features to obtain a favorable prediction without changing their true qualification (e.g., opening credit lines, cosmetically editing resumes, or cheating to obtain better test scores).  
More recent works (e.g., \citep{efthymiou2025incentivizing,xie2024learning,tsirtsis2024optimal,wang2023algorithmic,horowitz2023causal}) have further incorporated the possibility of \emph{improvement}: individuals can invest effort to genuinely improve their qualification state, by, e.g., studying for exams or building financial stability. From the perspective of a decision-making institution, improvement is strictly preferable to manipulation, raising the question of when and how a classifier can incentivize improvement rather than gaming by its strategic users.

This paper examines such incentives in a strategic classification setting in which agents may choose between manipulation and improvement. We allow these actions to differ both in \emph{cost} and in the \emph{stochastic efficacy} of feature changes; the later is an aspect that, unlike prior deterministic treatments of strategic responses, plays a central role in our analysis (see Section~\ref{sec:related_work}). We formulate the interaction as a Stackelberg game: the firm selects a (possibly fairness-constrained) classifier, and agents best-respond strategically. In Proposition~\ref{prop:agents-br-generic}, we show that, under mild assumptions (improvement is costlier but more beneficial in expectation), equilibrium behavior depends on how the firm’s threshold interacts with agents’ cost and boost distributions. Proposition~\ref{prop:firm_impact_comp} then uses these findings to characterize how a strategic firm adjusts its policy relative to a non-strategic one. The main technical challenge is that the regions in which agents choose manipulation or improvement depend endogenously on the firm's decision variable, requiring the firm’s objective to be evaluated over moving integration limits; we address this using the Leibniz integral rule (with modifications for fairness constraints) and the linkage between thresholds and the agents' indifference points to obtain tractable optimality conditions.

Building on this characterization, we find that a strategic firm not only curbs manipulation (as also found in prior work) but also adopts stricter selection thresholds that encourage all agents, whether initially qualified or unqualified, to opt for improvement decisions instead. Specifically, incentivizing improvement by unqualified agents (driving them to improve both their qualification states and observable features), while still leaving the manipulation option open to qualified agents (who do not have sufficiently high observable features to be selected otherwise, but whose acceptance benefits the firm). Together, these findings lead to our first key takeaway: \emph{anticipating agents' strategic responses, which the firm does solely with the goal of optimizing its utility, has the added benefit of guiding agents' behavior in a desirable direction.}

Then, in Section~\ref{sec:fairness_effect}, we study the effects of fairness interventions when both manipulation and improvement are available. We argue that, in attempting to attain fair outcomes, a firm may inadvertently \emph{decrease} improvement incentives for disadvantaged groups. This negative by-product arises even when group disparities are purely \emph{historical} (in qualification rates or feature distributions) rather than \emph{current} differences in access to strategic resources. Our numerical results in Section~\ref{sec:numerical-exp} further show that fairness interventions, when imposed by a non-strategic firm, can unintentionally reduce improvement incentives for disadvantaged groups despite equalizing decision rates. In contrast, a strategic firm can satisfy the same fairness criteria while preserving (and even amplifying) improvement incentives, by jointly adjusting thresholds and leveraging endogenous behavioral responses. These findings point out a third lever for fairness beyond inclusiveness (flipping disadvantaged group predictions to positive) and selectivity (flipping advantaged group predictions to negative) \citep{estornell2023group,keswani2023addressing,hu2019disparate,liu2018delayed,zafar2017fairness}: \emph{anticipating strategic responses} to induce improvement and in turn improving fairness ``for free''. Our findings are supported (and expanded) by numerical experiments and illustrations in Section~\ref{sec:numerical-exp}, conducted on semi-synthetic data constructed from real-world FICO credit score statistics.

To summarize, our findings offer two new insights for algorithm designers. First, anticipating the users' strategic responses to algorithmic systems (which are increasingly observed, as noted in the examples in the beginning of this section) can not only help protect the planned operating points of these systems, but also will ultimately help enhance the system by increasing the number of agents (workers, employees, students, etc.) who select the preferred improvement actions over undesired manipulation actions, in turn leading to an increase in users' qualification/success levels. Second, as developing fair and responsible algorithmic decision systems is becoming of increased interest due to reputation effects and applicable laws, driving agents' improvement actions can be viewed as an additional lever in meeting these fairness goals, allowing firms to attain fairness desiderata without having to lower accuracy significantly. 
\section{Related Work}\label{sec:related_work}
Strategic machine learning has received growing attention in settings where agents alter their behavior in order to obtain more favorable algorithmic outcomes. Most existing work falls into two main categories: (i) models in which agents \emph{manipulate} observable features without changing their underlying qualification \citep[e.g.,][]{cohen2025bayesian, ebrahimi2025double, cohen2024learnability, ahmadi2024strategic, sundaram2023pac,  lechner2023strategic, ahmadi2023fundamental, shao2023strategic,zhang2022fairness, levanon_strategic_2021, ahmadi2021strategic}, and (ii) models where agents may also \emph{improve} their true qualification state through costly investment \citep[e.g.,][]{efthymiou2025incentivizing,xie2024algorithmic,jin2023collab, wang2023algorithmic, horowitz2023causal,ahmadi2022classification,jin_incentive_2022,haghtalab_maximizing_2020, bechavod2022information, zhang2020fair, kleinberg2020classifiers, bechavod2021gaming, harris2021stateful, shavit2020causal,alon2020multiagent}. A growing literature also considers fairness implications in these settings, though primarily under manipulation-only behavior \citep{milli2019social, hu2019disparate, zhang2022fairness, estornell2023group}.
Our work fits within this second category, but departs from prior approaches in two key respects. First, unlike causal or feature-based formulations that distinguish manipulation from improvement through structural relations among features \citep{chen2023learning,horowitz2023causal, wang2023algorithmic,shavit2020causal, bechavod2021gaming,miller_casual_2020}, we explicitly model \emph{heterogeneous costs and stochastic efficacy} for the two action types. While \citet{emelianov2022fairness} also introduce stochastic action effects via normally distributed boosts, we impose no parametric restriction on the boost distribution. This modeling choice enables us to analyze not only whether agents choose to improve, but when improvement is actually \emph{incentivized} by the economic environment.

Our work is also closely related to \citep{xie2024learning}, which model manipulation and improvement within a Stackelberg framework, but assume that only \emph{unqualified} agents behave strategically by imitating qualified agents. In contrast, we allow \emph{all agents} to behave strategically without imitation. (We also differ in some other model elements.) Notably, we show that the effects of anticipating strategic behavior differs across qualified and unqualified individuals. In addition, \citet{xie2024learning} also study how anticipating strategic responses affects fairness across demographic groups and, like us, find that fairness coincides with incentivizing improvement and discouraging manipulation. While we consider similar environments and reach aligned high-level conclusions, we differ in how fairness is modeled. Their analysis evaluates fairness through outcome disparities under a given (potentially unfair) policy, whereas we explicitly model fair decision-making as a constrained optimization problem, deriving policies that satisfy fairness constraints by construction and enable analyses of different fairness interventions that are not accessible under outcome-based fairness evaluation alone.

Finally, we note that there are other forms and dimensions of strategic responses considered in the literature, but that are outside our scope; these include strategic coordination across agents or firms \citep{narang_2022}, strategic hiding \citep{zhang2020predictive}, behavioral biases \citep{ebrahimi2025double}, and randomized policies \citep{geary2025strategic,sundaram2023pac,braverman2020role}.  

An early version of this work appeared in \citep{alhanouti2024could}; the present paper provides a full characterization of the firm's optimal response (Proposition~\ref{prop:firm_impact_comp}), the fairness–incentive tradeoffs it induces (Corollary~\ref{cor:fair-policies}), and extensive empirical validation on semi-synthetic dataset based on real-word FICO credit data (Section~\ref{sec:numerical-exp}). 
\section{Problem Setting and Preliminaries}\label{sec:model}

We consider a Stackelberg game where the firm (algorithm designer) first announces a classifier, following which the agents respond strategically. In this section, we detail the agents' and the firm's characteristics, their actions, and their utilities. Our model is similar to, and extends, those of prior work on strategic classification (e.g., \citet{zhang2020fair,zhang2022fairness}). 

\subsection{The agents}
Consider a population of agents with two types of observable attributes: \emph{a feature/score} $x\in\mathbb{R}_{\geq 0}$ (e.g., a credit score, an exam score) and a \emph{sensitive feature} that divide the population into two demographic groups $S\in \{a,b\}$ (e.g., sex dividing population into Male and Female, or race dividing the population into White and Non-White agents). Let $n_s$ be the fraction of agents in group $s$. 

Each agent has a true (hidden) \emph{label} or \emph{qualification state}, denoted $Y \in \{0,1\}$, with $Y=1$ denoting a qualified agent (e.g., one who can pay off a loan if granted) and $Y=0$ denoting an unqualified agent. Let $\alpha_s:=\mathbb{P}(Y=1| S=s)$ denote the qualification rate in group $s$. In addition, let $G^y_s(x):=\mathbb{P}(X=x| Y=y, S=s)$ denote the probability density function (pdf) of the distribution of the features for individuals with qualification state $y$ from group $s$. We make the following (common) assumption on the feature distributions.

\begin{assumption}\label{as:strict-monotonicity-assumption} 
The feature distributions $G^y_s(x)$ are continuous, and satisfy the strict monotone likelihood ratio property: $\frac{G^1_s(x)}{G^0_s(x)}$ is strictly increasing in $x \in \mathbb{R}_{\geq 0}$. 
\end{assumption}
In words, this entails that as an agent's feature $x$ increases, the likelihood that the agent is qualified increases. 

\subsection{The firm} A firm makes decisions on agents based on their observable features $x$ and their group memberships $s$. The decision is binary, denoted $d \in \{0,1\}$, where $d=1$ represents acceptance and $d=0$ represents rejection. This decision is determined by a group-dependent binary classifier $\pi_s(x) = \mathbb{P}(D=1|X=x, S=s)$. We assume that this is a threshold policy, such that an agent is accepted if and only if  ${x} \geq \theta_s$. \citet{zhang2020fair} show that such threshold policies are optimal under mild assumptions.

\subsection{Agents' strategic actions}
After the threshold policy $\theta_s$ is announced by the firm, agents have the option to behave strategically to improve their chances of receiving favorable outcomes. This is done by choosing one of the strategic actions $w \in \{M, I, N\}$, with $M$ denoting \emph{manipulation}, $I$ denoting \emph{improvement}, and $N$ denoting \emph{doing nothing}. 

Taking these actions may impact the agents' features and/or true labels. We use $X$ and $Y$ to denote the \emph{pre-strategic} feature and true label (before an action $w$ is taken), and $\hat{X}$ and $\hat{Y}$ to denote the post-strategic ones. 
In particular, both qualified and unqualified agents ($y=1$ and $y=0$) can opt to manipulate ($w=M$) by changing their feature $x$ to some new feature $\hat{x}$, without changing their true qualification state ($\hat{y}=y$); e.g., this could represent cheating on an exam, which increases the score without affecting true learning.  Alternatively, both types of agents can choose to invest in improving themselves ($w=I$), which leads to a change in both their feature from $x$ to some $\hat{x}$, as well as their true label to $\hat{y}=1$; e.g., this could represent studying for an exam, which improves both the score and true learning. Lastly, agents who opt to do nothing ($w=N$) maintain their feature and label ($\hat{x}=x$ and $\hat{y}=y$). The firm will only observe the agents' \emph{post-strategic} features $\hat{x}$ when making its decision. We let $\hat{\alpha}_s$ and $\hat{G}_s^y(x)$ denote the {post-strategic} population statistics after these changes have happened as a result of agents' selected strategic actions.  (We analytically characterize these updated statistics as a function of their pre-strategic counterparts in Section~\ref{sec:post-strategic-stats}).

In addition to differing in their impacts on changes in features and/or true labels, the actions differ in two aspects: their cost and their efficacy. First, each action requires exerting effort and comes at a certain group-dependent (constant) cost ${C_{w,s}}\in[0,1)$, $w\in\{M, I, N\}, s\in\{a,b\}$. Such difference in costs across actions and groups is a common assumption in prior works on strategic classification (e.g., \cite{zhang2022fairness, liu2020disparate, jin_incentive_2022}). The element that is new to our model is that we assume all actions lead to a (weak) increase in the feature (i.e., we assume that ${\hat{x}\ge x}$), but that this increase in feature is \emph{stochastic} and different across actions and groups.  
Formally, we assume the probability that the feature $x$ of a qualified/unqualified individual from a group $s$ increases by $b_w$ after opting for action $w$ is distributed according to a \emph{boost distribution} with pdf ${\mathbb{\tau}^y_{w,s}(b)}:=\mathbb{P}(\hat{X}=x+b|X=x, Y=y, W=w, S=s)$. These boost distributions determine the \emph{efficacy} of the action. As an example, they could capture that both cheating and studying can increase an agent's exam score, but that the studying is more effective for increasing scores. Let $\{\underline{b}^y_{w,s}, \bar{b}^y_{w,s}\}$ denote the minimum and maximum boost under action $w$, and $\mathbb{T}^y_{w,s}(b)$ denote the CDF of the boost function. 

We let $C_{N,s}=0$ and $\mathbb{\tau}^y_{N,s}(0)=1$, meaning that the ``do nothing'' action has zero cost and no impact on changing the agent feature. The two remaining actions, $M$ and $I$, can differ in cost and efficacy. We conduct our analysis under the following assumption. 

\begin{assumption}\label{as:cost-efficacy} Improvement is more costly than manipulation (i.e., $C_{I,s}\geq C_{M,s}$), and the improvement boost distribution first-order stochastically dominates (FOSD) that of manipulation (i.e., $\mathbb{T}^y_{M,s}(b)\geq \mathbb{T}^y_{I,s}(b), \forall b$).
\end{assumption}

In words, an action first-order stochastically dominating another means that it ``gets the agent further'', in that it has a higher probability of increasing the agent's feature from $x$ to a feature greater or equal to $\hat{x}=x+b$. Assumption~\ref{as:cost-efficacy} gives rise to the conflict between the two actions: manipulation is cheaper, while improvement is more effective in advancing the agent. (We note that there are three scenarios beyond Assumption~\ref{as:cost-efficacy}: one can be recovered by interchanging the indexes $M$ and $I$ in our results. The remaining two cases are less interesting, as one of the actions would trivially dominate the other.) 

\subsection{Agents' utility} 
In general, each agent's strategic choice among the three available actions (manipulation, improvement, or doing nothing) depends on its budget $B$, the cost of effort $C_{w,s}$, the efficacy of the selected action $\mathbb{\tau}^y_{w,s}$, and the firm's deployed policy $\theta_s$. For simplicity, we assume the same budget $B$ for all agents, and assume this budget is sufficiently high so that all agents can afford any action. Then, the choice among the actions depends on the relative benefit vs. the cost of each action. 

Formally, an agent chooses to incur the cost $C_{w,s}$ of action $w$ if and only if it (sufficiently) increases the probability of being accepted by the firm. For an agent from group $s$ with feature $x$ and true label $y$, the benefit from strategic action $w$ is the increase in the probability of being accepted, given by
\begin{align}
    \mathcal{B}_{w,s}(x,y)&:=\mathbb{P}(D=1|X=x,Y=y,W=w,S=s)-\mathbb{P}(D=1|X=x,Y=y,W=N,S=s).
    \label{eq:agent-benefit}
\end{align} 
Note that the efficacy of the selected action $\mathbb{\tau}^y_{w,s}$, and the firm's deployed policy $\theta_s$, will affect this benefit. The agent's overall utility is given by $u_s(x,y,w):=\mathcal{B}_{w,s}(x,y) - C_{w,s}$. 
 
\subsection{The firm's utility} The firm's goal is to find an optimal policy that maximizes its expected utility by correctly classifying the agents. The firm receives benefit $u_+$ from accepting qualified individuals, and incurs penalty $u_-$ from accepting unqualified individuals. The firm's goal may alternatively be finding the \emph{fair} optimal policy by also imposing a fairness constraint $\mathcal{C}^f$ on its decision problem. 
 
Formally, let $U(\theta_a, \theta_b)$ represent the firm's total utility given the decision thresholds $\theta_a$ and $\theta_b$ for each group. Then, the firm's expected utility is: 
 \begin{align}
    \mathbb{E}[U(\theta_a, \theta_b)] &= \sum_{s\in\{a,b\}} n_s [\mathbb{P}(D=1,Y=1|S=s) u_+ - \mathbb{P}(D=1,Y=0|S=s) u_-] \notag\\  
    &=\sum_{s\in\{a,b\}} n_s \int_{\theta_s}^\infty   [ G^1_s(x)\alpha^{}_s u_{+} - G^0_s(x)(1-\alpha^{}_s) u_- ] \mathrm{d}x~.
     \label{eq:expected-utility}
 \end{align}
 
 The firm's (fair) optimization problem is to choose the decision thresholds $\theta_a$ and $\theta_b$ as follows: 
 \begin{align}
     \max_{\theta_{a}, \theta_{b}}  \quad \mathbb{E}[U(\theta_a, \theta_b)]~, \qquad
    \textbf{s.t.} \quad \mathcal{C}^{f}_{a}(\theta_{a})=\mathcal{C}^{f}_{b}(\theta_{b}).
    \label{eq:fair-expected-utility}
 \end{align}

The fairness constraint here, donated by $f$, will be considered to be one of the two more commonly studied notions of fairness \citep{mehrabi2021survey}: 
\begin{itemize}
    \item Equality of Opportunity (EOP): 
    $\mathcal{C}^{EOP}_{s} = \int_{\theta_s}^{\infty} {G}_s^{1}(x)\mathrm{d}x$. This will equalize the true positive rates between groups $a$ and $b$. 
    \item Demographic Parity (DP): $\mathcal{C}^{DP}_{s} = \int_{\theta_s}^{\infty} ({G}_s^{1}(x) \alpha_s + {G}_s^{0}(x) (1-\alpha_s) )\mathrm{d}x$. This will equalize selection (acceptance) rates between groups $a$ and $b$. 
\end{itemize}

When agents are strategic and this is known to the firm, the agents' statistics in \eqref{eq:expected-utility} will be replaced by the post-strategic values $\hat{\alpha}_s$ and $\hat{G}_s^y(x)$, as characterized shortly in Section~\ref{sec:agents-behavior}. 

\section{Agents' Strategic Behavior}\label{sec:agents-behavior}

We begin by characterizing agents' best responses to a firm's given classifier. 
We show that for any given qualification state $y$, group $s$, and corresponding threshold $\theta_s$, the feature space $x$ can be partitioned into disjoint regions determining the agents' best-response, with the boundaries of these regions determined by the points where agents become indifferent between pairs of actions. 

Specifically, we first define a set of \emph{indifference features} based on the cost-efficacy trade-off of the actions available to the agents. 

\begin{definition}\label{def:indiff-features}
Given efficacy CDFs $\mathbb{T}^y_{w,s}$, with inverse CDFs $(\mathbb{T}^y_{w,s})^{-1}$, and costs $C_{w,s}$, define:
\begin{itemize}
    \item \textbf{Opt in features:} $\xindg = \max\{0, \theta_s- (\mathbb{T}^y_{w,s})^{-1}(1-C_{w,s})\}$, for $w\in\{M, I\}$. 
    \item \textbf{Flip decision feature:} $\xindmi\in [\theta-\bar{b}_M, \theta-\underline{b}_I]$ satisfying $\mathbb{T}^y_{M,s}(\theta_s-\xindmi) - \mathbb{T}^y_{I,s}(\theta_s-\xindmi) = C_{I,s}-C_{M,s}$,
    if a solution exists; zero otherwise.
    \item \textbf{Risk taker feature:} $\xrisk := \max\{0, \theta_s-(\mathbb{T}^y_{M,s})^{-1}(C_{I,s}-C_{M,s})\}$. 
\end{itemize}
\end{definition}
Intuitively, and as shown formally in the proof of Proposition~\ref{prop:agents-br-generic}, these features can be interpreted as follows: The \emph{opt in features} determine the first feature at which the agents benefit from opting for an action $M$ or $I$ as opposed to doing nothing $N$; the \emph{flip decision feature} is the feature at which the cost-efficacy trade-off of the $M$ and $I$ actions are equalized so that agents choices will flip between these decisions at this feature; and the \emph{risk taker feature} is the first feature at which an agent opts for an uncertain shot at getting a favorable classification by choosing $M$ over the certain admission attainable from $I$, given $M$'s relatively lower cost. 

Using the above, we characterize the agents' optimal best responses to a given decision threshold $\theta_s$. 

\begin{proposition}\label{prop:agents-br-generic} 
If $\xindmi$ is unique,
the agents' optimal response $w^*_{s}(x,y)=\arg\max_{w\in\{M,I,N\}} u_s(x,y,w)$ to a given threshold $\theta_s$ will be one of the three types outlined in Table~\ref{table:prop-agent-br-generic}. 
\begin{table}[thpb]
    \begin{threeparttable}
        \caption{Agents' best responses (Proposition \ref{prop:agents-br-generic}).} 
        \label{table:prop-agent-br-generic}
        \centering
        \begin{tabular}{||p{1.5cm}|p{4 cm}|p{9 cm}||}
            \hline
            \textbf{Type} & \textbf{Condition} & \textbf{Range : $w_s^*(x,y)$} \\
            \hline
            \textbf{Type 1} & $\xindmi \leq \xindi \leq \xindm \leq \xrisk$ \tnote{*} 
            & $[0, \xindi): N$, $[\xindi, \xrisk):I$, 
                $[\xrisk, \theta_s) : M$, $[\theta_s, \infty): N$ \tnote{\dag} \\
            \hline
            \textbf{Type 2} & $\xindm\leq\xindi\leq\xindmi$ &   $[0, \xindm): N$, $[\xindm, \xindmi): M$, $[\xindmi, \xrisk): I$, $[\xrisk, \theta_s): M$, $[\theta_s, \infty): N$ \\ \hline
            \textbf{Type 3} & $\xindmi\leq \xindm \leq \xindi$ & 
                $[0, \xindm): N$, $[\xindm, \theta_s): M$, $[\theta_s, \infty): N$ \\
            \hline
        \end{tabular}
        \begin{tablenotes}
            \item[*] Or $\xindi \leq \xindmi \leq \xindm \leq \xrisk$.
            \item[\dag] $[0, \xindi): N$, $[\xindi, \xindmi):I$, $[\xindmi, \theta_s):M$, $[\theta_s, \infty): N$. 
        \end{tablenotes}
    \end{threeparttable}
\end{table}
\end{proposition}

A detailed proof is given in Appendix~\ref{app:proof-agent-best-response}. These three types of best-responses are illustrated in Figure~\ref{fig:prop_1}. In particular, note that in all types of equilibrium, agents who are close to the decision threshold opt to be risk takers, choosing uncertain but cheap manipulation over certain but costly improvement. Interestingly, we posit that this may present a parallel with gaming behavior observed in education settings: students committing academic dishonesty typically have higher GPAs \citep{oedb-cheaters-gpa}.

\begin{figure*}[thpb!]
  \centering
  \begin{subfigure}[b]{0.3\textwidth}
    \centering
    \includegraphics[width=\textwidth]{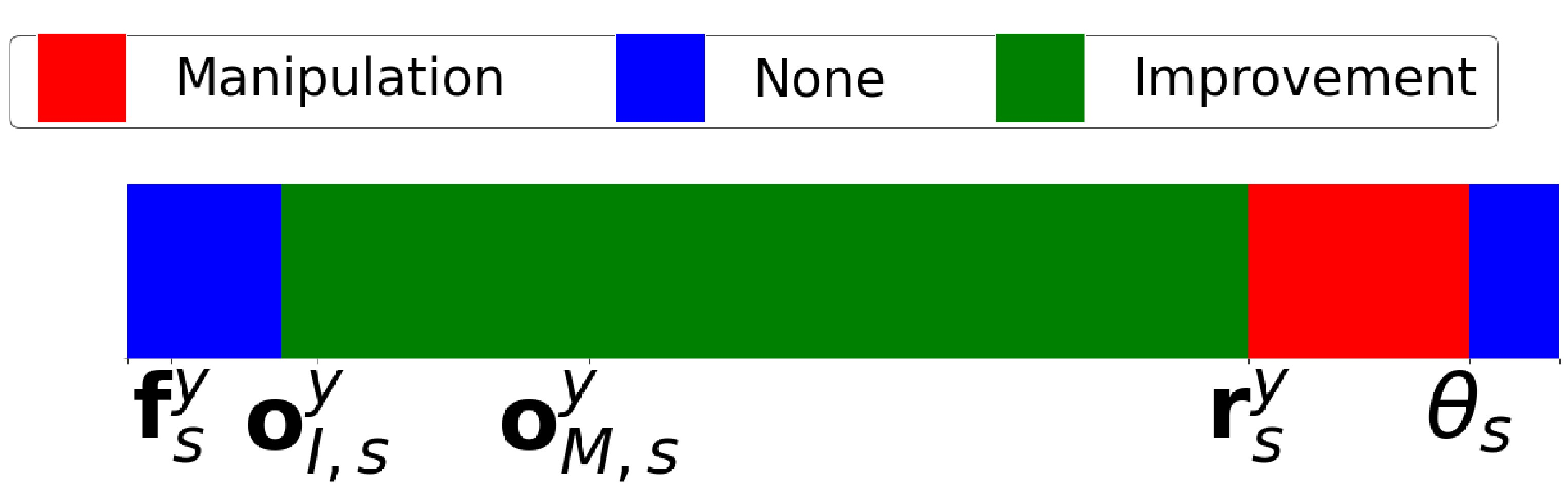}
    \caption{Type 1.}
    \label{fig:type1-general}
  \end{subfigure}
  \hspace{0.2in}
  \begin{subfigure}[b]{0.3\textwidth}
    \centering
    \includegraphics[width=\textwidth]{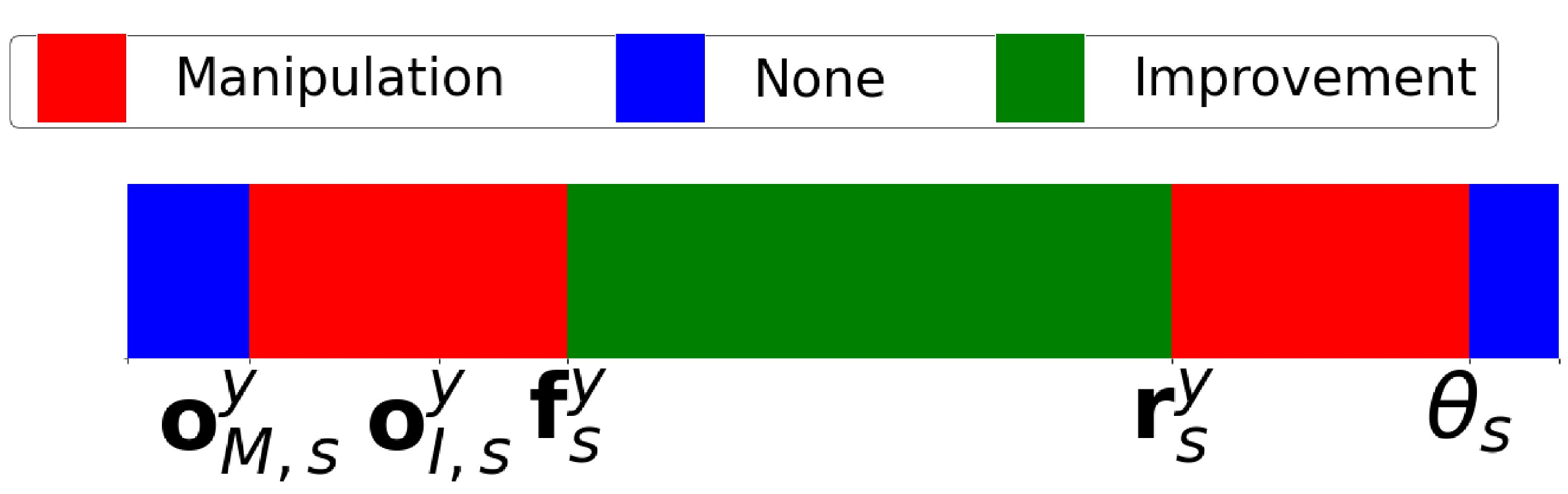}
    \caption{Type 2.}
    \label{fig:type2-general}
  \end{subfigure}
  \hspace{0.2in}
  \begin{subfigure}[b]{0.3\textwidth}
    \centering
    \includegraphics[width=\textwidth]{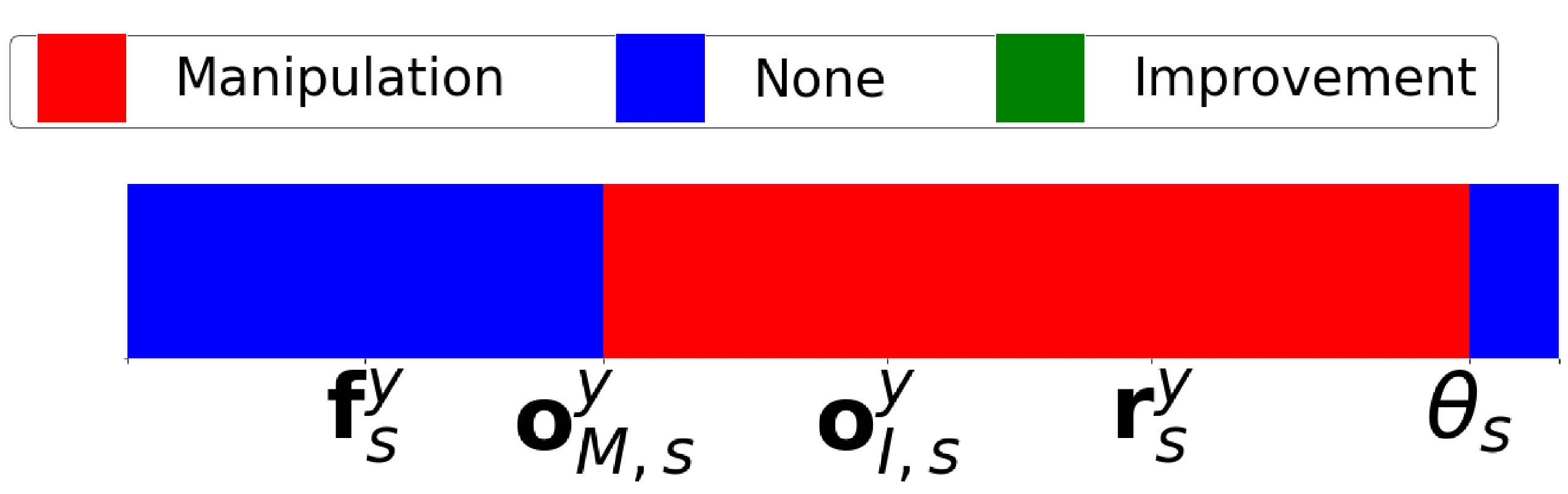}
    \caption{Type 3.}
    \label{fig:type3-general}
  \end{subfigure}
  \caption{Agent best-responses identified in Proposition~\ref{prop:agents-br-generic}.}
  \vspace{-0.1in}
  \label{fig:prop_1}
\end{figure*}

We also note that the type of equilibrium is determined solely by the relative cost and efficacy of the manipulation and improvement actions; the choice of the classifier $\theta_s$ changes the indifference points where agents opt for each action, but not the \emph{type} of equilibrium. 
Specifically, when the improvement cost is \emph{low}, the best response is of \emph{Type 1}, as agents find it beneficial to opt for improvement (higher efficacy, and relatively low cost) before manipulation becomes beneficial (which is near the threshold where risk of negative outcome is low enough). Moreover, when the improvement cost is \emph{Moderate}, the best response is of \emph{Type 2}, as now improvement is too expensive for agents who are far from the threshold, while still low enough that agents opt for improvement over manipulation when both actions are uncertain. Lastly, when the improvement cost is \emph{high}, the best response is of \emph{Type 3}, as now all agents choose the manipulation (once beneficial) over improvement, due to the considerably high improvement cost.
We illustrate these intuitive interpretations more formally using uniform boost distributions in Appendix~\ref{sec:best-response-uniform}. 

\subsection{Post-strategic population statistics}\label{sec:post-strategic-stats}

We next characterize the \emph{post-strategic} population statistics. In particular, as noted earlier, the strategic actions $w \in \{M,I\}$ change the agent's true qualification and/or feature. Denote the post-strategic qualification rates by $\hat{\alpha}_s=\mathbb{P}(\hat{Y}=1| S=s)$, and the post-strategic feature distribution by ${\hat{G}^{y}_s(x)}:=\mathbb{P}(\hat{X}={x}| \hat{Y}={y}, S=s)$. Let $\mathbb{I}^y_{s, (i)}:=\{x| Y=y, S=s, w=I\}$ be the set of agents from group $s$ with label $y$ who opt for improvement under a given threshold policy $\theta_s$, in a type \emph{(i)} best response (as identified in Proposition \ref{prop:agents-br-generic}). Define $\mathbb{M}^y_{s, (i)}$ and $\mathbb{N}^y_{s, (i)}$ similarly for the set of agents who opt for manipulation and for doing nothing under equilibrium type \emph{(i)}. 
The following lemma presents the expressions for the post-strategic population statistics in terms of the pre-strategic population statistics and the best-response regions. 
 
\begin{lemma}\label{lemma:post-strategic-stats}
    The post-strategic qualification rate is given by
    \begin{align}
    \hat{\alpha}_s = \alpha_s + (1-\alpha_s) \int_{x\in\mathbb{I}^0_{s, (i)}} G^0_s(x)\mathrm{d}x
        \label{eq:post-strategic-alpha}
    \end{align}
    The post-strategic feature distributions are given by:
\begin{align}
\hat{G}^0_s(x) &= \frac{1-\alpha_s}{1-\hat{\alpha}_s} \Big(\mathds{1}(x\in\mathbb{N}^0_{s, (i)})G^0_s(x) + (G^0_s*\tau^0_{M,s})(x)\Big),\label{eq:post-strategic-G-0}\\
    \hat{G}^1_s(x) &= \frac{\alpha_s}{\hat{\alpha}_s} \Big(\mathds{1}(x\in\mathbb{N}^1_{s, (i)})G^1_s(x) + (G^1_s*\tau^1_{M,s})(x)+ (G^1_s*\tau^1_{I,s})(x)\Big)+ \frac{1-\alpha_s}{\hat{\alpha}_s}(G^0_s*\tau^0_{I,s})(x),
            \label{eq:post-strategic-G-1}
    \end{align}
where $(G^y_s*\tau^y_{w,s})(x):= \int_{z\in \mathbb{W}^y_{s, (i)}}G^y_s(z)\tau^y_{w,s}(x-z)\mathrm{d}z$ is the convolution of the feature distribution and boost function restricted to the action $w$ region. 
\end{lemma}

A detailed proof appears in Appendix~\ref{app:proof-post-strategic}. Intuitively, \eqref{eq:post-strategic-alpha} captures the increase in qualification rates due to improvement actions taken by previously unqualified individuals (those in $\mathbb{I}^0_s$). The new feature distribution of unqualified individuals in \eqref{eq:post-strategic-G-0} consists of unqualified individuals who do nothing, and unqualified individuals who choose manipulation actions (with their features shifted by the boost distribution). The new feature distribution of qualified individuals in \eqref{eq:post-strategic-G-1}, on the other hand, consists of all qualified individuals (with appropriate feature shifts for those who opt for manipulation or improvement), as well as previously unqualified individuals who are now qualified due to opting for improvement actions. 	
\section{The Firm's Optimal Policy}\label{sec:firm}

Using the findings of Section~\ref{sec:agents-behavior}, we can now determine the firm's optimal choice of decision thresholds.  

We start with a firm who does not account for agents' strategic behavior and does not implement any fairness constraints; we refer to this as the \emph{unfair non-strategic} firm. The lemma below characterizes such firm's optimal decision thresholds as a function of the population statistics. 

\begin{lemma}\label{lemma:nonstrategic-unfair-theta}
    The unfair non-strategic firm's optimal decision thresholds $\theta_s^{U} $ satisfies $\frac{ G^1_s(\theta_s^{U})}{ G^0_s(\theta_s^{U})} = \frac{u_-(1-\alpha_s)}{u_+\alpha_s}$. 
\end{lemma}
This is similar to results obtained in prior work (e.g., \citep{zhang2022fairness,liao2022social}). The proof is provided in Appendix~\ref{app:proof-lemma-nonstrategic-unfair-theta} for completeness.

If the firm is cognizant of the agents' strategic behavior, the \emph{unfair strategic} firm's optimal thresholds $\hat{\theta}_s^{U} $ can be obtained by finding the utility maximizer when the firm's utility \eqref{eq:expected-utility} is evaluated on the population post-strategic statistics characterized in Lemma \ref{lemma:post-strategic-stats}. Specifically, the strategic firm's utility can be written as: 
\begin{align}\label{eq:strategic-utility-with-post-statistics}
    \hspace{-0.2in}\mathbb{E}[\hat{U}({\theta}_a, {\theta}_b)]  &:= \mathbb{E}[U({\theta}_a, {\theta}_b)] \notag\\
    &\hspace{-0.6in}+\sum_{s\in\{a,b\}} n_s ( u_+\alpha_s \boldsymbol{\Phi}_{s,(i)}^1({\theta_s}) 
    + u_+ (1-\alpha_s) \boldsymbol{\Phi}_{s,(i)}^0({\theta_s}) +u_+\alpha_s \boldsymbol{\Psi}_{s,(i)}^1({\theta_s}) - u_- (1-\alpha_s) \boldsymbol{\Psi}_{s,(i)}^0({\theta_s})), 
\end{align} 
where the terms $\boldsymbol{\Phi}_{s,(i)}^y({\theta_s})$ (resp. $\boldsymbol{\Psi}_{s,(i)}^y({\theta_s})$) capture the expected value of the change in the firm's utility relative to the non-strategic setting due to label $y$ agents from group $s$ who opt for improvement (resp. manipulation) when facing classifier $\theta_s$ under type \emph{(i)} best-responses, and are given by
\begin{align}\label{eq:phi_psi_definition}
    \boldsymbol{\Phi}_{s,(i)}^y({\theta_s}) &=  
    \int_{z \in \mathbb{I}^y_{s,(i)}} G^y_s(z)\left(1- \mathbb{T}^y_{I,s}({\theta_s} - z)\right) \mathrm{d}z, \\ \notag 
    \boldsymbol{\Psi}_{s,(i)}^y({\theta_s}) &= \int_{z \in \mathbb{M}^y_{s,(i)}} G^y_s(z)\left(1- \mathbb{T}^y_{M,s}({\theta_s} - z)\right) \, \mathrm{d}z.
\end{align}

Note that the first three of these terms \emph{increase} the firm's utility over the non-strategic utility: these are all agents who improved and got accepted, and the qualified agents who manipulated their features and got accepted. The last term, with the negative sign, \emph{decreases} the firm's utility, as these are unqualified agents who pass the threshold $\theta_s$ through manipulation. 

Using the above expressions, we can characterize the unfair strategic firm's optimal policy. 

\begin{lemma}\label{lemma:strategic_unfair_theta}
    The unfair strategic firm's optimal decision thresholds $\hat{\theta}_s^{U} $ for best-response type (i)
    satisfy $ \linebreak[4] \frac{G_s^1(\hat{\theta}^{U}_s)-\boldsymbol{\Phi_{s,(i)}^{'1}}(\hat{\theta}^{U}_s) - \boldsymbol{\Psi_{s,(i)}^{'1}}(\hat{\theta}^{U}_s) - \frac{(1-\alpha_s)}{\alpha_s} \boldsymbol{\Phi_{s,(i)}^{'0}}(\hat{\theta}^{U}_s)}{G_s^0(\hat{\theta}^{U}_s) - \boldsymbol{\Psi_{s,(i)}^{'0}}(\hat{\theta}^{U}_s)} = \frac{u_-(1-\alpha_s)}{u_+\alpha_s} $, where $\boldsymbol{\Phi_{s,(i)}^{'y}} $ (resp. $\boldsymbol{\Psi_{s,(i)}^{'y}} $) is the first order derivative of $\boldsymbol{\Phi_{s,(i)}^{y}} $ (resp. $\boldsymbol{\Psi_{s,(i)}^{y}} $) with respect to ${\theta_s} $, and with the specific forms under each best-response type given in Table \ref{table:equilibria_ratio}.

\begin{table*}[ht]
    \begin{threeparttable}
    \centering
    \caption{Optimal decision thresholds in Lemma~\ref{lemma:strategic_unfair_theta}.}\label{table:equilibria_ratio}
        \begin{tabular}{||p{0.05\linewidth}|p{0.3\linewidth}|p{0.36\linewidth}|p{0.16\linewidth}||}
        \hline
         Type  & \textbf{Type 1} & \textbf{Type 2} & \textbf{Type 3} \\\hline
        \multirow{8}{*}{\rotatebox[origin=c]{90}{Numerator}}  & $C_{I,s}G^1_s(\mathbf{o}^1_{I,s})  
        + (G^1_s*\tau^1_{I,s})(\hat{\theta}^U_s) - (C_{I,s}-C_{M,s})G^1_s(\mathbf{r}^1_{s}) + (G^1_s*\tau^1_{M,s})(\hat{\theta}^U_s)
        + \frac{(1-\alpha_s)}{\alpha_s} \Big(C_{I,s}G^0_s(\mathbf{o}^0_{I,s}) + (G^0_s*\tau^0_{I,s})(\hat{\theta}^U_s) - G^0_s(\mathbf{r}^0_{s})\Big)$ & $ C_{M,s}G^1_s(\mathbf{o}^1_{M,s}) + (G^1_s*\tau^1_{I,s})(\hat{\theta}^U_s)   + (G^1_s*\tau^1_{M,s})(\hat{\theta}^U_s) 
        + G^1_s(\mathbf{f}^1_{s})(C_{I,s}-C_{M,s}) - G^1_s(\mathbf{r}^1_{s}) (C_{I,s}-C_{M,s})
        + \frac{(1-\alpha_s)}{\alpha_s} \Big(G^0_s(\mathbf{f}^0_{s})  (1-\mathbb{T}^0_{I,s}(\hat{\theta}_s^{U}-\mathbf{f}^0_{s})) - G^0_s(\mathbf{r}^0_{s}) + (G^0_s*\tau^0_{I,s})(\hat{\theta}^U_s) 
        \Big) $ & $C_{M,s}G^1_s(\mathbf{o}^1_{M,s}) + (G^1_s*\tau^1_{M,s})(\hat{\theta}^U_s)$ 
        \\
        \hline
        \multirow{3}{*}{\rotatebox[origin=c]{90}{\makecell{Deno- \\ minator}}}  & $ (1-(C_{I,s}-C_{M,s}))G^0_s(\mathbf{r}^0_{s})   + (G^0_s*\tau^0_{M,s})(\hat{\theta}^U_s)
        $ & $ C_{M,s}G^0_s(\mathbf{o}^0_{M,s})  - (1-\mathbb{T}^0_{M,s}(\hat{\theta}_s^{U}-\mathbf{f}^0_{s}))G^0_s(\mathbf{f}^0_{s}) + (1-(C_{I,s}-C_{M,s}))G^0_s(\mathbf{r}^0_{s})  + (G^0_s*\tau^0_{M,s})(\hat{\theta}^U_s) 
        $ & $ C_{M,s} G^0_s(\mathbf{o}^0_{M,s}) +(G^0_s*\tau^0_{M,s})(\hat{\theta}^U_s)
        $ \\
        \hline
        \end{tabular}
    \end{threeparttable}
\end{table*}
     
\end{lemma}

The detailed proof appears in Appendix~\ref{app:proof_lemma_strategic_unfair_theta}. The main challenge in finding these optimal thresholds is that the regions $\mathbb{{I}}^y_{s, (i)}$ and $\mathbb{{M}}^y_{s, (i)}$ in which agents opt for the improvement and manipulation actions are functions of the decision variables ${\theta}_s$, and these dependencies should be accounted for when evaluating the derivatives $\boldsymbol{\Phi_{s,(i)}^{'y}} $ and $\boldsymbol{\Psi_{s,(i)}^{'y}} $. This is done by applying the Leibniz integral rule, and then leveraging the relation between the threshold and the indifference points in agents' strategic responses, which appear in the integration limits and their derivatives, to simplify the resulting expressions. 

Intuitively, the terms in Table~\ref{table:equilibria_ratio} can be interpreted as balancing the marginal gains and losses of a firm at the optimal decision thresholds. To illustrate, let's take the expression $(1-\alpha_s)\Big(C_{I,s}G^0_s(\mathbf{o}^0_{I,s}) + (G^0_s*\tau^0_{I,s})(\hat{\theta}^U_s) - G^0_s(\mathbf{r}^0_{s})\Big)$, appearing in the numerator of Type 1 equilibrium characterization, which reflects the rate of change in the benefits from accepting unqualified agents who opt for improvement, as the decision threshold changes. If the decision threshold $\hat{\theta}^U_s$ increases by a small $\epsilon$, the firm will lose the agents at $x=\mathbf{o}^0_{I,s}$ who successfully made it to the old threshold through improvement (at a rate $(1-\alpha_s)C_{I,s}G^0_s(\mathbf{o}^0_{I,s})$). The firm will also lose some of those with $x\in(\mathbf{o}^0_{I,s}, \mathbf{r}^0_{I,s})$ who no longer make it to the new threshold (at a rate $(1-\alpha_s)(G^0_s*\tau^0_{I,s})(\hat{\theta}^U_s)$), yet will \emph{gain} from agents with $x=\mathbf{r}^0_{s}$ who now opt for improvement instead of manipulation (at a rate $(1-\alpha_s)G^0_s(\mathbf{r}^0_{s})$). Other expressions can be intuitively interpreted in a similar way. 

Lastly, we extend the above two lemmas when the firm also incorporates a fairness constraint in its selection. 
Lemma \ref{lemma:non-strategic_fair_theta} characterizes the optimal non-strategic fairness-constrained decision thresholds. 
\begin{lemma}\label{lemma:non-strategic_fair_theta}
    The fair non-strategic firm's optimal decision thresholds $\theta_s^{\mathcal{C}} $ satisfy: $\linebreak[4] \sum_s n_s\frac{u_+ \alpha_s G^1_s(\theta_s^{\mathcal{C}}) - u_- (1-\alpha_s)G^0_s(\theta_s^{\mathcal{C}}) }{\frac{\partial{\mathcal{C}^{f}_{s}(\theta_s^{\mathcal{C}})}}{\partial{\theta_s^{\mathcal{C}}}}} = 0$.
\end{lemma} 
Lemma \ref{lemma:strategic_fair_theta} characterizes the optimal strategic fairness-constrained decision thresholds.
\begin{lemma}\label{lemma:strategic_fair_theta}
    The fair strategic firm's optimal decision thresholds $\hat{\theta}_s^{\mathcal{C}} $
    satisfy: $ \linebreak[4] \sum_{s} n_s $ $\frac{u_+ \alpha_s (\boldsymbol{\Phi_{s,(i)}^{'1}}(\hat{\theta_s^{\mathcal{C}}}) + \boldsymbol{\Psi_{s,(i)}^{'1}}(\hat{\theta_s^{\mathcal{C}}}) - G_s^1(\hat{\theta_s^{\mathcal{C}}})) + (1-\alpha_s) (u_+ \boldsymbol{\Phi_{s,(i)}^{'0}}(\hat{\theta_s^{\mathcal{C}}}) - u_- (\boldsymbol{\Psi_{s,(i)}^{'0}(\hat{\theta_s^{\mathcal{C}}})} - G_s^0(\hat{\theta_s^{\mathcal{C}}})))}{\frac{\partial{\mathcal{C}^{f}_{s}(\hat{\theta}_s^{\mathcal{C}})}}{\partial{\theta_s^{\mathcal{C}}}}} = 0$.
\end{lemma}

The proofs follow from applying similar techniques to those in \citep{zhang2022fairness}, combined with the results of Lemma~\ref{lemma:strategic_unfair_theta}; these proofs are included in Appendix~\ref{app:fair-threshold-lemmas-proofs} for completeness.

Lemmas~\ref{lemma:nonstrategic-unfair-theta}-\ref{lemma:strategic_fair_theta}, together with Proposition~\ref{prop:agents-br-generic}, fully characterize the possible (Stackelberg) equilibria of the game between agents and the firm, depending on the relative cost-efficacy of strategic actions, whether the firm is cognizant of agents' strategic behavior, and whether the firm imposes any fairness constraints. We next use these characterizations to assess the impacts of anticipating agents' strategic behavior on the optimal decision thresholds and the firm's utility, as well as the fairness implications of this across groups. 
\section{Effects of Anticipating Strategic Behavior}\label{sec:effect-strateg-prediction}

We first proceed with investigating the impacts of anticipating agents' strategic behavior on the optimal policies and the firm's utility by comparing the non-strategic policy $\theta^U_s$ (from Lemma~\ref{lemma:nonstrategic-unfair-theta}) with the strategic policy $\hat{\theta}_s$ (from Lemma~\ref{lemma:strategic_unfair_theta}) in the absence of any fairness interventions.  
For concreteness, we do this analysis {for cases when all agents opt for the same type of best-response, and specifically,} for Types 1 and 3 of agents' best responses, which can be interpreted as relatively low and high improvement costs, respectively, as illustrated in Section~\ref{sec:agents-behavior}. {Accordingly, we will also refer to Type 1 and Type 3 equilibria as ``affordable improvement'' and ``costly improvement'', respectively.} Notably, Type 3/costly improvement equilibria only include manipulation decisions, whereas Type 1/affordable improvement equilibria include both manipulation and improvement decisions at equilibrium. 

Contrasting these two types of equilibria, we can investigate to what extent anticipating gaming can prevent manipulation but also potentially encourage improvement decisions. Notably, while some prior works (e.g.,  \citet{zhang2022fairness, xie2024learning}) show that anticipating gaming can reduce manipulations and potentially encourage improvements, they are limited by the assumption that these actions are only available to \emph{unqualified} individuals. Our work takes a broader view by considering the availability of these actions to any agent (in addition to other model extensions noted earlier), and therefore sheds light on how these effects differ across qualification states, an aspect not considered in prior work. Our broader view also has implications for the \emph{fairness} of anticipating strategic behavior, precisely as there are differences in how qualified vs. unqualified agents are impacted by the firm's anticipation of strategic responses. 

Our main result, Proposition~\ref{prop:firm_impact_comp}, compares the strategic and non-strategic decision thresholds, and their influence on different components of the firm utility. Recall that when a decision threshold $\theta_s$ is lowered (resp. increased), more (resp. fewer) agents are accepted without taking any strategic action. The relative ordering of the thresholds therefore allows us to comment on how the strategic responses of agents from different labels and groups changes if a firm can anticipate strategic responses. To establish these results, we make two additional assumptions: one on the feature distributions, and one on the costs and boost functions for different actions across qualification states. 

\begin{assumption}\label{as:sym_bounded_dist}
     Agents' feature distribution $G^y_s(x)$ is symmetric and bounded, with lower and upper domain limits $\underline{x}^y_s$, and $\overline{x}^y_s$, respectively. Further, $\mu^0_s<\underline{x}^1_s < \overline{x}^0_s<\mu^1_s$, where $\mu^y_s$ denotes the mean of $G^y_s(x)$. 
\end{assumption}

The assumption of symmetric feature distributions is made for tractability of the analysis. The inequality in the assumption poses a natural (and mild) restriction: unqualified agents have, on average, lower scores than qualified agents ($\mu_s^0<\mu^1_s$), and  there are overlaps in their feature distributions ($\underline{x}^1_s < \overline{x}^0_s$) so that distinguishing them based on their feature $x$ is a non-trivial task.

\begin{assumption}\label{as:relation_of_boost}
    The manipulation cost and boost distribution are the same for both labels (i.e., $C^1_{M,s} = C^0_{M,s} = C_{M,s}$ and $\tau^1_{M,s} = \tau^0_{M,s} = \tau_{M,s}$). The improvement boost distribution of qualified agents first-order stochastically dominates (FOSD) that of unqualified agents ($i.e., \mathbb{T}_{I,s}^0(b) \geq \mathbb{T}_{I,s}^1(b) $). 
\end{assumption}

Assumption~\ref{as:relation_of_boost} states that the cost and effectiveness of the manipulation action is independent of the qualification state. For instance, using a Generative AI tool to solve basic homework problems has the same cost and effectiveness regardless of a student's preparation level/qualification state. 
In contrast, the second part of the assumption implies that the improvement action is more effective for qualified agents. Continuing on the same example, when a qualified student with a stronger background spends time preparing well for an exam, their effort can boost their exam grades more than a less prepared student spending the same time. 

Given Assumptions~\ref{as:sym_bounded_dist} and \ref{as:relation_of_boost}, we can now establish relative orderings between $\theta_s$ and $\hat{\theta}_s$, and their impact on the different components of the firm's utility. Recall, from the expression for the firm's utility \eqref{eq:strategic-utility-with-post-statistics}, that $\boldsymbol{\Phi}_{s,(i)}^y({\theta_s})$ and $\boldsymbol{\Psi}_{s,(i)}^y({\theta_s})$ capture the change in the firm's utility due to label $y$ agents from group $s$ who opt for improvement and manipulation, respectively, in equilibrium Type (i), given a threshold $\theta_s$.

\begin{proposition}\label{prop:firm_impact_comp}  
 Assume $\underline{x}^1_s \geq \mu^0_s+\underline{b}_{M,s}$, $ \overline{x}^0_s- \underline{x}^0_s < \xrisk - \mathbf{o}^0_{I,s}$, $\underline{x}^1_s- \underline{x}^0_s>(\mathbb{T}^y_{M,s})^{-1}(1-C_{M,s})$, $\overline{x}^0_s +(\mathbb{T}^y_{M,s})^{-1}(C_{I,s}-C_{M,s})  < \mu^0_s + \underline{b}^0_{I,s} < \mu^1_s$, and 
 $\overline{x}^0_s + (\mathbb{T}^y_{M,s})^{-1}(1-C_{M,s})<\mu^1_s + \underline{b}_{M,s}$, and that Assumptions~\ref{as:sym_bounded_dist} and \ref{as:relation_of_boost} hold. Then, 
 \begin{itemize}
        \item[(i)] The strategic firm chooses a higher threshold: $\theta^{U}_s < \hat{\theta}^{U}_s$,
        \item[(ii)] The non-strategic firm faces more manipulation by unqualified agents: $ \boldsymbol{\Psi}^0_{s,(i)}(\theta^{U}_s) > \boldsymbol{\Psi}^0_{s,(i)}(\hat{\theta}^{U}_s), ~i\in \{1,3\}$,
        \item[(iii)] The strategic firm faces more manipulation by qualified agents: $\boldsymbol{\Psi}^1_{s,(i)}(\theta^{U}_s) < \boldsymbol{\Psi}^1_{s,(i)}(\hat{\theta}^{U}_s), ~i\in \{1,3\}$,
        \item[(iv)] The strategic firm drives more improvement by both qualified and unqualified agents: $\linebreak[4] \boldsymbol{\Phi}^y_{s,(1)}(\theta^{U}_s) < \boldsymbol{\Phi}^y_{s,(1)}(\hat{\theta}^{U}_s), ~y \in \{0,1\}$. 
    \end{itemize}
\end{proposition}

The proof is provided in Appendix~\ref{sec:proof_Prop2_coro1}. We next provide a brief intuitive interpretation for the findings of this proposition. We discuss the impacts of relaxing the assumptions of  Proposition~\ref{prop:firm_impact_comp} in Appendix~\ref{app:relax_assumptions_prop2}, as well as in the numerical experiments in Section~\ref{sec:numerical-exp}. 

The key takeaways behind Proposition~\ref{prop:firm_impact_comp} is as follows: by anticipating agents' strategic responses, the firm chooses a higher threshold than a non-strategic firm, consequently reducing incentive toward manipulation for unqualified agents while increasing manipulation incentives for qualified ones, both of which align with improving its utility. Notably, the fact that anticipation of gaming increases the ``social burden'' on qualified individuals, by forcing them to opt of manipulation has also been observed under different models in prior work~\citep{milli2019social}. 
Moreover, this higher threshold has even more preferred consequences for the scenario of affordable improvement (Type 1). It increases genuine improvement incentives for all agents, and especially for unqualified agents, for whom genuine improvement becomes the only successful path to acceptance. Together, these observations highlight how anticipating strategic responses will not only reduce gaming of the algorithm by unqualified individuals (as also shown in prior work), but also lead to increased improvement incentives among \emph{all agents} (a new effect identified by our model). An additional detailed and illustrative interpretation of both costly-improvement (Type 3) and affordable-improvement (Type 1) equilibria, and how each drive the changes in manipulation vs. improvement, is provided in Appendix~\ref{sec:intuition_prop2}. 
\subsection{Effects of fairness interventions}\label{sec:fairness_effect}

We next proceed to studying how enforcing fairness at the policy level shapes agents' strategic responses. As we show in this section, fairness constraints can have nontrivial and sometimes negative effects on improvement incentives across groups.

Specifically, assume the firm imposes a fairness constraint $\mathcal{C}$ when selecting the optimal threshold (as formalized in \eqref{eq:fair-expected-utility}), on the contrast between the unfair policies (characterized in Lemma~\ref{lemma:nonstrategic-unfair-theta} and \ref{lemma:strategic_unfair_theta}) with the corresponding fair policies (characterized in Lemmas~\ref{lemma:non-strategic_fair_theta} and \ref{lemma:strategic_fair_theta}). Similar to Section~\ref{sec:effect-strateg-prediction}, we do so when all agents opt for the same response type, and for Types 1 and 3 of agents' best responses. Our comparison, outlined in the following Corollary, is a direct consequence of the proof of Proposition~\ref{prop:firm_impact_comp} under the stated assumptions. 

\begin{corollary}\label{cor:fair-policies}
For a non-strategic firm, assume that $\theta^{U}_a < \theta^{\mathcal{C}}_a \leq \arg\max_{\theta_s} \boldsymbol{\Phi}^0_{a,(1)}(\theta_s)$ and that $\theta^{U}_b > \theta^{\mathcal{C}}_b \geq \boldsymbol{\Phi}^y_{b,(1)}(\hat{\theta}^{\mathcal{C}}_b)$, where $a$ and $b$ denote the advantaged and disadvantaged groups, respectively. Then, imposing a fairness constraint:
\begin{enumerate}
    \item Decreases (increases) manipulation by unqualified advantaged (disadvantaged) agents ($\linebreak[4] \boldsymbol{\Psi}^0_{a,(i)}(\hat{\theta}^{\mathcal{C}}_a) \leq \boldsymbol{\Psi}^0_{a,(i)}(\hat{\theta}^{U}_a)$ and $\boldsymbol{\Psi}^0_{b,(i)}(\hat{\theta}^{\mathcal{C}}_b) > \boldsymbol{\Psi}^0_{b,(i)}(\hat{\theta}^{U}_b),~ i \in \{1,3\}$), 
    \item Increases (decreases) manipulation by qualified advantaged (disadvantaged) agents (i.e., $\boldsymbol{\Psi}^1_{a,(i)}(\hat{\theta}^{\mathcal{C}}_a) \geq \boldsymbol{\Psi}^1_{a,(i)}(\hat{\theta}^{U}_a)$ and $\boldsymbol{\Psi}^1_{b,(i)}(\hat{\theta}^{\mathcal{C}}_b) \leq \boldsymbol{\Psi}^1_{b,(i)}(\hat{\theta}^{U}_b),~ i \in \{1,3\}$), 
    \item Increases (decreases) improvement by both qualified and unqualified advantaged (disadvantaged) agents (i.e., $\boldsymbol{\Phi}^y_{a,(1)}({\theta}^{\mathcal{C}}_a) > \boldsymbol{\Phi}^y_{a,(1)}({\theta}^{U}_a)$, and $\boldsymbol{\Phi}^y_{b,(1)}(\hat{\theta}^{\mathcal{C}}_b) \leq \boldsymbol{\Phi}^y_{b,(1)}(\hat{\theta}^{U}_b)$, $y\in\{0,1\}$). 
\end{enumerate}
The same statements hold for the strategic firm with the corresponding $\hat{\theta}^{{U}}_s$ and $\hat{\theta}^{\mathcal{C}}_s$. 
\end{corollary}

These observations are primarily due to the ordering of fair vs. unfair policies, where we have assumed that the firm attains fairness by raising the decision threshold for the advantaged group (\(\theta^{\mathcal{C}}_a > \theta^{U}_a\)) and lowering it for the disadvantaged group (\(\theta^{\mathcal{C}}_b < \theta^{U}_b\)). This is a natural assumption, supported by our numerical experiments in Section~\ref{sec:numerical-exp}, and also commonly observed in existing literature on fair machine learning. In particular, \cite{estornell2023group} discuss that in binary classification, fairness can be improved in two main ways: \emph{inclusiveness} (flipping disadvantaged group predictions to positive) or \emph{selectivity} (flipping advantaged group predictions to negative). Our assumption of a decrease in the group $b$ threshold aligns with inclusiveness, while that of an increase in the group $a$ threshold aligns with selectivity, to attain fairness. 

In our context, Corollary~\ref{cor:fair-policies} argues that in its attempt to attain fair outcomes, the firm \emph{decreases} improvement incentives of the disadvantaged group $b$, while \emph{increasing} the improvement incentives for the advantaged one; this can be viewed as a negative consequence of fairness interventions on shaping agents' strategic responses. As we next elaborate through numerical experiments in Section~\ref{sec:numerical-exp}, this negative by-product is more severe if the firm is non-strategic (i.e., unaware of agents' strategic responses). 
\section{Numerical Experiments} \label{sec:numerical-exp} 
In this section, we use numerical experiments to validate the analytical findings from Section~\ref{sec:effect-strateg-prediction}. We focus on the affordable-improvement (Type~1) equilibrium, which features both manipulation and improvement; results for the costly-improvement (Type~3) case are deferred to Appendix~\ref{app:type3_fico_experiments}. Experiments are conducted on a semi-synthetic dataset based on real-world FICO credit data that includes four demographic groups: African American (AA), Hispanic (H), Caucasian (C), and Asian (A). To illustrate the main observations from Proposition~\ref{prop:firm_impact_comp}, we first compare strategic vs. non-strategic policies in a two-group setting (AA vs. C) across varying target qualification rates $\alpha_s$. The full extension to all four groups is given in Appendix~\ref{app:Prop2_all_groups}. We then examine how fairness interventions reshape these incentive effects in settings with two demographic groups (C vs. AA), and subsequently extend the analysis to all demographic groups in Appendix~\ref{app:fairness_all_groups}. Similar and additional experiments on a synthetic dataset are provided in Appendix~\ref{app:add_NE}. 

\paragraph{FICO data \citep{hardt2016equality}.}
We use the empirical FICO score distributions from prior work \citep{xie2024learning,zhang2022fairness} to generate group-conditioned samples $(X,Y)\mid S=s$, without imposing any parametric structure. The dataset includes four demographic groups (AA, H, C, A) with qualification rates $\alpha_s \in (0.33849, 0.56977, 0.75972, 0.80467)$, and all scores are normalized to $[0,1]$. The empirical CDFs, PDFs, and repayment likelihoods $P(Y=1\mid X,S)$ are shown in Appendix~\ref{app:fico_data} for reference.

\paragraph{Improvement vs.\ manipulation: empirical grounding.}
To calibrate the boost and cost parameters in our simulations, we draw on empirical evidence from credit-scoring contexts, where both genuine improvement (e.g., on-time repayment, credit-builder loans) and manipulative strategies (e.g., tradeline renting) are well-documented. Improvement actions typically yield substantially larger score increases than manipulation (e.g., $45$--$85$ vs.\ $7$--$45$ points on average \citep{self2025cba,nasdaq2025esusu,CFPB2020_2,sofi2024piggybacking,dovly2024,ABA2020,avery2010credit,foxnews2007piggybacking,ftc2020}), but at higher effort and financial cost, consistent with our assumption that improvement is more effective yet costlier in expectation (e.g., \$300 vs.\ \$390 \citep{LendEDU2025,nasdaq2021,ftc2020}). Moreover, improvement actions appear more effective for individuals with better credit histories, with an average increase of $12$ points. Formally, we specify the improvement-boost distribution for unqualified and qualified agents as $\tau^0_{I,s}(b) \sim \text{TruncNorm}\!\big([0.45,0.85],\,0.65,\,0.15\big), \quad \tau^1_{I,s}(b) \sim \text{TruncNorm}\!\big([0.57,0.97],\,0.77,\,0.15\big),$ while manipulation boosts for all agents follow $\linebreak[4] \tau^y_{M,s}(b) \sim\text{TruncNorm}\!\big([0.07,0.45],\,0.26,\,0.22\big)$. The normalized costs are $C_{M,s}=0.2$ and $C_{I,s}=0.3$. Under these choices, group-level disparities are attributed to \emph{historical} differences in qualification rates and feature distributions, rather than to \emph{current} disparities in strategic capacity.

\paragraph{Experimental setup for two demographic groups.} We compare policy behavior, without considering fairness, across two demographic groups:  African American (AA) and Caucasian (C). For each group and target qualification rate $\alpha_s$, a single replicate proceeds as follows. We sample $n=1000$ agents such that approximately $\alpha_s n$ agents are drawn from the empirical distribution of scores among repayers ($y=1$) and the remaining agents are drawn from the empirical distribution among non-repayers ($y=0$). We vary $\alpha_s$ over $\{0.1, 0.2, \dots, 0.9\}$. For each $\alpha_s$ and each group, we run 50 independent replications and report averages across runs. 

\paragraph{Strategic vs. non-strategic (unfair) firms.} 
Figure~\ref{fig:type1_unfair_vis_all_fico} shows that incorporating strategic responses reshapes behavior for both groups. We start the discussion on the unqualified agents' behavior. Relative to a non-strategic firm, the strategic firm induces more genuine improvement (Figure~\ref{fig:type1_comp_Phi0_fico}), with the effect being more pronounced for smaller values of $\alpha_s$. Simultaneously, it reduces overall manipulation (Figure~\ref{fig:type1_comp_Psi0_fico}), although a slight increase is observed when $\alpha_s$ is very small ($<0.25$). This is because the firm has a considerable opportunity for inducing improvement in these settings. These effects appear in both AA and C groups and reflect a shift from gaming toward genuine improvement, which results in shifting the group's qualification rates toward 1, as illustrated in Figure~\ref{fig:alpha_post_com_unfair_fico}. 

Although the non-strategic firm generates modest improvement, which is most noticeable around qualification rate of $0.25$ due to the high thresholds as appears in Figure~\ref{fig:type1_comp_theta_fico}, it rejects most unqualified agents who become qualified by improvement by missing out on their change due to being unaware of strategic behavior. As a result, downstream utility and improvement effects are weaker under non-strategic decision-making (Figures~\ref{fig:type1_comp_utility_fico} and \ref{fig:type1_comp_Phi0_fico}).

\begin{figure}[htbp]
    \centering
    \begin{subfigure}{0.22\textwidth}
        \centering
        \includegraphics[width=\textwidth,trim={0 0 0 0},clip]{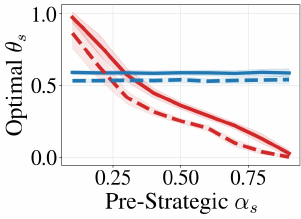}
        \caption{Optimal thresholds}
        \label{fig:type1_comp_theta_fico}
    \end{subfigure}
    \begin{subfigure}{0.22\textwidth}
        \centering
        \includegraphics[width=\textwidth,trim={0 0 0 0},clip]{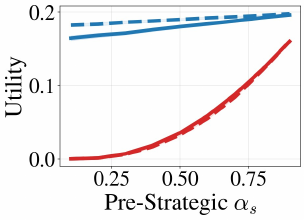}
        \caption{Optimal utility}
        \label{fig:type1_comp_utility_fico}
    \end{subfigure}
        \begin{subfigure}{0.22\textwidth}
        \centering
        \includegraphics[width=\textwidth,trim={0 0 0 0},clip]{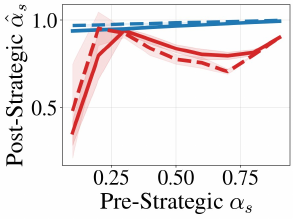}
        \caption{Post-strategic $\hat{\alpha}_s$}
        \label{fig:alpha_post_com_unfair_fico}
    \end{subfigure}
    \begin{subfigure}{0.22\textwidth}
        \centering
        \includegraphics[width=\textwidth,trim={0 0 0 0},clip]{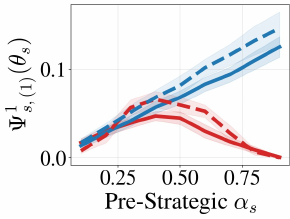}
        \caption{$\boldsymbol{\Psi}^1_{s,(1)}$ (Manipulation)}
        \label{fig:type1_comp_Psi1_fico}
    \end{subfigure}
    \begin{subfigure}{0.22\textwidth}
        \centering
        \includegraphics[width=\textwidth,trim={0 0 0 0},clip]{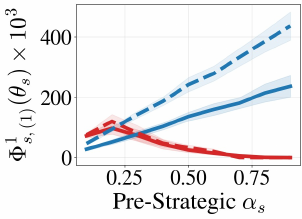}
        \caption{$\boldsymbol{\Phi}^1_{s,(1)}$ (Improvement)}
        \label{fig:type1_comp_Phi1_fico}
    \end{subfigure}
    \begin{subfigure}{0.22\textwidth}
        \centering
        \includegraphics[width=\textwidth,trim={0 0 0 0},clip]{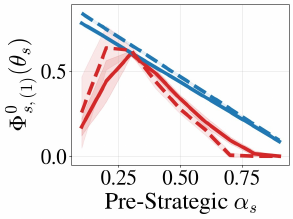}
        \caption{$\boldsymbol{\Phi}^0_{s,(1)}$ (Improvement)}
        \label{fig:type1_comp_Phi0_fico}
    \end{subfigure}
    \begin{subfigure}{0.22\textwidth}
        \centering
        \includegraphics[width=\textwidth,trim={0 0 0 0},clip]{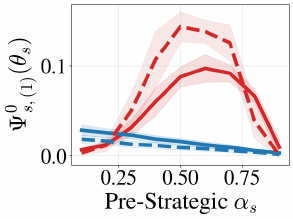}
        \caption{$\boldsymbol{\Psi}^0_{s,(1)}$ (Manipulation)}
        \label{fig:type1_comp_Psi0_fico}
    \end{subfigure}
    \begin{subfigure}{0.2\textwidth}
        \centering
        \includegraphics[width=\textwidth,trim={0 0 0 0},clip]{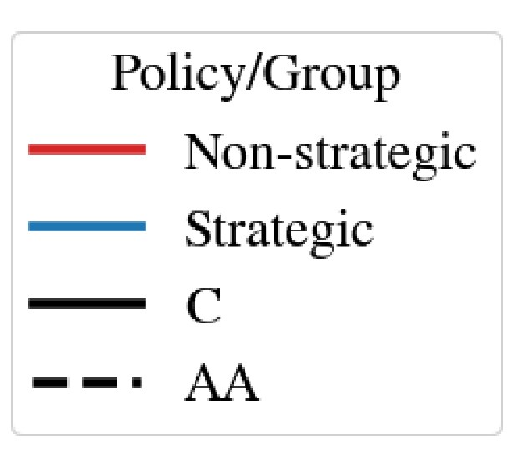}
        \caption{Legend}
        \label{fig:type1_comp_legend}
    \end{subfigure}
    \caption{Unfair policies: non-strategic vs. strategic policies' comparison when varying pre-strategic $\alpha_s$.}
    \label{fig:type1_unfair_vis_all_fico}
\end{figure}

For qualified agents, the effects depend on the group’s initial composition $\alpha_s$ (Figures~\ref{fig:type1_comp_Psi1_fico} and \ref{fig:type1_comp_Phi1_fico}). When a group is majority qualified (high $\alpha_s$), the strategic firm induces both more improvement and manipulation; when it is majority unqualified (low $\alpha_s$), both improvement and manipulation decline as a result of the lower threshold (compared to non-strategic as seen in Figure~\ref{fig:type1_comp_theta_fico}), thereby admits most qualified agents without additional effort. Note that the small corner divergences in the qualified-agents patterns relative to Proposition~\ref{prop:firm_impact_comp} arise because some of its technical assumptions (notably Assumption~\ref{as:sym_bounded_dist}) are relaxed in these experiments. Importantly, both manipulation and improvement by such agents increase firm payoff. As a result, these deviations do not alter our main conclusions, which are driven by the behavior of \emph{unqualified} agents and remain fully consistent with the theory. Consistent with this, Figure~\ref{fig:type1_comp_utility_fico} shows that the strategic firm's utility is significantly higher than that of the non-strategic firm.

A further implication is that the post-strategic qualification rate $\hat{\alpha}_s$ rises more for AA than for C when the initial composition is majority unqualified (Figure~\ref{fig:alpha_post_com_unfair_fico}). This is driven by the pre-strategic score distributions: most unqualified AA agents lie well below the threshold (e.g., many scores fall below $0.5$, which near the optimal strategic threshold as seen in Figure~\ref{fig:type1_comp_theta_fico}), whereas many unqualified C agents are already near or above it ($0.6$) (See Appendix~\ref{app:fico_data}). Consequently, a larger share of AA agents must undertake genuine improvement to be accepted, producing a larger upward shift in $\hat{\alpha}_s$ under strategic behavior.

\begin{figure}[htbp!]
    \centering
        \begin{subfigure}{0.35\textwidth}
        \centering
        \includegraphics[width=\textwidth]{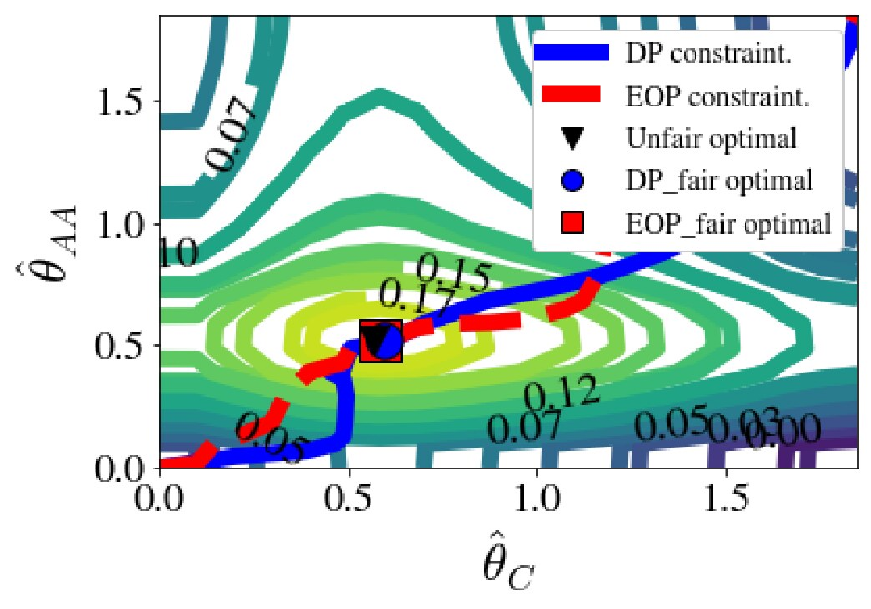}
        \caption{Strategic firm's utility}
        \label{fig:c_aa_3}
    \end{subfigure}
    \quad \quad \quad
    \begin{subfigure}{0.35\textwidth}
        \centering
        \includegraphics[width=\textwidth]{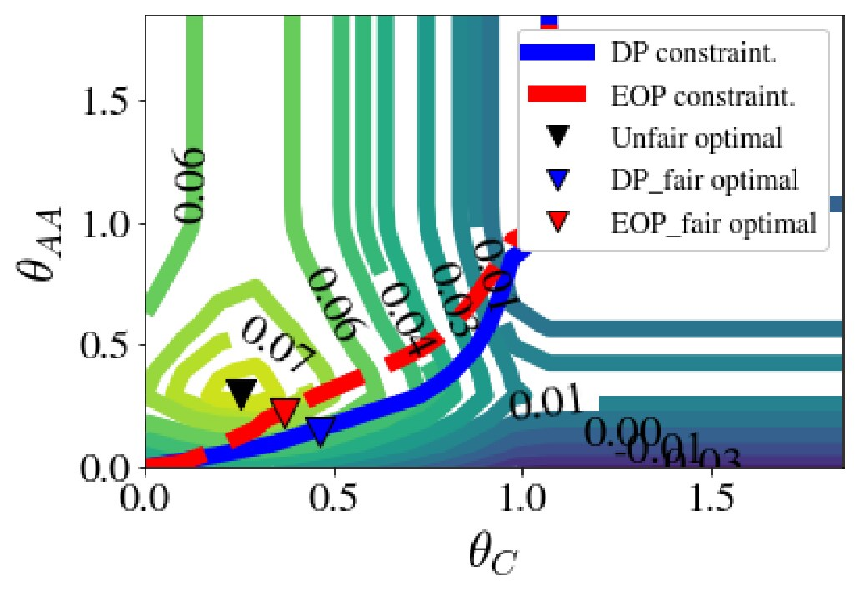}
        \caption{Non-strategic firm's utility}
        \label{fig:c_aa_4}
    \end{subfigure}
    \caption{Firm's utility and optimal thresholds when imposing fairness constraints across Caucasian and African American groups.}
    \label{fig:fair_vis_Type1_c_aa_2}
\end{figure}

\paragraph{Impacts of fairness interventions for strategic vs. non-strategic firms.}
Figure~\ref{fig:fair_vis_Type1_c_aa_2} provides contour visualizations of the firm’s utility surface, together with the corresponding utilities under DP and EOP parity constraints, as the group-specific thresholds $\theta_C$ and $\theta_{AA}$ vary. The optimal thresholds for the unfair, DP-fair, and EOP-fair cases are also indicated in the figure. 
A key distinction is that the strategic firm evaluates fairness over post-strategic statistics, whereas the non-strategic firm evaluates fairness over \emph{pre}-strategic statistics. Consequently, the fairness desiderata of a non-strategic firm are not satisfied ex-post, even though they appear satisfied under its own (incorrect) evaluation.

We first observe that, consistent with both prior literature and Corollary~\ref{cor:fair-policies}, fairness constraints induce both strategic and non-strategic firms to raise the threshold for the majority-qualified group and lower it for the majority-unqualified group. We also observe that the strategic firm (under both fair and unfair policies) selects higher thresholds than the non-strategic firm, in line with Proposition~\ref{prop:firm_impact_comp}.

The more substantive distinction arises when comparing the \emph{type} of fairness intervention. In particular, the gap between strategic and non-strategic firms is much larger under DP fairness than under EOP fairness. The DP-fair threshold chosen by a non-strategic firm (Figure~\ref{fig:c_aa_4}) is substantially lower for AA group, because the firm does not anticipate agents’ strategic responses. This leads to many agents in AA group being accepted by default, while others with even weaker features pass the lowered threshold through manipulation. 

In contrast, a strategic firm lowers the AA group threshold under DP fairness only modestly (Figure~\ref{fig:c_aa_3}), recognizing that parity in selection rates can be achieved by \emph{jointly} adjusting thresholds and inducing improvement responses. Consequently, fairness can be satisfied with far smaller changes to the decision rule, yielding a significantly smaller loss in utility once strategic behavior is accounted for.

\paragraph{Impacts of fairness interventions on agents' post-strategic qualification rates.}
To examine how strategic policies shape agents’ behavior under different fairness interventions (unfair, DP, and EOP), Figure~\ref{fig:c_aa_1} reports the post-strategic qualification rates $\hat{\alpha}_s$ for each group. As noted earlier, the strategic firm is consistently more effective at incentivizing improvement: both the C group and the AA group achieve higher $\hat{\alpha}_s$ under the strategic policy than under the non-strategic one (in Figure~\ref{fig:c_aa_1}, the blue bars are always higher than the red bars). This effect is even more pronounced for AA group, reflecting its initially lower qualification rate.

More importantly, we observe that when a firm ignores strategic responses, the application of fairness interventions can have unintended negative consequences for the disadvantaged group (AA). Specifically, fairness constraints reduce the extent of improvement by AA group, as the red bars under fair policies fall below those under the unfair policy in Figure~\ref{fig:c_aa_1}. The mechanism is straightforward: fairness constraints typically lower the decision threshold for AA group, which shifts behavior toward manipulation rather than genuine improvement. By contrast, a strategic firm internalizes agents’ responsiveness and uses fairness constraints together with threshold adjustment to continue incentivizing improvement, thereby generating higher post-strategic qualification rates for the disadvantaged group (AA).
    
\begin{figure}[hbtp!]
    \centering
        \centering
        \includegraphics[width=0.7\textwidth,trim={0 0 0 0},clip]{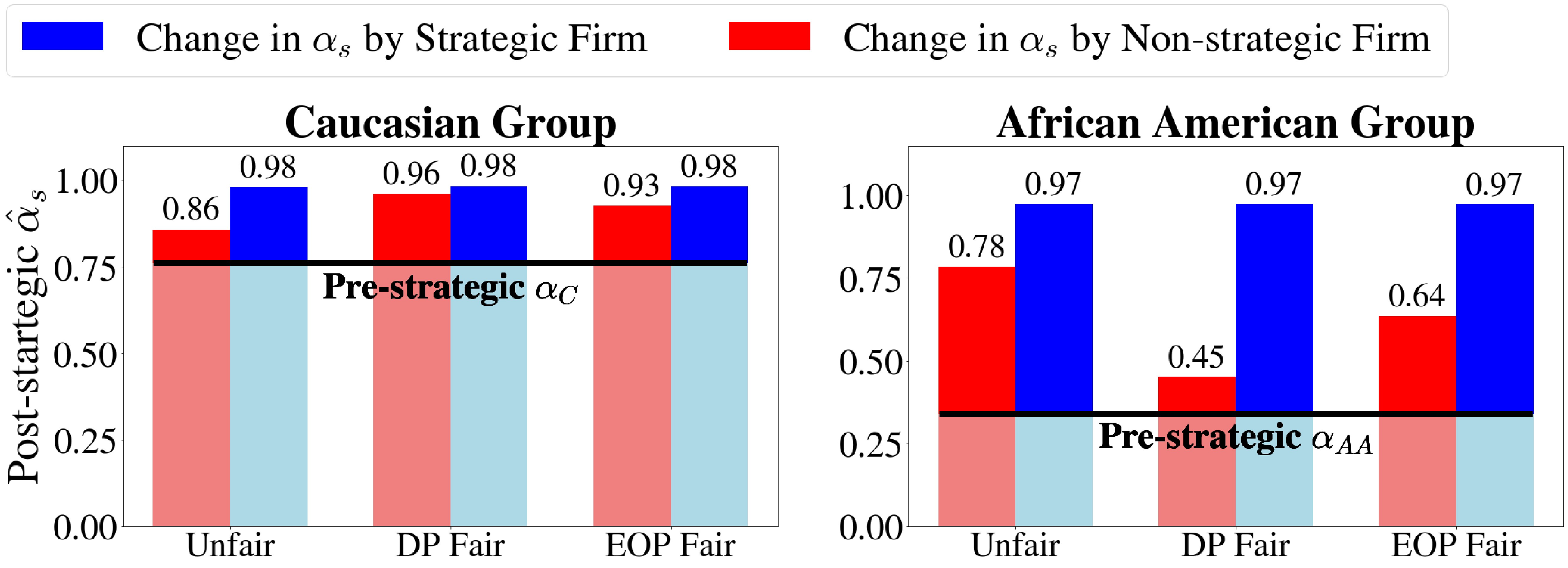}
        \caption{Impact of fair vs.\ unfair and non-strategic vs.\ strategic policies on post-strategic $\hat{\alpha}_s$ across Caucasian and African American groups}
        \label{fig:c_aa_1}
\end{figure}

\section{Conclusion and Future Work}\label{sec:conclusion}
We proposed a Stackelberg framework for strategic classification, where a firm deploys a (fair) classifier, and agents strategically respond. Agents respond by either genuinely improving their qualification and observable features or by manipulating observable features to increase their chances of acceptance. We model heterogeneity in costs, and stochastic effectiveness of these actions yields three types of agents’ best responses based on cost–effectiveness relationships. We find that anticipating these strategic behaviors enables the firm to choose policies that both curb manipulation behavior (as also noted in prior work) and incentivize agents, particularly unqualified ones, to opt for improvement decisions. This shifts unqualified agents away from manipulation towards improvement, yet continues to allow some manipulation by qualified agents who would otherwise be rejected.

We further showed that imposing fairness constraints without anticipating agents' strategic responses can inadvertently reduce improvement incentives for disadvantaged groups, since fairness may be achieved by accepting more individuals from this group. This can be viewed as a negative byproduct of fairness interventions. At the same time, such constraints can increase improvement incentives for advantaged groups as a result of a stricter threshold policy. In contrast, a strategic firm can satisfy fairness while largely maintaining incentives across all groups as it achieves fairness with minimal deviation from the unfair policy \emph{and} by shaping agents’ best responses.

Our analysis assumes similar strategic behavior resources (cost and boost distributions) across groups, and focuses on settings in which all agents (qualified and unqualified) follow the same type of equilibrium best-responses, as characterized in Proposition~\ref{prop:agents-br-generic}. A natural direction for future work is to relax this assumption by allowing asymmetry, which would lead to heterogeneous response types across agents and groups, and yield a richer picture of how incentive design interacts with fairness in practice. An additional promising direction is to extend the model to allow endogenous costs that are shaped by the firm’s policy. 

\bibliography{article}

\newpage
\appendix
\section{Supplementary Material}
\subsection{Proofs for Section~\ref{sec:agents-behavior}}\label{app:proofs-agents-behavior}
\subsubsection{Proof of Proposition \ref{prop:agents-br-generic}}\label{app:proof-agent-best-response}

We begin by establishing the possible orderings between the indifference points of Definition~\ref{def:indiff-features}, under Assumption~\ref{as:cost-efficacy}, 
showing that the cases identified in Proposition~\ref{prop:agents-br-generic} cover all possible equilibrium outcomes given the possible orderings of the four constants $\xindm, \xindi, \xindmi, \xrisk$.

\begin{lemma}\label{lemma:order-indiff-points}  
We have $\xrisk\geq \{\xindmi, \xindm\}$. Further, if $\xindmi$ is unique, either $\xindmi\geq \max\{\xindm, \xindi\}$, or $\xindmi\leq \min\{\xindm, \xindi\}$. Additionally, if $\xindi\leq \xindm$, there is at most one $\xindmi\leq \xindi$. 
\end{lemma}

\begin{proof}
    First, note that $(\mathbb{T}^y_{M,s})^{-1}(C_I-C_{M,s}) \leq (\mathbb{T}^y_{M,s})^{-1}(1-C_{M,s})$; this is because the inverse CDF is an increasing function, and $C_{I,s}\leq 1$. Therefore,  $\xindm\leq\xrisk$. Also, by definition, $\mathbb{T}^y_{M,s}(\theta_s-\xindmi) - \mathbb{T}^y_{I,s}(\theta_s-\xindmi) = \mathbb{T}^y_{M,s}(\theta_s-\xrisk)$, and therefore $\xindmi\leq\xrisk$. Finally, note that if $\xindmi\geq \xindm$, then $\mathbb{T}^y_{I,s}(\theta_s-\xindmi) = \mathbb{T}^y_{M,s}(\theta_s-\xindmi) - (C_{I,s}-C_{M,s})\leq \mathbb{T}^y_{M,s}(\theta_s-\xindm) - (C_{I,s}-C_{M,s}) = 1-C_{I,s} = \mathbb{T}^y_{I,s}(\theta_s-\xindi)$, and therefore $\xindmi \geq \xindi$. Similarly, we can show that if $\xindmi\leq \xindm$, then $\xindmi \leq \xindi$. Lastly, if $\xindi\leq \xindm$, and given that $u_s(0,y,I)<u_s(0,y,M)$ due to FOSD,  it must be that the utility of $I$ crosses the utility of $M$ at some $x\leq \xindi$, and therefore $\xindmi\leq \xindi$. 
\end{proof}

Now, note that if $x\geq \theta_s$, the agent is already admitted by the classifier, and finds it optimal to do nothing. As such, only agents with $x<\theta_s$ may opt for manipulation or improvement decisions. For these agents, the probability of being admitted if neither action is taken is zero. Together with Eq.\eqref{eq:agent-benefit}, this means that the utility of such agents with $x<\theta_s$ when choosing $w\in\{M,I\}$ reduces to $u_s(x,y,w)=\mathbb{P}(\hat{x}\geq \theta_s|X=x,Y=y,W=w,S=s)-C_{w,s}$.  

We now proceed by finding the features $\xindm$ at which the agent first finds it beneficial to opt for manipulation over doing nothing. Note that the agent's utility is non-decreasing in $x$. This is because we have assumed adopting either of the two actions $M$ or $I$ weakly increases the agent's feature, and hence (weakly) increases the probability of being admitted by the classifier. This means that if an agent with feature $\bar{x}$ prefers action $w$ over doing nothing, so will all $x>\bar{x}$. 

Recall that $u_s(x,y,N)=0$. Therefore, $\xindm$ is the first x at which $u_s(x,y,M)\geq 0$. This is given by: 
\begin{align*}
    & \mathbb{P}(\hat{x}\geq \theta_s|X=x,Y=y,W=M,S=s) \geq C_{M,s} \quad 
     \Leftrightarrow \quad x \geq \theta_s- (\mathbb{T}^y_{M,s})^{-1}(1-C_{M,s})~.
\end{align*}
Therefore, the first point at which the agent finds it beneficial to opt for $M$ over $N$ is $\xindm = \max\{0, \theta_s- (\mathbb{T}^y_{M,s})^{-1}(1-C_{M,s})\}$. The first point at which the agent benefits from $I$ over $N$ can be similarly found to be $\xindi = \max\{0,\theta_s- (\mathbb{T}^y_{I,s})^{-1}(1-C_{I,s})\}$. Let $z:=\arg\min_{\{M,I\}}\{\xindi, \xindm\}$. 
Then, agents with $0\leq x<\min\{\xindi, \xindm\}$ opt for $N$, while those with $\min\{\xindi, \xindm\}\leq x < \max\{\xindi, \xindm\}$ opt for action $z$. 

Next, given that the improvement action first-order stochastically dominates the manipulation action, it must be that $\underline{b}^y_{M,s}\leq \underline{b}^y_{I,s}$ (and also that $\bar{b}^y_{M,s}\leq \bar{b}^y_{I,s}$). First, we note that once $x\geq \theta_s-\underline{b}^y_{M,s}$ the agent gets admitted with probability 1 with either $M$ or $I$, and therefore would choose the cheaper action. This means that for $\theta_s-\underline{b}^y_{M,s}\leq x< \theta_s$, the agent chooses action $M$ over $I$. Given the continuity of the utility functions under actions $M$ and $I$, we expect this argument to carry for some $x$ smaller than $\theta_s-\underline{b}^y_{M,s}$ as well. 

Specifically, once $\theta_s-\underline{b}^y_{I,s}\leq x\leq \theta_s-\underline{b}^y_{M,s}$, the agent receives sufficient boost to get admitted with probability 1 when choosing action $I$, but is still uncertain when choosing $M$. Formally, improvement has a utility of $1-C_{I,s}$, whereas manipulation has a utility of $1-T^y_{M,s}(\theta_s-x) - C_{M,s}$. Therefore, if $1-T^y_{M,s}(\theta_s-x) - C_{M,s}\geq 1-C_{I,s}$ in this region, or equivalently once $x\geq \xrisk$, the uncertainty from action $M$ is small enough for the agent to choose action $M$ over $I$. Note also that this argument will continue to hold even if the indifference point $\xrisk$ is such that $\xrisk\leq \theta_s-\underline{b}^y_{I,s}$, because this would only increase the uncertainty about $I$, making the utility from choosing $I$ smaller than $1-C_{I,s}$; this means that the utility of action $M$ will still be higher than the utility of $I$ when $x\geq \xrisk$, making action $M$ preferable to $I$.

Finally, agents with $\max\{\xindi, \xindm\}\leq x \leq \min\{\xrisk, \theta_s-\underline{b}^y_{I,s}\}$ (provided the region is non-empty) would benefit from either manipulation or improvement actions over doing nothing, but face uncertainties about making it to an admit decision when opting for these actions, leading to a cost-efficacy trade off between these choices. Formally, define $\Delta u_s(x,y):=u_s(x,y,I)-u_s(x,y,M)$, the difference between the utility of improvement and manipulation. This difference is:
\begin{align}
    \Delta u_s(x,y) &= \mathbb{P}(\hat{x}\geq \theta_s|X=x,Y=y,W=M,S=s)  - C_{M,s} \notag\\
    &\qquad- \left(\mathbb{P}(\hat{x}\geq \theta_s|X=x,Y=y,W=I,S=s)  - C_{I,s}\right) \notag\\
    & = \left(\mathbb{T}^y_{M,s}(\theta_s-x) - \mathbb{T}^y_{I,s}(\theta_s-x)\right) - \left(C_{I,s}-C_{M,s}\right)~.\label{eq:delta-u}
\end{align}
If this difference is positive, the agent will opt for $I$ over $M$. 
Recall that $\xindmi$ is the feature such that $\mathbb{T}^y_{M,s}(\theta_s-\xindmi) - \mathbb{T}^y_{I,s}(\theta_s-\xindmi) = C_{I,s}-C_{M,s}$. This is the point at which the agent is indifferent between the $M$ and $I$ actions. If $M$ is preferred to $I$ before this point, this is the point at which the agent would switch from $M$ to $I$, once $I$ has a non-negative utility (and vice versa for when $I$ is initially preferred to $M$). Using the above characterizations, together with Lemma \ref{lemma:order-indiff-points}, leads to the identified best responses in each case. 

\subsubsection{Proof of Lemma \ref{lemma:post-strategic-stats}}\label{app:proof-post-strategic}
We begin by noting that the new set of qualified ($\hat{Y}=1$) agents consists of previously qualified agents, as well as previously unqualified agents who opted for improvement decisions. Therefore, the new qualification rate is given by $\hat{\alpha}_s = \alpha_s + (1-\alpha_s) \int_{\mathbb{I}^0_s} G^0_s(x)\mathrm{d}x$. 

Next, note that the set of (now) qualified agents with feature ${x}$ consists of the previously qualified agents with the same old feature ${x}$ (who have opted for $w=N$), the previously qualified or unqualified agents with feature $x-b$ who improved to feature $x$ (i.e., opted for $w=I$ and got a boost realization $b$), and previously qualified agents with feature $x-b$ who chose manipulation and reached feature ${x}$ (i.e., opted for $w=M$ and got a boost realization $b$). Thus, 
\begin{align}
\hat{G}^1_s(x) =& \tfrac{\alpha_s}{\hat{\alpha}_s} \big(\mathds{1}(x\in\mathbb{N}^1_s)G^1_s(x) + \int_b \mathds{1}(x-b\in\mathbb{M}^1_s)G^1_s(x-b)\tau_{M,s}^1(b)\mathrm{d}b \notag\\
&+ \int_b \mathds{1}(x-b\in\mathbb{I}^1_s)G^1_s(x-b)\tau_{I,s}^1(b)\mathrm{d}b\big) + \tfrac{1-\alpha_s}{\hat{\alpha}_s} \int_b \mathds{1}(x-b\in\mathbb{I}^0_s)G^0_s(x-b)\tau_{I,s}^0(b)\mathrm{d}b.
\label{eq:post-strategic-feature-dist-1-proof}
\end{align}
We can re-write the expressions for the integrals using a change of variable $z:=x-b$ as follows:
\begin{align*}
    &\int_0^{\infty} \mathds{1}(x-b\in\mathbb{M}^1_s)G^1_s(x-b)\tau_{M,s}^1(b)\mathrm{d}b = \int_{x}^{-\infty} \mathds{1}(z\in\mathbb{M}^1_s)G^1_s(z)\tau_{M,s}^1(x-z)\mathrm{d}(x-z)\\
     = &\int_{-\infty}^x \mathds{1}(z\in\mathbb{M}^1_s)G^1_s(z)\tau_{M,s}^1(x-z)\mathrm{d}z \int_{z\in\mathbb{M}^1_s}G^1_s(z)\tau_{M,s}^1(x-z)\mathrm{d}z~.
\end{align*}
Substituting the above in Eq.\eqref{eq:post-strategic-feature-dist-1-proof} leads to Eq.\eqref{eq:post-strategic-G-1}. 

Using similar arguments, the post-strategic feature distribution of (now) unqualified agents from group $s$ consists of unqualified agents with the same feature who have opted to do nothing, and previously unqualified agents who have opted for manipulation and have reached feature $x$. Thus:
\begin{align}
\hat{G}^0_s(x) =& \tfrac{1-\alpha_s}{1-\hat{\alpha}_s} \big(\mathds{1}(x\in\mathbb{N}^0_s)G^0_s(x) + \int_b \mathds{1}(x-b\in\mathbb{M}^0_s)G^0_s(x-b)\tau_{M,s}^0(b)\mathrm{d}b\big)
\label{eq:post-strategic-feature-dist-0-proof}
\end{align}
Re-writing the integral similar to before leads to Eq.\eqref{eq:post-strategic-G-0}. 

\subsection{Proofs for Section~\ref{sec:firm}}\label{app:proofs-firms-optimal-thresholds}

\subsubsection{Proof of Lemma~\ref{lemma:nonstrategic-unfair-theta}}\label{app:proof-lemma-nonstrategic-unfair-theta} 
    The proof is similar to those in \cite{zhang2022fairness, liao2022social}. The unfair non-strategic firm's utility Eq.\eqref{eq:expected-utility}, and its first derivative, are given by
    \begin{align*}
       U(\theta^U_a, \theta^U_b) &= 
       \sum_s ~ u_+ \alpha_s  (1-\mathbb{G}^1_s(\theta^U_s)) - u_- (1-\alpha_s) (1-\mathbb{G}^0_s(\theta^U_s) ).\\
       \Rightarrow \quad \frac{\partial{U(\theta^U_s})}{\partial{\theta^U_s}} &= - u_+ \alpha_s  G^1_s(\theta^U_s) + u_- (1-\alpha_s) G^0_s(\theta^U_s).
    \end{align*}
where $\mathbb{G}^y_s$ denotes the CDF of the feature distribution of agents with true label $y$ from group $s$. Taking the second derivative, and considering Assumption~\ref{as:strict-monotonicity-assumption}, we can see that the firm's utility is convex in $\theta_s$. Therefore, by setting the first derivative to zero, the optimal thresholds satisfy   $\frac{G^1_s(\theta^U_s)}{G^0_s(\theta^U_s)} = \frac{u_- (1-\alpha_s)}{u_+ \alpha_s}$. 

\subsubsection{Proof of Lemma~\ref{lemma:strategic_unfair_theta}}\label{app:proof_lemma_strategic_unfair_theta}
    From the unfair strategic firm's utility Eq.\eqref{eq:strategic-utility-with-post-statistics} can be written as 
    \begin{align*}
        \hat{U}(\hat{\theta}^U_s) &= \sum_s ~~ u_+ \alpha_s  (1-\mathbb{G}^1_s(\hat{\theta}^U_s)) - u_- (1-\alpha_s) (1-\mathbb{G}^0_s(\hat{\theta}^U_s)) + u_+ \alpha_s \boldsymbol{\Phi}^1_{s,(i)}(\hat{\theta}^U_s) + u_+ \alpha_s \boldsymbol{\Psi}^1_{s,(i)}(\hat{\theta}^U_s) \\ 
        &+ u_+ (1-\alpha_s) \boldsymbol{\Phi}^0_{s,(i)}(\hat{\theta}^U_s) - u_- (1-\alpha_s) \boldsymbol{\Psi}^0_{s,(i)}(\hat{\theta}^U_s).\\
       \frac{\partial{\hat{U}(\hat{\theta}^U_s})}{\partial{\hat{\theta}^U_s}} &= - u_+ \alpha_s  G^1_s(\hat{\theta}^U_s) + u_- (1-\alpha_s) G^0_s(\hat{\theta}^U_s) + u_+ \alpha_s \boldsymbol{\Phi}_{s,(i)}^{'1} + u_+ \alpha_s \boldsymbol{\Psi}_{s,(i)}^{'1} \\
       &+ u_+ (1-\alpha_s) \boldsymbol{\Phi}_{s,(i)}^{'0} - u_- (1-\alpha_s) \boldsymbol{\Psi}_{s,(i)}^{'0}.
    \end{align*}
    where $\boldsymbol{\Phi}_{s,(i)}^{'y}$ and $\boldsymbol{\Psi}_{s,(i)}^{'y}$ are the first derivative of the improvement impact and manipulation impact for agents Y=y at $\hat{\theta}^U_s$, respectively. Therefore, by setting the above equation of the first derivative of the unfair strategic firm's utility to zero we have $\frac{\boldsymbol{\Phi_{s,(i)}^{'1}} + \boldsymbol{\Psi_{s,(i)}^{'1}} - G_s^1(\hat{\theta}^{U}_s) + \frac{(1-\alpha_s)}{\alpha_s} \boldsymbol{\Phi_{s,(i)}^{'0}}}{\boldsymbol{\Psi_{s,(i)}^{'0}} - G_s^0(\hat{\theta}^{U}_s)} = \frac{u_-(1-\alpha_s)}{u_+\alpha_s}. $

It remains to find the specific forms of the first derivatives $\boldsymbol{\Phi}_{s,(i)}^{'y}$ and $\boldsymbol{\Psi}_{s,(i)}^{'y}$ for each equilibrium type \emph{(i)}. 
We show this for agents' best-response Type 1, finding the expressions in the first row in Table~\ref{table:equilibria_ratio}. A similar procedure can be used to find the unfair strategic decision threshold for agents' best-response Type 2 and 3.

From Proposition~\ref{prop:agents-br-generic}, we know that the improvement set $\mathbb{I}^y_{s,(1)}$ is bounded by $\xindi$ and $\xrisk$ for Type 1 best-response, while the manipulation set $\mathbb{M}^y_{s, (i)}$ is bounded By $\xrisk$ and $\theta_s$.  Therefore, we have
    \begin{align*}
       \boldsymbol{\Phi}_{s,(1)}^y({\theta_s}) &= \int_{\xindi}^{\xrisk}  G^y_s(z) \left(1- \mathbb{T}^y_{I,s}({\theta_s} - z)\right) \mathrm{d}z, ~\text{and }
    \boldsymbol{\Psi}_{s,(1)}^y({\theta_s}) &= \int_{\xrisk}^{\theta_s} G^y_s(z)\left(1 - \mathbb{T}^y_{M,s}({\theta_s} - z)\right) \, \mathrm{d}z.
    \end{align*}
We next find the first derivative of these impacts with respect to $\theta_s$. Note that all the limits of the integration are functions of $\theta_s$; hence, we followed the Leibniz integral rule to find these derivatives. We also note that the derivatives of $\xindi$ and $\xrisk$ with respect to $\theta$ are 1. For the improvement impact, we have:
\[\boldsymbol{\Phi}_{s,(1)}^{'y}({\theta_s}) = {G}^y_s(\xrisk) \left(1- \mathbb{T}^y_{I,s}({\theta_s}- \xrisk)\right) - {G}^y_s(\xindi) \left( 1-\mathbb{T}^y_{I,s}({\theta_s} - \xindi)\right) -  \int_{\xindi}^{\xrisk}  G^y_s(z) \tau^y_I({\theta_s} - z) \mathrm{d}z ).\]
    From the definition of the indifference points in Definition~\ref{def:indiff-features}, we have $\mathbb{T}^y_{I,s}({\theta_s} - \xindi) = \mathbb{T}^y_{I,s}((\mathbb{T}^y_{I,s})^{-1}(1-C_{I,s})) = 1-C_{I,s}$. We also know, from Proposition~\ref{prop:agents-br-generic}, that the conditions for this equilibrium to occur are that $\xindmi \leq \xindi \leq \xindm \leq \xrisk$, and therefore, $1-\mathbb{T}^y_{I,s}({\theta_s} - \xrisk)=1$ (i.e., admission under $I$ is certain at $\xrisk$). Therefore: 
    \begin{align}\label{eq:phi_prime_type1}
        \boldsymbol{\Phi}_{s,(1)}^{'y}({\theta_s}) = {G}^y_s(\xrisk) - {G}^y_s(\xindi) C_{I,s}  -  (G^y_s*\tau^y_{I,s})({\theta_s}). 
    \end{align}
    Similarly, for the manipulation impact we have,\[
        \boldsymbol{\Psi}_{s,(1)}^{'y}({\theta_s}) = {G}^y_s(\theta_s)\left(1-\mathbb{T}^y_{M,s}({\theta_s} - {\theta_s})\right) - {G}^y_s(\xrisk) \left(1-\mathbb{T}^y_{M,s}({\theta_s} - \xrisk)\right) - \int_{\xrisk}^{\theta_s}  G^y_s(z) \tau^y_{M,s}({\theta_s} - z) \mathrm{d}z ). \]
    From the definition of the indifference points in Definition~\ref{def:indiff-features}, we have $\mathbb{T}^y_{M,s}({\theta_s} - \xrisk) = \mathbb{T}^y_{M,s}((\mathbb{T}^y_{M,s})^{-1}(C_{I,s}-C_{M,s})) = C_{I,s}-C_{M,s}$. Therefore,
    \begin{align}\label{eq:psi_prime_type1}
        \boldsymbol{\Psi}_{s,(1)}^{'y}({\theta_s}) = {G}^y_s(\theta_s) - {G}^y_s(\xrisk) (1-(C_{I,s}-C_{M,s}))   -   (G^y_s* \tau^y_{M,s})({\theta_s}). 
    \end{align}

By substituting these in the fraction determining the optimal unfair strategic threshold in the statement of Lemma~\ref{lemma:strategic_unfair_theta}, we can find the expressions given in Table~\ref{table:equilibria_ratio}. Note that the derivative of all indifference points defined in Proposition~\ref{prop:agents-br-generic} with respect to the decision threshold $\theta_s$ is 1, specifically, for the \textit{flip feature} $(\xindmi)$ which can be expressed as $\xindmi = \theta_s - (\mathbb{T}^y_{M,s})^{-1}\big( \mathbb{T}^y_{I,s}(\theta_s - \xindmi) + C_{I,s} - C_{M,s} \big).$ Let us define: $g(\theta_s, \xindmi) = \mathbb{T}^y_{I,s}(\theta_s - \xindmi) + C_{I,s} - C_{M,s},$ and $h(g(\theta_s, \xindmi)) = (\mathbb{T}^y_{M,s})^{-1}(g(\theta_s, \xindmi)).$ Then, the derivative of $\xindmi$ with respect to $\theta_s$ can be written as: \[
    \frac{\partial \xindmi}{\partial \theta_s} = 1 - \frac{\partial h(g(\theta_s, \xindmi))}{\partial \theta_s}. \text{
Next, the derivative of $h(g(\theta_s, \xindmi))$ with respect to $\theta_s$ is given by:}\]  
\begin{align*}
    \frac{\partial h(g(\theta_s, \xindmi))}{\partial \theta_s} &= 
\frac{1}{\tau^y_{M,s}\big((\mathbb{T}^y_{M,s})^{-1}(g(\theta_s, \xindmi))\big)} \cdot \frac{\partial g(\theta_s, \xindmi)}{\partial \theta_s}.~=  
\frac{\tau^y_{I,s}(\theta_s - \xindmi)}{\tau^y_{M,s}((\mathbb{T}^y_{M,s})^{-1}(g(\theta_s, \xindmi)))} \cdot \big(1 - \frac{\partial \xindmi}{\partial \theta_s}\big).
\end{align*} Simplifying further, and noting that $\tau^y_{M,s}((\mathbb{T}^y_{M,s})^{-1}(g(\theta_s, \xindmi))) = \tau^y_{M,s}(\theta_s - \xindmi)$, we have:  
\[
\frac{\partial h(g(\theta_s, \xindmi))}{\partial \theta_s} = 
\frac{\tau^y_{I,s}(\theta_s - \xindmi)}{\tau^y_{M,s}(\theta_s - \xindmi)} \cdot \big(1 - \frac{\partial \xindmi}{\partial \theta_s}\big).
\text{ Substituting this back into the expression for $\frac{\partial \xindmi}{\partial \theta_s}$} , \text{ we get:}\] \[  
\frac{\partial \xindmi}{\partial \theta_s} = 1 - \frac{\tau^y_{I,s}(\theta_s - \xindmi)}{\tau^y_{M,s}(\theta_s - \xindmi)} \cdot \big(1 - \frac{\partial \xindmi}{\partial \theta_s}\big). \text{ 
Rearranging terms yields:}  \frac{\partial \xindmi}{\partial \theta_s} = 1.\]

\subsubsection{Proof of Lemmas~\ref{lemma:non-strategic_fair_theta}-\ref{lemma:strategic_fair_theta}}\label{app:fair-threshold-lemmas-proofs} 

Similar to \cite{zhang2022fairness}, from the non-strategic firm's utility with the fairness constraint in Eq.\ref{eq:fair-expected-utility} the optimal threshold satisfies $\mathcal{C}^{f}_{a}(\theta^{\mathcal{C}}_{a}) =\mathcal{C}^{f}_{b}(\theta^{\mathcal{C}}_{b})$, thus we can write one of the groups' fair decision thresholds as $\theta^{\mathcal{C}}_{b} = (\mathcal{C}^{f}_{b})^{-1} ( \mathcal{C}^{f}_{a}(\theta^{\mathcal{C}}_{a}))$. The following holds:
    \begin{align*}
        \frac{\partial{{\theta}^{\mathcal{C}}_b}}{\partial{{\theta}^{\mathcal{C}}_a}} = \frac{\partial{(\mathcal{C}^{f}_{b})^{-1} ( \mathcal{C}^{f}_{a}(\theta^{\mathcal{C}}_{a}))}}{\partial{{\theta}^{\mathcal{C}}_a}} = \frac{\partial{(\mathcal{C}^{f}_{b})^{-1} ( \mathcal{C}^{f}_{a}(\theta^{\mathcal{C}}_{a}))}}{\partial{{\mathcal{C}^{f}_{a}(\theta^{\mathcal{C}}_{a})}}} \frac{\partial{ \mathcal{C}^{f}_{a}(\theta^{\mathcal{C}}_{a})}}{\partial{{\theta}^{\mathcal{C}}_a}} &= \frac{1}{(\mathcal{C}^{'f}_{b})\left( (\mathcal{C}^{f}_{b})^{-1} (\mathcal{C}^{f}_{a}(\theta^{\mathcal{C}}_{a}))\right)} \frac{\partial{ \mathcal{C}^{f}_{a}(\theta^{\mathcal{C}}_{a})}}{\partial{{\theta}^{\mathcal{C}}_a}} \\
        &= \frac{1}{(\mathcal{C}^{'f}_{b})\left( \theta^{\mathcal{C}}_b\right)} \frac{\partial{ \mathcal{C}^{f}_{a}(\theta^{\mathcal{C}}_{a})}}{\partial{{\theta}^{\mathcal{C}}_a}} = \frac{\frac{\partial{ \mathcal{C}^{f}_{a}(\theta^{\mathcal{C}}_{a})}}{\partial{{\theta}^{\mathcal{C}}_a}}}{\frac{\partial{ \mathcal{C}^{f}_{b}(\theta^{\mathcal{C}}_{b})}}{\partial{{\theta}^{\mathcal{C}}_b}}}.
    \end{align*}
    then the firms' utility can be written in terms of $\theta^{\mathcal{C}}_s$ as, 
    \begin{align*}
    U(\theta^{\mathcal{C}}_a,\theta^{\mathcal{C}}_b) &= n_a \left( u_+ \alpha_a (1- \mathbb{G}^1_a(\theta^{\mathcal{C}}_a)) - u_- (1-\alpha_a) (1-\mathbb{G}^0_a(\theta^{\mathcal{C}}_a)) \right)\\ &+ n_b \left( u_+ \alpha_b  (1-\mathbb{G}^1_b((\mathcal{C}^{f}_{b})^{-1} ( \mathcal{C}^{f}_{a}(\theta_{a})))) - u_- (1-\alpha_b) (1- \mathbb{G}^0_b((\mathcal{C}^{f}_{b})^{-1} ( \mathcal{C}^{f}_{a}(\theta_{a}))) )\right).\\
        \frac{\partial{U(\theta^{\mathcal{C}}_a,\theta^{\mathcal{C}}_b)}}{\partial{\theta^{\mathcal{C}}_a}} 
        &= n_a \left(- u_+ \alpha_a  G^1_a(\theta^{\mathcal{C}}_a) + u_- (1-\alpha_a) G^0_s(\theta^{\mathcal{C}}_a) \right) \\ &+ n_b \left( -u_+ \alpha_b G^1_b(\theta^{\mathcal{C}}_{b}))  \frac{\frac{\partial{ \mathcal{C}^{f}_{a}(\theta^{\mathcal{C}}_{a})}}{\partial{{\theta}^{\mathcal{C}}_a}}}{\frac{\partial{ \mathcal{C}^{f}_{b}(\theta^{\mathcal{C}}_{b})}}{\partial{{\theta}^{\mathcal{C}}_b}}} + u_- (1-\alpha_b) G^0_b(\theta^{\mathcal{C}}_{b})  \frac{\frac{\partial{ \mathcal{C}^{f}_{a}(\theta^{\mathcal{C}}_{a})}}{\partial{{\theta}^{\mathcal{C}}_a}}}{\frac{\partial{ \mathcal{C}^{f}_{b}(\theta^{\mathcal{C}}_{b})}}{\partial{{\theta}^{\mathcal{C}}_b}}} \right).
    \end{align*}
    By setting the first derivative to zero and re-arrange terms we have $\theta_s^{\mathcal{C}} $ satisfy: $\linebreak[4] \sum_s n_s\frac{u_+ \alpha_s G^1_s(\theta_s^{\mathcal{C}}) - u_- (1-\alpha_s)G^0_s(\theta_s^{\mathcal{C}}) }{\frac{\partial{\mathcal{C}^{f}_{s}(\theta_s^{\mathcal{C}})}}{\partial{\theta_s^{\mathcal{C}}}}} = 0$.

A similar approach can prove Lemma~\ref{lemma:strategic_fair_theta} using the strategic utility from~\eqref{eq:strategic-utility-with-post-statistics}. 

\subsection{Proofs for Section~\ref{sec:effect-strateg-prediction}}\label{app:proofs-anticipating-strategic}

\subsubsection{Proof of Proposition~\ref{prop:firm_impact_comp}}\label{sec:proof_Prop2_coro1}

For this part of the proof, we analyze the selected optimal unfair thresholds ($\theta^U_s, \hat{\theta}^U_s$) over different groups' qualification status ($\alpha_s$) and agents' best response types (Type 1 and 3).
 Let $\alpha_s = 0$ such that all agents are unqualified, the unfair non-strategic firm's utility in equation~\ref{eq:expected-utility} can be reduced to $U(\theta_s) = - u_- (1-\mathbb{G}^0_s(\theta_s))$, hence to achieve the maximum utility the optimal threshold is $\theta_s = \overline{x}^0_s$ which means rejecting all the unqualified agents. Moreover, Let $\alpha_s = 1$ such that all agents are qualified, the unfair non-strategic firm's utility in equation~\ref{eq:expected-utility} can be reduced to $U(\theta_s) = u_+ (1-\mathbb{G}^1_s(\theta_s))$, hence the optimal threshold is $\theta_s = \underline{x}^1_s$ which means accepting majority all qualified agents. 
 Note as the group qualification rate ($\alpha_s$) decreases, the unfair non-strategic optimal threshold increases to reject the more unqualified agents. To summarize, the unfair non-strategic optimal policy satisfies the following:  $[\underline{x}^1_s, \overline{x}^0_s]$. Under Assumption~\ref{as:relation_of_boost}, $\mathbf{o}^y_{M,s}$ and $\mathbf{r}^y_{s}$ are equal across agents, so we denote them as $\mathbf{o}_{M,s}$ and $\mathbf{r}_{s}$, omitting the $y$ label in this proof.
 To establish a unique optimal strategic utility under Type 3 agent responses, we analyze the derivative from Table~\ref{table:equilibria_ratio}:
\begin{align*}
\frac{\partial{\hat{U}(\hat{\theta}^U_s)}}{\partial{\hat{\theta}^U_s}} =& u_+ \alpha_s \left(- G^1_s(\mathbf{o}_{M,s}) C_{M,s} - (G^1_s * \tau_{M,s})(\hat{\theta}^U_s)\right) \\ &+ u_- (1-\alpha_s) \left(G^0_s(\mathbf{o}_{M,s}) C_{M,s} + (G^0_s * \tau_{M,s})(\hat{\theta}^U_s)\right).    
\end{align*}

For $\hat{\theta}^U_s \leq \mu^0_s$, the derivative is positive, led by $(G^0_s * \tau_{M,s})(\hat{\theta}^U_s)$. Per Proposition~\ref{prop:firm_impact_comp} ($\underline{x}^1_s - \underline{x}^0_s > (\mathbb{T}^y_{M,s})^{-1}(1-C_{M,s})$), $\mathbf{o}_{M,s} > \underline{x}^0_s$ before $\hat{\theta}^U_s > \underline{x}^1_s$, so $G^0_s(\mathbf{o}_{M,s})$ rises, boosting the derivative as $\hat{\theta}^U_s$ nears $\underline{x}^1_s$. Past $\underline{x}^1_s$, it drops as $(G^1_s * \tau_{M,s})(\hat{\theta}^U_s)$ grows, possibly turning negative or not. The fall persists as $\mathbf{o}_{M,s} > \mu^0_s$ approaches $\underline{x}^1_s$. At $\mathbf{o}_{M,s} > \overline{x}^0_s$, only negative terms $\left(- G^1_s(\mathbf{o}_{M,s}) C_{M,s} - (G^1_s * \tau_{M,s})(\hat{\theta}^U_s)\right)$ remain, forcing the derivative negative. As $\hat{\theta}^U_s \geq \overline{x}^1_s$, the $(G^1_s * \tau_{M,s})(\hat{\theta}^U_s)$ shrinks, lifting the derivative until $\mathbf{o}_{M,s} > \overline{x}^1_s$, where it reaches zero, proving a unique optimum.

Next, we bound the convolution terms by treating the convolution as a weighted average of $G^y_s(x)$ over the interval $\mathbb{W}^y_s$, with weights $\tau^y_{w,s}(\hat{\theta}^U_s-x)$. The upper bound is, 
\[(G^y_s * \tau_{w,s})(\hat{\theta}^U_s) \leq G^y_s(z_{max}) \int_{u\in\mathbb{W}^y_s} \tau^y_{w,s}(\hat{\theta}^U_s-u) \mathrm{d}u,\] where $z_{max} = \arg\max_{z\in\mathbb{W}^y_s} G^y_s(z)$. Furthermore, the lower bound is,
\[(G^y_s * \tau_{w,s})(\hat{\theta}^U_s) \geq G^y_s(z_{min}) \int_{u\in\mathbb{W}^y_s} \tau^y_{w,s}(\hat{\theta}^U_s-u) \mathrm{d}u,\] where $z_{min} = \arg\min_{z\in\mathbb{W}^y_s} G^y_s(z)$. For instance, under Type 1 equilibrium $\mathbb{I}^0_s = [\mathbf{o}^0_{I,s},\mathbf{r}_s]$, if $\mathbf{o}^0_{I,s}<\mathbf{r}_s < \mu^0_s$ then the upper bound $(G^0_s * \tau^0_{I,s})(\hat{\theta}^U_s) \leq G^0_s(\mathbf{r}_s)(1-C_{I,s})$. Moreover, $\mathbb{M}^1_s = [\mathbf{r}_s, \hat{\theta}^U_s]$, if $\mathbf{r}_s < \mu^1_s$ the lower bound $(G^1_s * \tau_{M,s})(\hat{\theta}^U_s) \geq G^1_s(\mathbf{r}_s)(C_{I,s}-C_{M,s})$.

For Type 1 agent best responses, the firms' utility first derivative from Table~\ref{table:equilibria_ratio} is:
\begin{align*}
 \frac{\partial{\hat{U}(\hat{\theta}^U_s)}}{\partial{\hat{\theta}^U_s}} =& u_+ \alpha_s \left(G^1_s(\mathbf{r}_s)(C_{I,s}-C_{M,s}) - G^1_s(\mathbf{o}^1_{I,s}) C_{I,s} - (G^1_s * \tau_{M,s})(\hat{\theta}^U_s) - (G^1_s * \tau^1_{I,s})(\hat{\theta}^U_s)\right)  \\& + u_+ (1-\alpha_s) \left(G^0_s(\mathbf{r}_s) - G^0_s(\mathbf{o}^0_{I,s})C_{I,s} - (G^0_s * \tau^0_{I,s})(\hat{\theta}^U_s)\right)  \\&+ u_- (1-\alpha_s) \left(G^0_s(\mathbf{r}_s) (1 - (C_{I,s} - C_{M,s})) + (G^0_s * \tau_{M,s})(\hat{\theta}^U_s)\right). 
 \end{align*} 

For $\hat{\theta}^U_s < \underline{x}^0_s + (\mathbb{T}_{M,s})^{-1}(C{I,s} - C_{M,s})$, the derivative is positive, driven by $(G^0_s * \tau_{M,s})(\hat{\theta}^U_s)$. It rises as $\hat{\theta}^U_s$ nears $\mu^0_s + (\mathbb{T}_{M,s})^{-1}(C_{I,s} - C_{M,s})$ due to increasing $G^0_s(\mathbf{r}_s)$, and per convolution bounds, $(G^0_s * \tau^0_{I,s})(\hat{\theta}^U_s) \leq G^0_s(\mathbf{r}_s)(1-C_{I,s}) < G^0_s(\mathbf{r}_s)$. Near $\underline{x}^1_s$, the derivative drops as $\mathbf{r}_s > \mu^0_s$ (positive terms shrink, while negative $(G^0_s * \tau^0_{I,s})(\hat{\theta}^U_s)$ grows with $\mathbb{I}^0_s$), possibly staying positive or turning negative. Beyond $\underline{x}^1_s$, it keeps decreasing as $(G^1_s * \tau_{M,s})(\hat{\theta}^U_s)>0$ grows. When $\underline{x}^1_s \leq \mathbf{r}_s \leq \mu^1_s$, bounds shows $(G^1_s * \tau_{M,s})(\hat{\theta}^U_s) \geq G^1_s(\mathbf{r}_s)(C_{I,s}-C_{M,s})$ and $(G^1_s * \tau^1_{I,s})(\hat{\theta}^U_s) > 0$, ensuring a decline. Given Proposition~\ref{prop:firm_impact_comp} ($\overline{x}^0_s +(\mathbb{T}^y_{M,s})^{-1}(C_{I,s}-C_{M,s})  < \mu^0_s + \underline{b}^0_{I,s}$), $\mu^1_s>\mathbf{r}_s > \overline{x}^0_s$ before $\mathbf{o}^0_{I,s} > \mu^0_s$, so $G^0_s(\mathbf{r}_s)$ and $(G^0_s * \tau_{M,s})(\hat{\theta}^U_s)$ hits zero, turning the derivative negative. As $\mathbf{o}^0_{I,s}$ approaching $\mu^0_s$, only negative terms persist, driving further decrease. When $\mathbf{o}^0_{I,s} > \mu^0_s$, $\left(-G^0_s(\mathbf{o}^0_{I,s})C_{I,s} - (G^0_s * \tau^0_{I,s})(\hat{\theta}^U_s)\right)$ shrinks, increasing the derivative, yet it stays negative since $(G^1_s * \tau_{M,s})(\hat{\theta}^U_s) \geq G^1_s(\mathbf{r}_s)(C_{I,s}-C_{M,s})$. As $\hat{\theta}^U_s$ rises, the derivative increases, still negative, with $\mathbf{o}^1_{I,s} < \mu^1_s < \mathbf{r}_s$, and bound $(G^1_s * \tau^1_{I,s})(\hat{\theta}^U_s) \geq G^1_s(\mu^1_s)(1-C_{I,s}) \geq G^1_s(\mathbf{r}_s)(1-C_{I,s})$. Then, as $G^0_s(\mathbf{o}^0_{I,s}) = 0$ while $\mathbf{o}^1_{I,s} < \mu^1_s$, causing the derivative to decrease again. As $G^1_s(\mathbf{r}_s) = 0$ and $\mathbf{o}^1_{I,s}$ nears $\mu^1_s$, it continues decreasing until $\mathbf{o}^1_{I,s} > \mu^1_s$, then rises until $G^1_s(\mathbf{o}^1_{I,s}) = 0$, hitting zero, confirming a unique optimum. 
      
Under Type 1 agent best responses, when $\alpha_s = 0$, the unfair strategic firm's utility (Eq.~\ref{eq:strategic-utility-with-post-statistics}) simplifies to $\hat{U}(\hat{\theta}^U_s) = u_+ \Phi^0_{s,(1)}(\hat{\theta}^U_s) - u_- (1 - \mathbb{G}^0_s(\hat{\theta}^U_s)) - u_- \Psi^0_{s,(1)}(\hat{\theta}^U_s)$. At $\hat{\theta}^U_s = \overline{x}^0_s + (\mathbb{T}^0_{M,s})^{-1}(C_{I,s} - C_{M,s})$, it is $\hat{U}(\overline{x}^0_s + (\mathbb{T}^0_{M,s})^{-1}(C_{I,s} - C_{M,s})) = u_+ \Phi^0_{s,(1)}(\overline{x}^0_s + (\mathbb{T}^0_{M,s})^{-1}(C_{I,s} - C_{M,s}))$. Notably, utility decreases for $\hat{\theta}^U_s > \overline{x}^0_s + (\mathbb{T}^0_{M,s})^{-1}(C_{I,s} - C_{M,s})$, where $|\mathbb{I}^0_{s,(1)}|$ declines as $\xrisk > \overline{x}^0_s$. Thus, the optimal unfair strategic threshold $\hat{\theta}^U_s \leq \overline{x}^0_s + (\mathbb{T}^0_{M,s})^{-1}(C_{I,s} - C_{M,s})$
   
Next we show that the impacts of the strategic behavior are single peaked. For a fixed value of \(\theta_s\), let 
$h^y_{M,s}(z; \theta_s) := G^y_s(z) \left( 1 - \mathbb{T}^y_{M,s}(\theta_s - z) \right).$ The function \(h^y_{M,s}(z; \theta)\) is single-peaked at $z^{*}$ because \(G^y_s(z)\) is symmetric, and \(\left( 1 - \mathbb{T}^y_{M,s}(\theta - z) \right) \in [0,1]\). Moreover, \[\left( 1 - \mathbb{T}^y_{M,s}(\theta - z') \right) =
\begin{cases} 
1, & \text{if } z' \geq \theta - \underline{b}^y_{M,s}, \\
0, & \text{if } z' \leq \theta - \overline{b}^y_{M,s}, \\
(0, 1), & \text{otherwise}.
\end{cases}\]

We can re-write $\boldsymbol{\Psi}^y_{s,(3)}(\theta_s)= \int_{\mathbf{o}_{M,s}}^{\theta_s} h^y_{M,s}(z;\theta) \mathrm{d}z$, and because the integration limits depends on $\theta_s$, the following holds: 
If $ \mathbf{o}_{M,s} < \theta_s \leq z^{*}, \int_{\mathbf{o}_{M,s}}^{\theta_s} h^y_{M,s}(z;\theta_s)\mathrm{d}z $ increases as $\theta_s$ approaches the $z^{*}$, and If  $z^{*} \leq \mathbf{o}_{M,s} < \theta_s, \int_{\mathbf{o}_{M,s}}^{\theta_s} h^y_{M,s}(z;\theta_s)\mathrm{d}z$ decreases as $\mathbf{o}_{M,s}$ moves away from the $z^{*}$. Therefore, $\boldsymbol{\Psi}^y_{s,(3)}(\theta_s)$ is also single-peaked. Similarly, $\boldsymbol{\Psi}^y_{s,(i)}(\theta_s)$ for $i \in \{1,2,3\} $, and $\boldsymbol{\Phi}^y_{s,(i)}(\theta_s) $ for $i \in \{1,2\}$  are also single peaked. 

We define conditions for the peak values of $\boldsymbol{\Phi}^y_{s,(i)}(\theta_s)$ and $\boldsymbol{\Psi}^y_{s,(i)}(\theta_s)$, $y \in {0,1}$. For $\arg \max_{\theta_s} \boldsymbol{\Psi}^y_{s,(i)}(\theta_s)$, $i \in {1,3}$, $\theta_s \leq \mu^y_s + \underline{b}_{M,s}$ and $\xrisk < \mu^y_s$ ensure agents at $\mu^y_s$ \emph{manipulate} and are accepted with minimal boost $\underline{b}_{M,s}$. For $\arg \max_{\theta_s} \boldsymbol{\Phi}^y_{s,(1)}(\theta_s)$, $\theta_s \leq \mu^y_s + \underline{b}^y_{I,s}$ and $\xrisk > \mu^y_s$ compel agents at $\mu^y_s$ to \emph{improve} and gain acceptance with minimal boost. Given $\overline{x}^0_s - \underline{x}^0_s < \xrisk - \mathbf{o}^0_{I,s}$, if $\theta_s > \overline{x}^0_s + (\mathbb{T}^y_{M,s})^{-1}(C_{I,s} - C_{M,s})$ ($\xrisk > \overline{x}^0_s$), $|\mathbb{I}^0_{s,(1)}|$ remains unchanged, so $\arg \max_{\theta_s} \boldsymbol{\Phi}^0_{s,(1)}(\theta_s) \in [\overline{x}^0_s + (\mathbb{T}^y_{M,s})^{-1}(C_{I,s} - C_{M,s}), \mu^0_s + \underline{b}^0_{I,s}]$.

Given the analysis of optimal thresholds, we first order the policies as in Proposition~\ref{prop:firm_impact_comp} and then rank their impacts as per parts (i), (ii), and (iii) of the same proposition. The maximum value of the optimal unfair strategic threshold, $\hat{\theta}^U_s$, under Type 1 agent responses is $\hat{\theta}^U_s \leq \overline{x}^0_s + (\mathbb{T}^0_{M,s})^{-1}(C_{I,s} - C_{M,s}) < \mu^0_s + \underline{b}^0_{I,s} < \mu^1_s$. Thus, $\hat{\theta}^U_s$ is always less than $\arg\max \boldsymbol{\Phi}^y_{s,(1)}(\theta_s)$ and $\arg\max \boldsymbol{\Psi}^1_{s,(1)}(\theta_s)$, placing it in the first half of both $\boldsymbol{\Phi}^y_{s,(1)}(\theta_s)$ and $\boldsymbol{\Psi}^1_{s,(1)}(\theta_s)$. 
The minimum $\hat{\theta}^U_s > \mu^0_s + \underline{b}_{M,s}>\arg\max \boldsymbol{\Psi}^0_{s,(1)}(\theta_s)$ (second half of $\boldsymbol{\Psi}^0_{s,(1)}(\theta_s))$.

For the optimal non-strategic threshold, $\theta^U_s < \overline{x}^0_s < \mu^0_s + \underline{b}^0_{I,s} < \mu^1_s$, it is always less than $\arg\max \linebreak[4] \boldsymbol{\Phi}^y_{s,(1)}(\theta_s)$ and $\arg\max \boldsymbol{\Psi}^1_{s,(1)}(\theta_s)$, residing in the first half of $\boldsymbol{\Phi}^y_{s,(1)}(\theta_s)$ and $\boldsymbol{\Psi}^1_{s,(1)}(\theta_s)$. Its minimum, $\theta^U_s > \underline{x}^1_s > \mu^0_s + \underline{b}_{M,s}$, exceeds $\arg\max \boldsymbol{\Psi}^0_{s,(1)}(\theta_s)$, placing it in the second half of $\boldsymbol{\Psi}^0_{s,(1)}(\theta_s)$. Since the optimal strategic threshold always exceeds the non-strategic one, and both lie in the first half of $\boldsymbol{\Phi}^y_{s,(1)}(\theta_s)$ and $\boldsymbol{\Psi}^1_{s,(1)}(\theta_s)$, their impacts can be ordered as $\boldsymbol{\Phi}^y_{s,(1)}(\theta^U_s) < \boldsymbol{\Phi}^y_{s,(1)}(\hat{\theta}^U_s)$ and $\boldsymbol{\Psi}^1_{s,(1)}(\theta^U_s) < \boldsymbol{\Psi}^1_{s,(1)}(\hat{\theta}^U_s)$. Additionally, as both reside in the second half of $\boldsymbol{\Psi}^0_{s,(1)}(\theta_s)$, their impacts order as $\boldsymbol{\Psi}^0_{s,(1)}(\hat{\theta}^U_s) < \boldsymbol{\Psi}^0_{s,(1)}(\theta^U_s)$.

Following the same analysis for Type 3 agent responses, the lowest possible optimal unfair strategic threshold, $\hat{\theta}^U_s > \mu^0_s + \underline{b}_{M,s}$, exceeds $\arg\max \boldsymbol{\Psi}^0_{s,(3)}(\theta_s)$, placing it in the second half of $\boldsymbol{\Psi}^0_{s,(3)}$. 
The highest possible $\hat{\theta}^U_s$, $\overline{x}^0_s + (\mathbb{T}_{M,s})^{-1}(1 - C_{M,s}) < \mu^1_s + \underline{b}_{M,s}$, is less than $\arg\max \boldsymbol{\Psi}^1_{s,(3)}(\theta_s)$. Since the non-strategic optimal analysis mirrors the prior section, we can order the impacts as follows: $\boldsymbol{\Psi}^0_{s,(3)}(\hat{\theta}^U_s) < \boldsymbol{\Psi}^0_{s,(3)}(\theta^U_s)$ and $\boldsymbol{\Psi}^1_{s,(1)}(\theta^U_s) < \boldsymbol{\Psi}^1_{s,(1)}(\hat{\theta}^U_s)$.
\subsubsection{Impacts of relaxing the assumptions of Proposition~\ref{prop:firm_impact_comp}}\label{app:relax_assumptions_prop2}

We start with assumptions $\overline{x}^0_s +(\mathbb{T}^y_{M,s})^{-1}(C_{I,s}-C_{M,s})  < \mu^0_s + \underline{b}^0_{I,s}$, and $ \overline{x}^0_s - \underline{x}^0_s < \xrisk - \mathbf{o}^0_{s,I}$ imposed within the proposition. If these are violated, for $\alpha_s < \xi =  \frac{u_-}{u_- + u_+ }$ (which, intuitively, means a majority-unqualified population of agents),  part (iv) of the proposition would be flipped: $\boldsymbol{\Phi}^0_{s,(1)}(\theta^{U}_s) > \boldsymbol{\Phi}^0_{s,(1)}(\hat{\theta}^{U}_s)$. This is because under the affordable improvement (Type 1) equilibrium and with a majority-unqualified population of agents, the strategic firm has to increase the threshold considerable to severely restrict successful manipulation by the many unqualified agents, which in turn also reduces the number of unqualified agents, who can afford to opt for improvement compared to the non-strategic firm's case. 
This is illustrated in Figure~\ref{fig:beyond_prop_2_(ii),(iii)}, which compares the strategic behavior impact of all agents' at the optimal policies among different levels of $\alpha_s$, the difference from Proposition~\ref{prop:firm_impact_comp} is shown in the fourth picture from the left. 
\begin{figure}[htb]
    \centering
    \includegraphics[width=1\linewidth]{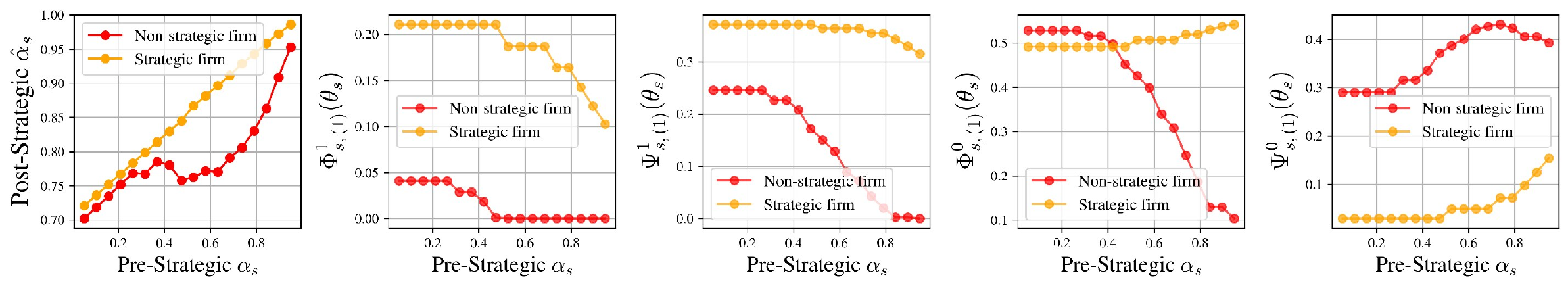}
    \caption{Impacts of removing  $\overline{x}^0_s + (\mathbb{T}^y_{M,s})^{-1}(C_{I,s}-C_{M,s})  < \mu^0_s + \underline{b}^0_{I,s}$, and $ \overline{x}^0_s - \underline{x}^0_s < \xrisk - \mathbf{o}^0_{s,I}$ from Prop.~\ref{prop:firm_impact_comp}.}
    \label{fig:beyond_prop_2_(ii),(iii)}
\end{figure}

Next, if Assumption~\ref{as:sym_bounded_dist} is relaxed (i.e., $\underline{x}^1_s < \mu^0_s < \mu^1_s < \overline{x}^0_s$), the conditions of Proposition~\ref{prop:firm_impact_comp} no longer hold. Specifically, this relaxation always violates $\overline{x}^0_s + (\mathbb{T}^y_{M,s})^{-1}(C_{I,s} - C_{M,s}) < \mu^1_s$, and $\underline{x}^1_s > \mu^0_s + \underline{b}_{M,s}$ either alone or alongside the other conditions such as $\mu^0_s+\underline{b}^0_{I,s}<\mu^1_s$ and/or $\overline{x}^0_s - \underline{x}^0_s < \xrisk - \mathbf{o}^0_{s,I}$. For $\alpha_s < \xi = \frac{u_-}{u_- + u_+}$ (a majority-unqualified population), part (iv) of Proposition~\ref{prop:firm_impact_comp} flips, consistent with the prior analysis, yielding $\boldsymbol{\Phi}^0_{s,(1)}(\theta^{U}s) > \boldsymbol{\Phi}^0_{s,(1)}(\hat{\theta}^{U}s)$. Here, the strategic firm imposes a high threshold to curb successful manipulation by the many unqualified agents, but this also limits successful manipulation by the few qualified agents. Consequently, part (iii) of the proposition reverses ($\boldsymbol{\Psi}^1_{s,(1)}(\theta^{U}s) > \boldsymbol{\Psi}^1_{s,(1)}(\hat{\theta}^{U}_s)$). 

Conversely, when $\alpha_s > \xi$ (indicating a majority-qualified population), firms adjust their thresholds to accommodate the many qualified agents, aligning with the reasoning in Proposition~\ref{prop:firm_impact_comp}. The non-strategic firm, unaware of agents’ strategic behavior, lowers its threshold substantially, reducing the need for strategic actions by agents as more are accepted by default. In contrast, the strategic firm, which accounts for this behavior, lowers its threshold less aggressively. This smaller reduction ultimately increases both the number of agents who successfully manipulate to gain acceptance and the number of unqualified agents who can afford to opt for improvement. Accordingly, part (ii) of Proposition~\ref{prop:firm_impact_comp}  flips yielding $\boldsymbol{\Psi}^0_{s,(1)}(\theta^{U}s) < \boldsymbol{\Psi}^0_{s,(1)}(\hat{\theta}^{U}_s)$. This is depicted in Figure~\ref{fig:beyond_ass_3}, which evaluates the impact of strategic behavior across all agents under optimal policies for different levels of $\alpha_s$. The deviation from Proposition~\ref{prop:firm_impact_comp} is evident in the third through fifth pictures from the left.

\begin{figure}[htb]
    \includegraphics[width=1\linewidth]{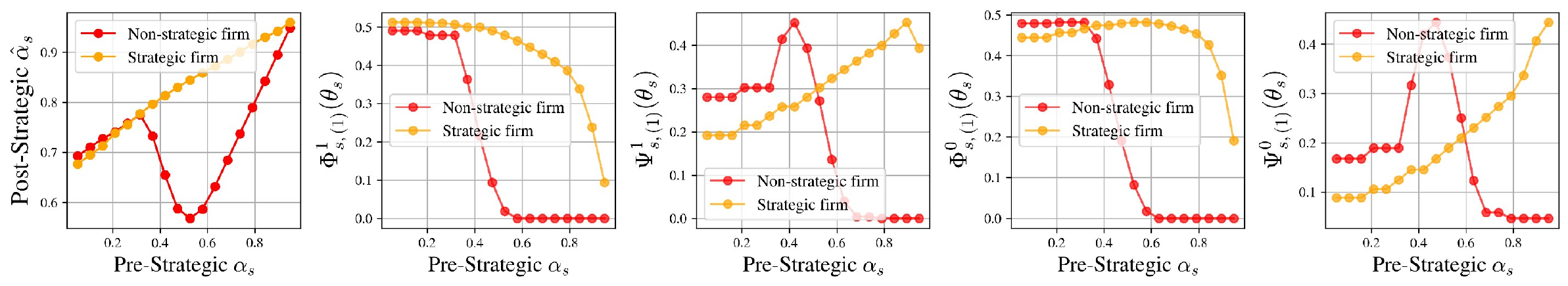}
    \caption{Impacts of removing Assumption~\ref{as:sym_bounded_dist} from Proposition~\ref{prop:firm_impact_comp}.}
    \label{fig:beyond_ass_3}
\end{figure}

\subsubsection{Detailed Intuitive interpretation of Proposition~\ref{prop:firm_impact_comp}}\label{sec:intuition_prop2}
In this appendix, we provide a walkthrough interpretation using the illustrations in Figure~\ref{fig:startegic_impact_agents_behavior}.  
In this figure, the horizontal, colored bars show the different regions of agents' strategic best-responses under the costly improvement (Type 3) and affordable improvement (Type 1) equilibria (as shown in Proposition~\ref{prop:agents-br-generic}), for the non-strategic $\theta^U_s$ and strategic $\hat{\theta}^U_s$ policies. The curves above each plot show the pre-strategic feature distribution of agents with qualification state $Y$. The lightly shaded green/red ovals are used to further highlight the agents that opt for improvement/manipulation in each case.

\begin{figure}[th]
    \centering
    \begin{minipage}{0.5\textwidth}
        \centering
        \includegraphics[width=\textwidth]{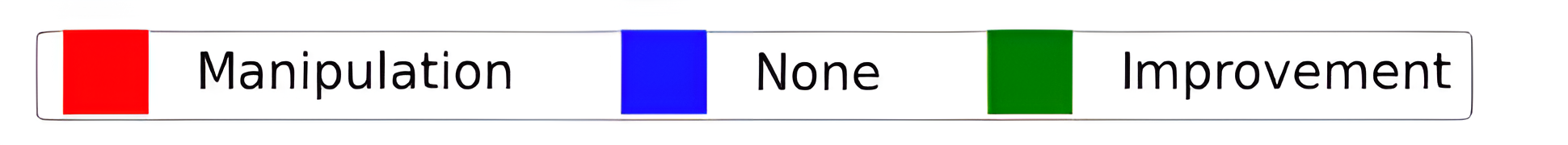}
    \end{minipage}
    \vspace{0.1cm}
\begin{minipage}{\textwidth}
        \centering
    \begin{subfigure}{0.5\textwidth}
        \centering
        \includegraphics[width=\textwidth,trim={0 0 0 0},clip]{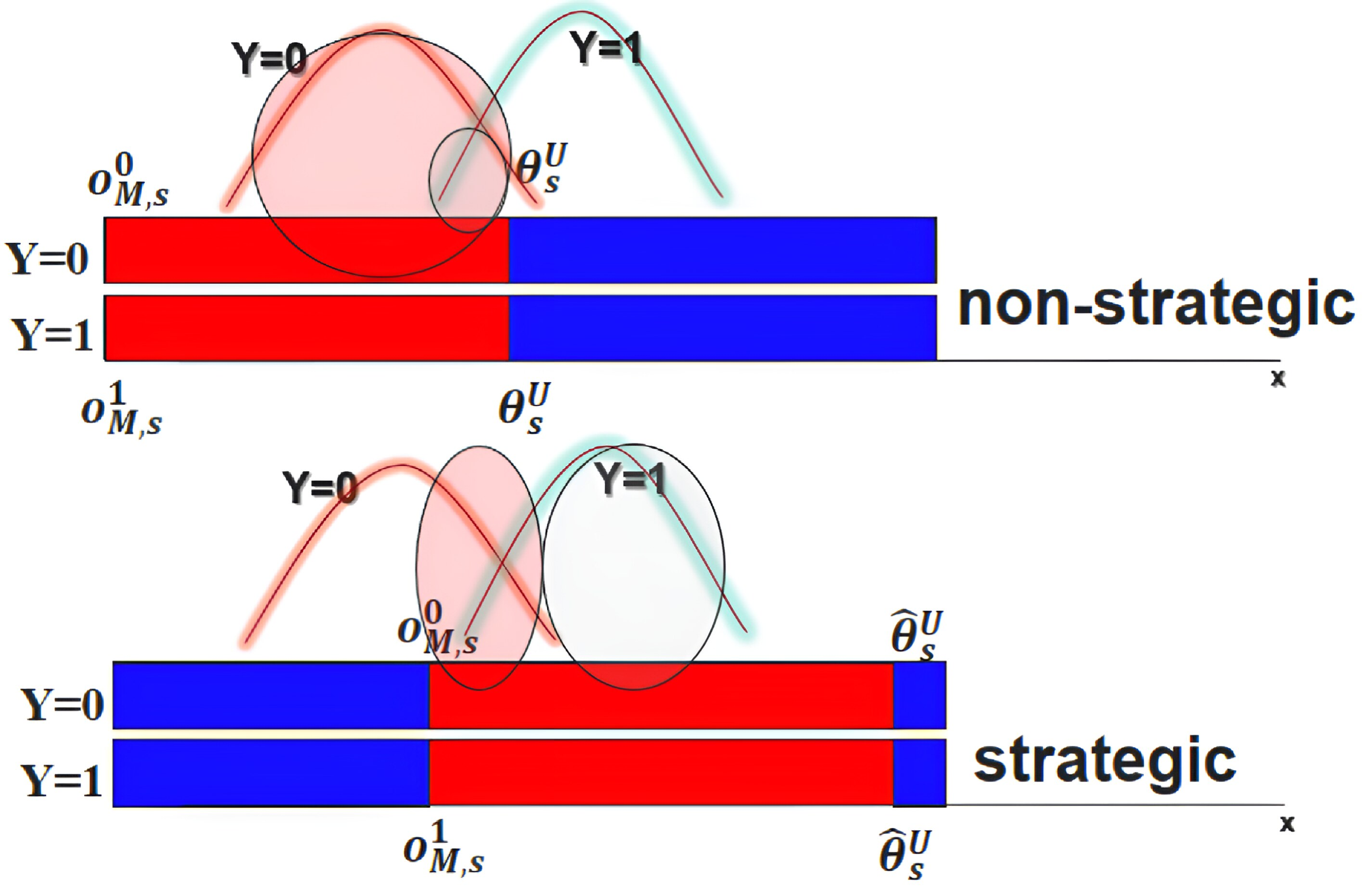}
        \caption{Costly improvement (Type 3) equilibrium}
        \label{fig:type3_startegic_impact_agents_behavior}
    \end{subfigure}
    \hfill
    \begin{subfigure}{0.49\textwidth}
        \centering
        \includegraphics[width=\textwidth,trim={0 0 0 0},clip]{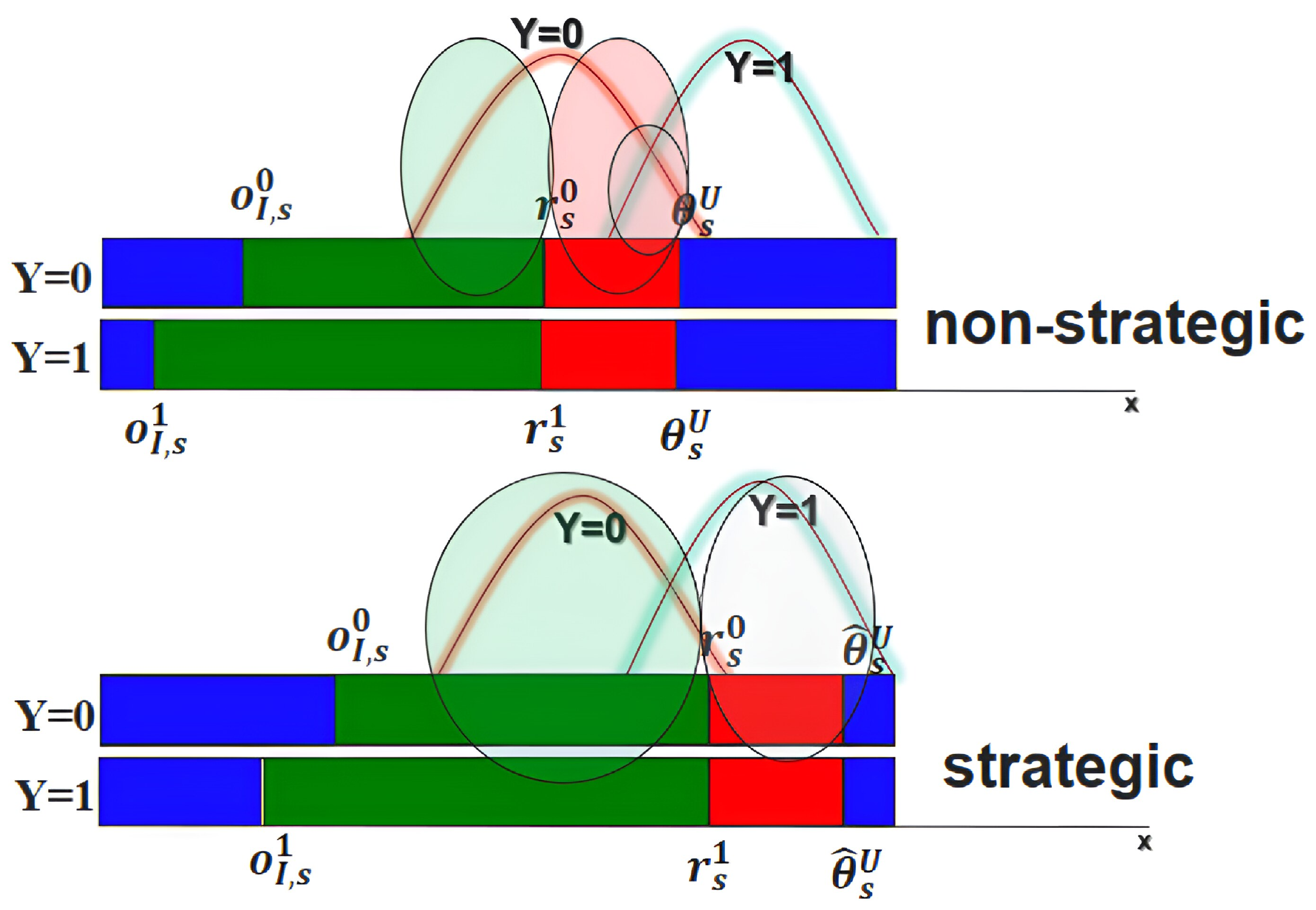}
        \caption{Affordable improvement (Type 1) equilibrium}
        \label{fig:type1_startegic_impact_agents_behavior}
    \end{subfigure}
\end{minipage}
\vspace{0.1in}
    \caption{The effect of the non-strategic and strategic policies on agents' strategic behavior in different equilibrium types. The curves represent the label $Y$ agents' feature distributions.}
    \label{fig:startegic_impact_agents_behavior}
\end{figure}

\paragraph{The costly improvement (Type 3) equilibrium.} We start with settings where a Type 3 (manipulation only) equilibrium emerges (e.g., when the improvement cost is too high). A nonstrategic firm, oblivious to the possibility of agents responding strategically, chooses a threshold $\theta^U_s$ solely with the goal of rejecting unqualified agents; this is seen in the top panel of Figure~\ref{fig:type3_startegic_impact_agents_behavior}. However, as agents are strategic, they will opt for manipulation to increase their chances of surpassing this threshold. Therefore, a non-strategic firm ends up admitting a considerable number of unqualified ($Y=0$) agents (along with a small number of qualified ($Y=1$) agents) who are below the threshold but pass it by opting for manipulation. 

A strategic firm accounts for this effect, and chooses a higher $\hat{\theta}^U_s$; this is asserted in part (i) of Proposition~\ref{prop:firm_impact_comp}. The higher threshold accommodates less manipulation by unqualified ($Y=0$) agents who are now too far from the threshold to benefit from manipulation (as supported by part (ii) of  Proposition~\ref{prop:firm_impact_comp}). At the same time, the higher threshold increases the number of qualified ($Y=1$) agents who are no longer accepted by default, and therefore leads to more qualified agents opting for manipulation compared to the non-strategic firm's case (as seen in part (iii) of Proposition~\ref{prop:firm_impact_comp}); this is also illustrated in the bottom subplot in Figure~\ref{fig:type3_startegic_impact_agents_behavior}. This is however the firm's desired strategic response by agents, as manipulation by qualified agents, which leads to their acceptance, also \emph{benefits} the firm and aligns with its objective.

\paragraph{The affordable improvement (Type 1) equilibrium.} We next turn to problem settings where Type 1 (mix of improvement and manipulation) responses emerge, as, e.g., improvement costs are sufficiently low; see Figure~\ref{fig:type1_startegic_impact_agents_behavior}. Here, the strategic firm again chooses a higher threshold than the non-strategic firm. Part of this is for the same reasons outlined earlier: by choosing a higher threshold, the strategic firm takes away manipulation opportunities from unqualified agents (who would otherwise reduce the firm's utility by getting admitted through manipulation), but increases manipulation by qualified agents (a choice that is non-detrimental to the firm). That said, the more interesting observations, which are new to our model compared to prior work, are the impacts of this increase on the agents' improvement choices. In the non-strategic firm scenario, a lower threshold leads to only \emph{some} of the unqualified ($Y=0$) agents opting for improvement, with a significant number opting for manipulation (see top panel of Figure~\ref{fig:type1_startegic_impact_agents_behavior}). In contrast, the increased threshold set by a strategic firm drives the \emph{majority} of the unqualified agents to opt for improvement as the only viable option to get admitted (see bottom panel of Figure~\ref{fig:type1_startegic_impact_agents_behavior}). This is the main driver of the increase in the strategic firm's utility, and is supported by part (iv) of Proposition~\ref{prop:firm_impact_comp}. We also note that the same observation is true for qualified ($Y=1$) agents.

\section{Agents' Strategic Behavior Illustration: uniform boost distributions}\label{sec:best-response-uniform}

To further illustrate the intuition behind the types of best-responses identified in Proposition~\ref{prop:agents-br-generic}, we consider uniform boost distributions $\mathbb{\tau}^y_{w,s}$. It is straightforward to verify that the three possible equilibria of Proposition~\ref{prop:agents-br-generic} can be obtained by varying improvement cost $C_I$ relative to manipulation cost $C_M$. Specifically, consider $\bar{b}_M^y-\underline{b}_M^y \geq \bar{b}_I^y-\underline{b}_I^y$. Then:
\begin{itemize}
    \item \textbf{Low improvement cost}: If $C_M \leq C_I \leq \frac{\bar{b}_M^y-\underline{b}_M^y}{\bar{b}_I^y-\underline{b}_I^y} C_M + \frac{\bar{b}_I^y - \bar{b}_M^y}{\bar{b}_I^y-\underline{b}_I^y}$, the best-response is of Type 1 in Proposition~\ref{prop:agents-br-generic}. In this case, the improvement cost is relatively low, so coupled with its higher efficacy, agents find it beneficial to opt for improvement before manipulation becomes beneficial. Ultimately, however, the lower manipulation cost leads agents to change their decision closer to the threshold, once uncertainties about receiving a positive outcome are low enough. 

    \item \textbf{Moderate improvement cost}: If $\frac{\bar{b}_M^y-\underline{b}_M^y}{\bar{b}_I^y-\underline{b}_I^y} C_M + \frac{\bar{b}_I^y - \bar{b}_M^y}{\bar{b}_I^y-\underline{b}_I^y} \leq C_I \leq C_M + \frac{\underline{b}_I^y - \underline{b}_M^y}{\bar{b}_M^y-\underline{b}_M^y}$, the best-response is of Type 2 in Proposition~\ref{prop:agents-br-generic}. Here, improvement costs are too high to benefit the agents who are far from the decision threshold, but low enough so that agents opt for improvement over manipulation when both actions are uncertain.

    \item \textbf{High improvement cost}: If $ C_M + \frac{\underline{b}_I^y - \underline{b}_M^y}{\bar{b}_M^y-\underline{b}_M^y} \leq C_I$, the best-response is of Type 3 in Proposition~\ref{prop:agents-br-generic}. In this case, the improvement cost is significantly higher than the manipulation cost, leading all agents to pick manipulation over improvement
    once profitable. 
\end{itemize} 
\section{FICO Data and Additional Experiments}
\subsection{Empirical FICO statistics}\label{app:fico_data}
Credit scores—and in particular, FICO scores—serve as a standard benchmark for evaluating creditworthiness in the United States. Building on publicly available resources from prior study \citep{hardt2016equality,xie2024learning,zhang2022fairness}, we leverage both the released data and accompanying code to implement the $(x,y)$ data-generating process. Specifically, we employ a preprocessed dataset that provides the cumulative distribution functions (CDFs) of scores $P(X \leq x \mid S=s)$ (see Figure~\ref{fig:pdf_cdf_fico}), qualification likelihoods $P(Y=1 \mid X=x, S=s)$ (see Figure~\ref{fig:qualification_likeli}), and the corresponding group-level qualification rates $\alpha_s$. The dataset includes four racial groups—African American (AA), Hispanic (H), Caucasian (C), and Asian (A)—with qualification rates $\alpha_s \in (0.33849, 0.56977, 0.75972, 0.80467)$.

We construct an empirical simulator from the observed score and repayment data. Using each group’s discrete score distribution (i.e., the empirical histogram), we draw $100{,}000$ scores with sampling probabilities determined by the histogram bin weights. For each sampled score $x$, we retrieve the group-specific repayment probability $P(Y=1 \mid X=x, S=s)$ and generate the binary outcome $Y \sim \mathrm{Bernoulli}\!\left(P(Y=1 \mid X=x, S=s)\right)$. This process yields an empirical labeled dataset $\{(X_i, Y_i)\}$ for each group, from which we form two subpopulations: $\mathcal{X}^{(1)}_s = \{X_i : Y_i = 1\}$ and $\mathcal{X}^{(0)}_s = \{X_i : Y_i = 0\}$.

All scores are normalized to the range $[0,1]$. We sample $(x,y)$ pairs directly from these empirical histograms—without imposing parametric assumptions—to preserve the original score distributions. These data are reproduced directly from the publicly available FICO-based simulators of \citet{xie2024learning} and serve as descriptive statistics of the population used in our experiments (see Figure~\ref{fig:group_fico_score_dist}).

\begin{figure}[htbp!]
    \centering
    \begin{subfigure}[t]{0.55\textwidth}
        \centering
        \includegraphics[width=\linewidth]{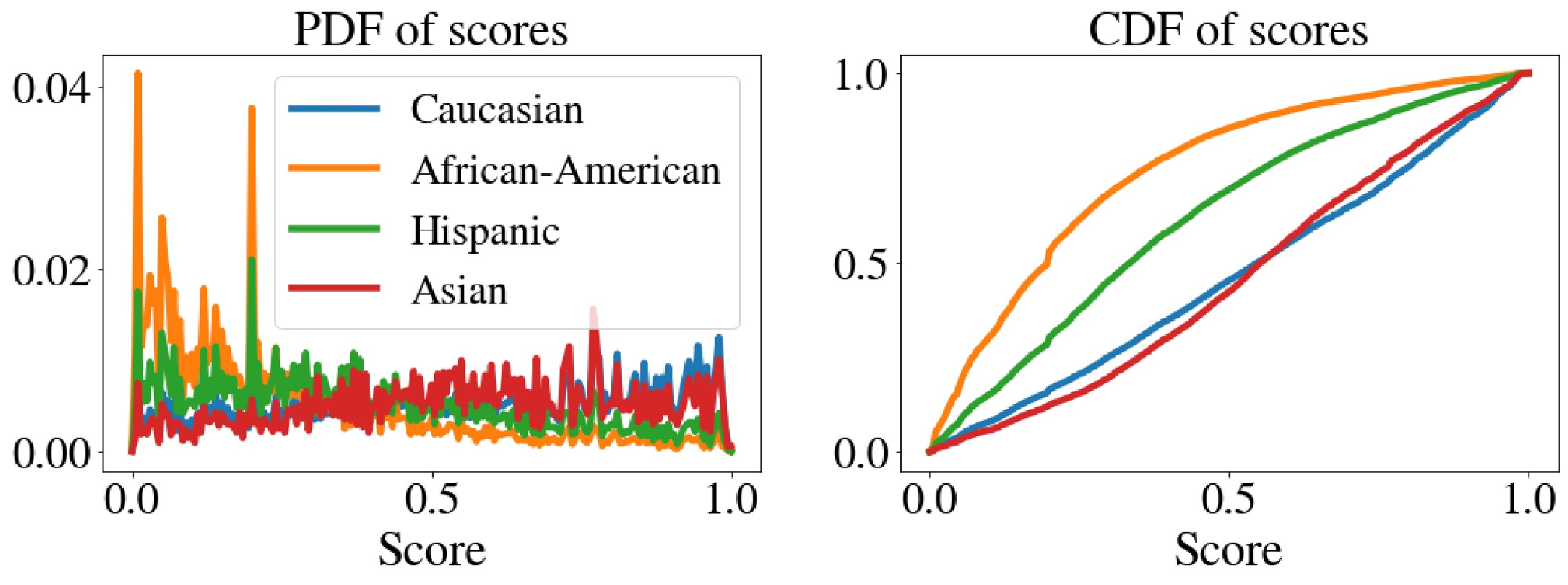}
        \subcaption{FICO score distributions (PDF and CDF).}
        \label{fig:pdf_cdf_fico}
    \end{subfigure}
    \quad 
    \begin{subfigure}[t]{0.3\textwidth}
        \centering
        \includegraphics[width=\linewidth]{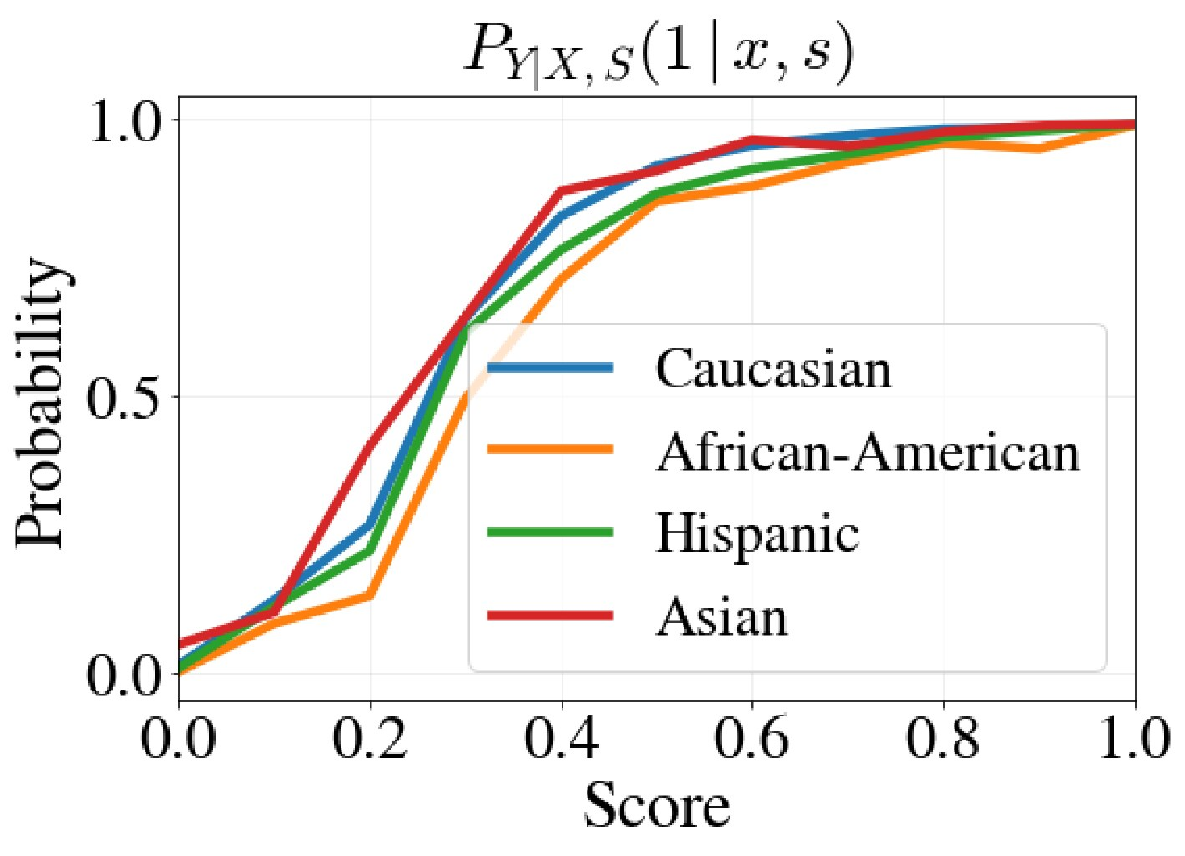}
        \subcaption{Group-wise qualification probability}
        \label{fig:qualification_likeli}
    \end{subfigure}
    \caption{FICO score distributions (a) and group-wise qualification probabilities (b).}
    \label{fig:fico_qual_combined}
\end{figure}

\begin{figure}[tbp!]
    \centering
    \includegraphics[width=\linewidth]{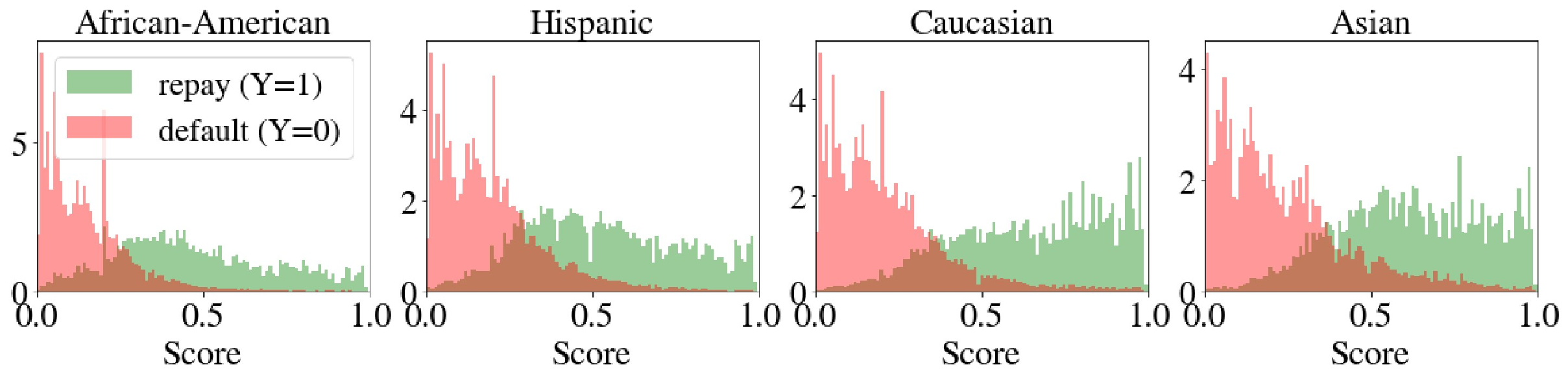}
    \caption{Histogram of score distributions across racial groups.}
    \label{fig:group_fico_score_dist}
\end{figure}

\subsection{Additional Experiments on FICO Data}
\subsubsection{Strategic vs non-strategic policies impact}\label{app:Prop2_all_groups}
\paragraph{Experimental setup for each demographic group.}
We compare policies' impact, without considering fairness, across the demographic groups African American (AA), Caucasian (C), Asian (A), and Hispanic (H). For each group $s$ with qualification rate $\alpha_s \in (0.33849, 0.56977, 0.75972, 0.80467)$, we generate a single replicate by sampling $n=1000$ agents: specifically, $n \alpha_s$ agents are drawn from the empirical score distribution of repayers $(Y=1)$, and the remaining $n(1-\alpha_s)$ agents are drawn from the empirical score distribution of non-repayers $(Y=0)$.

Figures ~\ref{fig:type1_unfair_vis_all_fico_all} show that strategic firm incentive more agents to choose improvement over manipulation across different groups, which reduces (resp. increases) the impact of manipulation (resp. improvement) by the unqualified agents as seen in Figure~\ref{fig:type1_comp_Psi0_fico_all} (resp. Figure~\ref{fig:type1_comp_Phi0_fico_all}). This leads to a higher post-strategic qualification rate $\hat{\alpha_s}$ for all demographic groups under the strategic firm (Figure~\ref{fig:alpha_post_com_unfair_fico_all}). Figure~\ref{fig:type1_comp_theta_fico_all} and Figure~\ref{fig:type1_comp_utility_fico_all} show that the strategic optimal thresholds and the strategic firm’s utility are greater than those under non-strategic policy at all groups; this is consistent with
Proposition~\ref{prop:firm_impact_comp}. Simultaneously, this choice leads to both more improvement (Figure~\ref{fig:type1_comp_Phi1_fico_all}) and (weakly) more manipulation (Figure~\ref{fig:type1_comp_Psi1_fico_all}) by the qualified agents of all groups, both of which also benefit the firm. Only the AA group exhibits relatively little manipulation among qualified agents. This occurs because the strategic threshold faced by this group is substantially higher than its corresponding non-strategic threshold. As a result, a larger share of qualified agents opt for genuine improvement rather than manipulation—particularly those who are far from the threshold—since the potential benefit from improvement is greater than the marginal gain from strategic manipulation. 

\begin{figure}[htbp!]
    \centering
    \begin{subfigure}{0.27\textwidth}
        \centering
        \includegraphics[width=\textwidth,trim={0 0 0 0},clip]{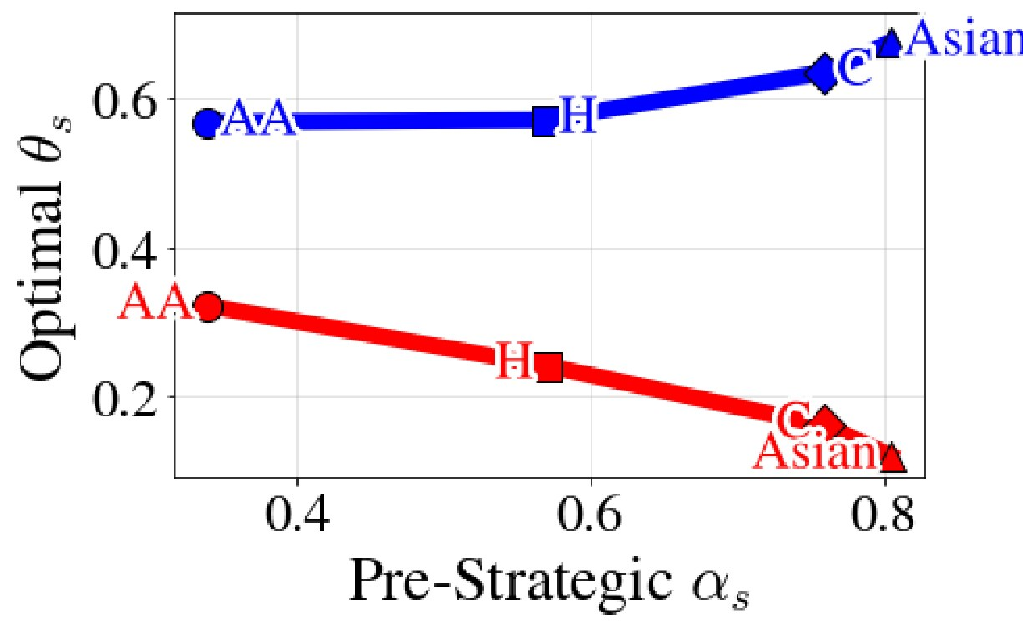}
        \caption{Optimal thresholds}
        \label{fig:type1_comp_theta_fico_all}
    \end{subfigure}
    \begin{subfigure}{0.27\textwidth}
        \centering
        \includegraphics[width=\textwidth,trim={0 0 0 0},clip]{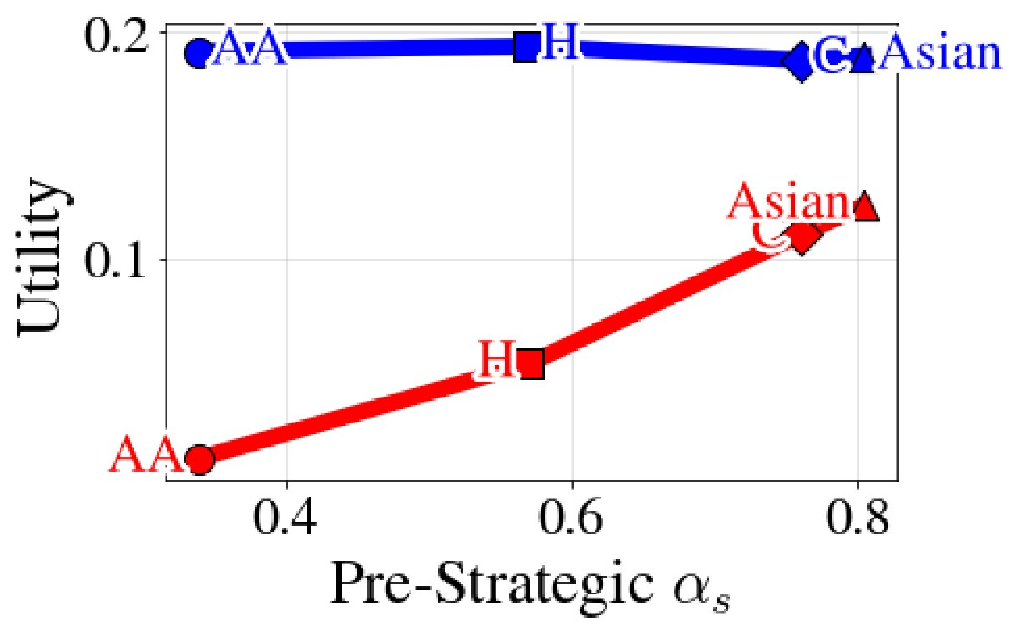}
        \caption{Optimal utility}
        \label{fig:type1_comp_utility_fico_all}
    \end{subfigure}
    \begin{subfigure}{0.27\textwidth}
        \centering
        \includegraphics[width=\textwidth,trim={0 0 0 0},clip]{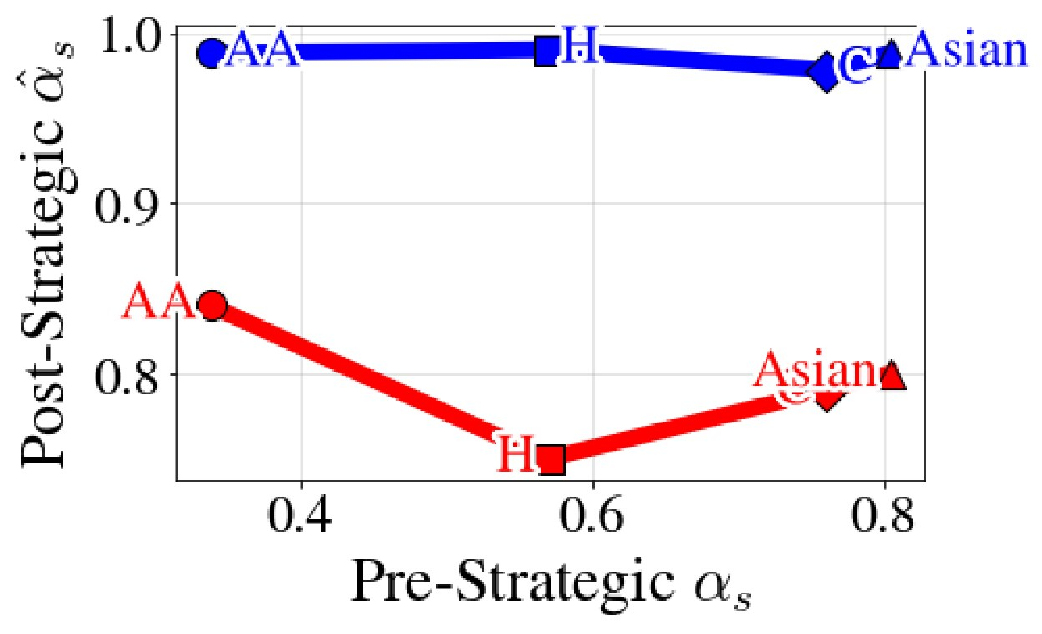}
    \caption{post-strategic $\hat{\alpha}_s$}
    \label{fig:alpha_post_com_unfair_fico_all}
    \end{subfigure}
    \begin{subfigure}{0.27\textwidth}
        \centering
        \includegraphics[width=\textwidth,trim={0 0 0 0},clip]{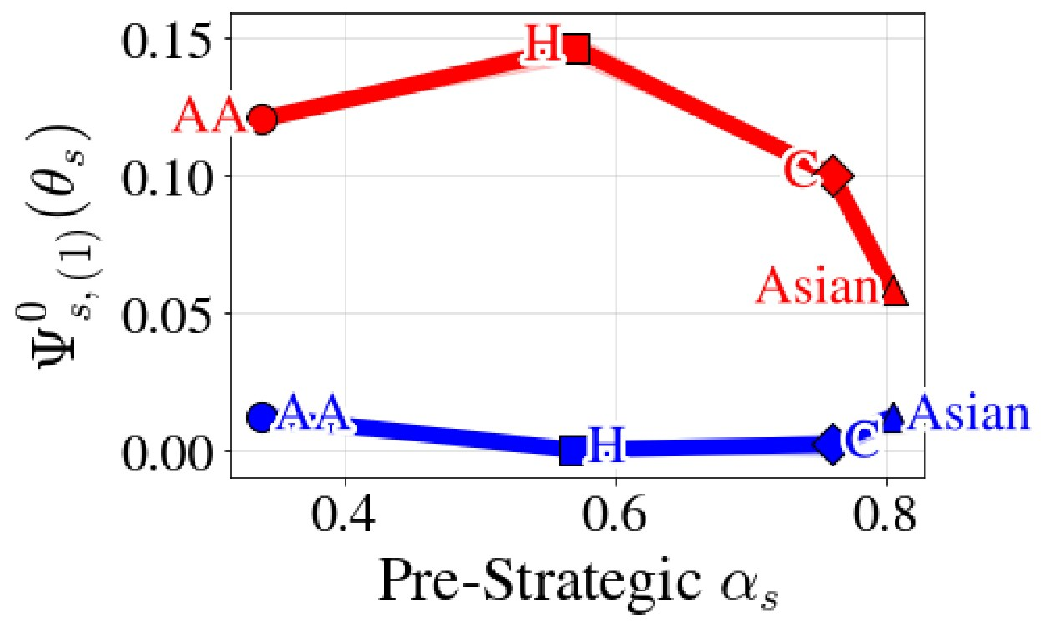}
        \caption{$\boldsymbol{\Psi}^0_{s,(1)}$ (Manipulation)}
        \label{fig:type1_comp_Psi0_fico_all}
    \end{subfigure}
    \begin{subfigure}{0.27\textwidth}
        \centering
        \includegraphics[width=\textwidth,trim={0 0 0 0},clip]{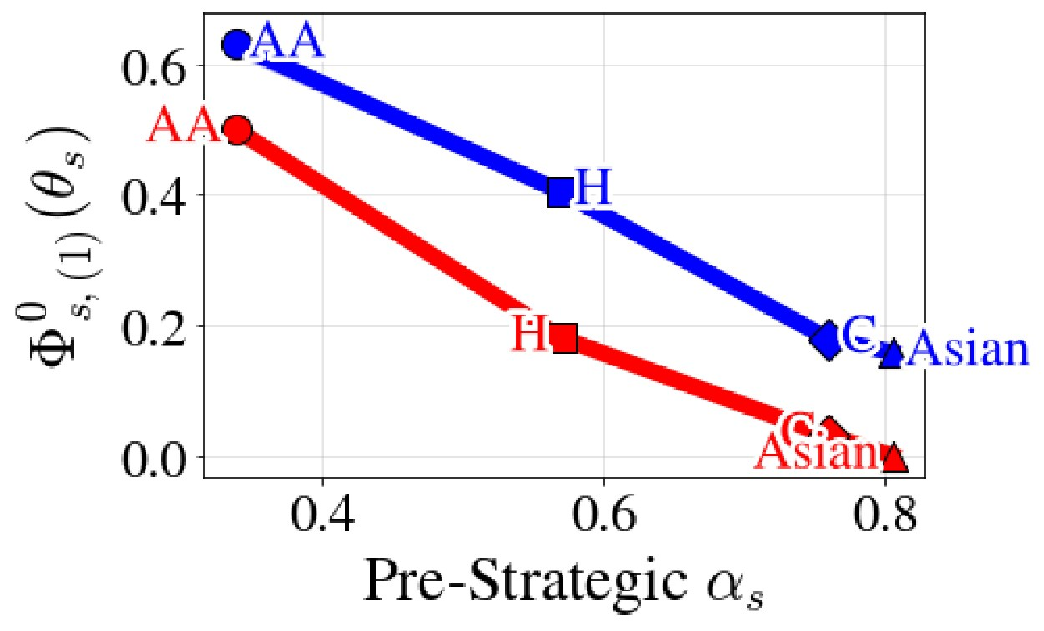}
        \caption{$\boldsymbol{\Phi}^0_{s,(1)}$ (Improvement)}
        \label{fig:type1_comp_Phi0_fico_all}
    \end{subfigure}
    \begin{subfigure}{0.27\textwidth}
        \centering
        \includegraphics[width=\textwidth,trim={0 0 0 0},clip]{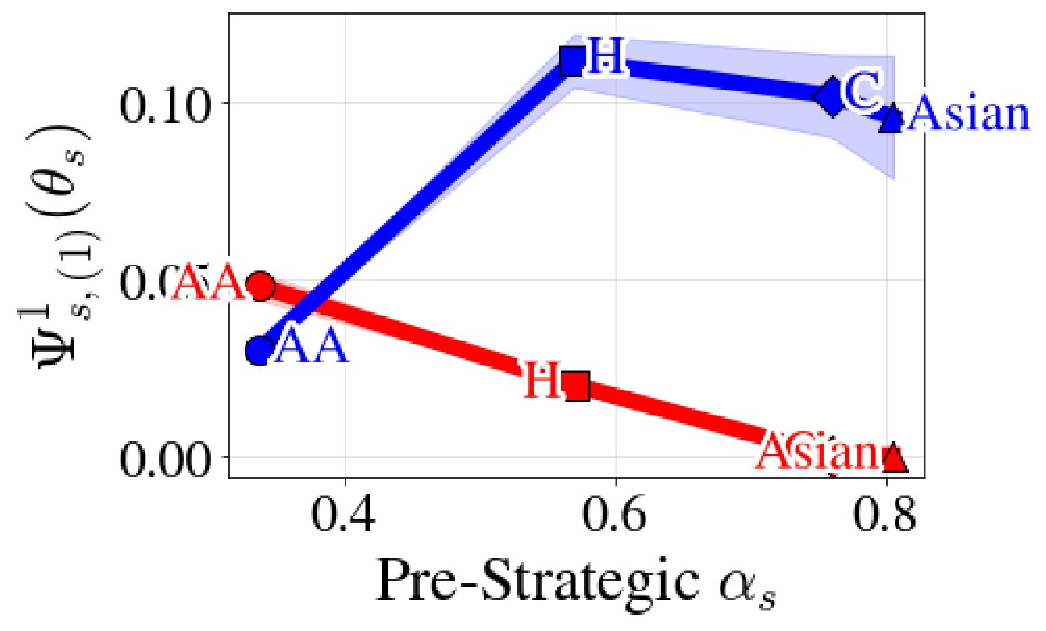}
        \caption{$\boldsymbol{\Psi}^1_{s,(1)}$ (Manipulation)}
        \label{fig:type1_comp_Psi1_fico_all}
    \end{subfigure}
    \begin{subfigure}{0.27\textwidth}
        \centering
        \includegraphics[width=\textwidth,trim={0 0 0 0},clip]{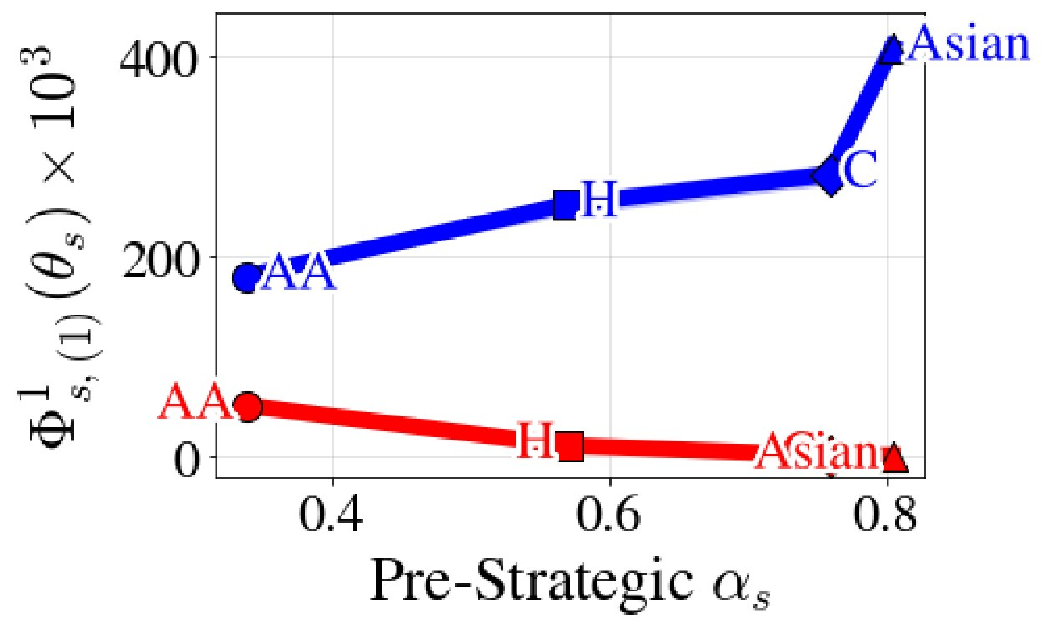}
        \caption{$\boldsymbol{\Phi}^1_{s,(1)}$ (Improvement)}
        \label{fig:type1_comp_Phi1_fico_all}
    \end{subfigure}
    \begin{subfigure}{0.2\textwidth}
        \centering
        \includegraphics[width=\textwidth,trim={0 0 0 0},clip]{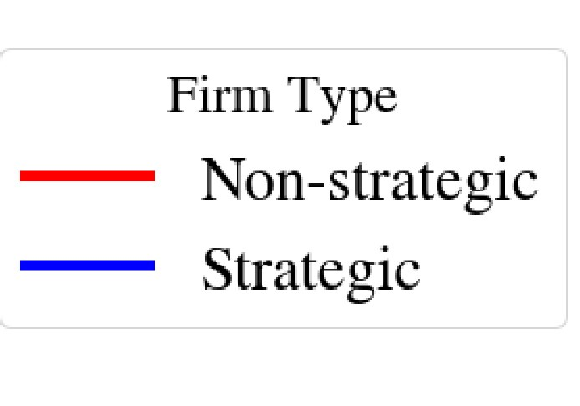}
        \caption{Legend}
        \label{fig:legend}
    \end{subfigure}
    \caption{Unfair policies comparison: non-strategic vs.\ strategic policies across demographic groups.}
    \label{fig:type1_unfair_vis_all_fico_all}
\end{figure}

\subsubsection{Impacts of fairness interventions}\label{app:fairness_all_groups}
\paragraph{Four-groups experiment setup.}
We next consider a population composed of four equal-sized demographic groups (C, AA, H, and A), where groups C, H, and A are majority-qualified, while group AA is majority-unqualified. As in the two-group fairness experiment setup, we use empirical histograms for the feature distributions and assume symmetric access to strategic actions with $C_{M,s}=0.2$ and $C_{I,s}=0.3$. We evaluate both unfair and fair policies (DP and EOP) under strategic and non-strategic firms using ROC curves. Depending on which group statistics are used to compute the true positive rates (TPRs) and false positive rates (FPRs) and the firm’s decision policy, three ROC curves may arise (see Figure~\ref{fig:ROC_cureves}). Figure~\ref{fig:roc_str} computes the TPR–FPR using post-strategic statistics $(\hat{x},\hat{y})$, and highlights the TPR–FPR of the strategic decisions, which is obtained using post-statistics. Figure~\ref{fig:roc_non} computes TPR–FPR using post-strategic statistics $(\hat{x},\hat{y})$, and highlights the TPR–FPR of the non-strategic decisions, which were still obtained using pre-statistics $(x,y)$. Finally, Figure~\ref{fig:roc_fixed_x_y} computes TPR–FPR using pre-strategic statistics $(x,y)$, and highlights the TPR–FPR of non-strategic decisions (obtained using $(x,y)$) and strategic decisions (obtained using post-strategic statistics $(\hat{x},\hat{y})$). 

\begin{figure}[htbp!]
    \centering
    \begin{subfigure}[t]{0.27\textwidth}
        \centering
        \includegraphics[width=\linewidth]{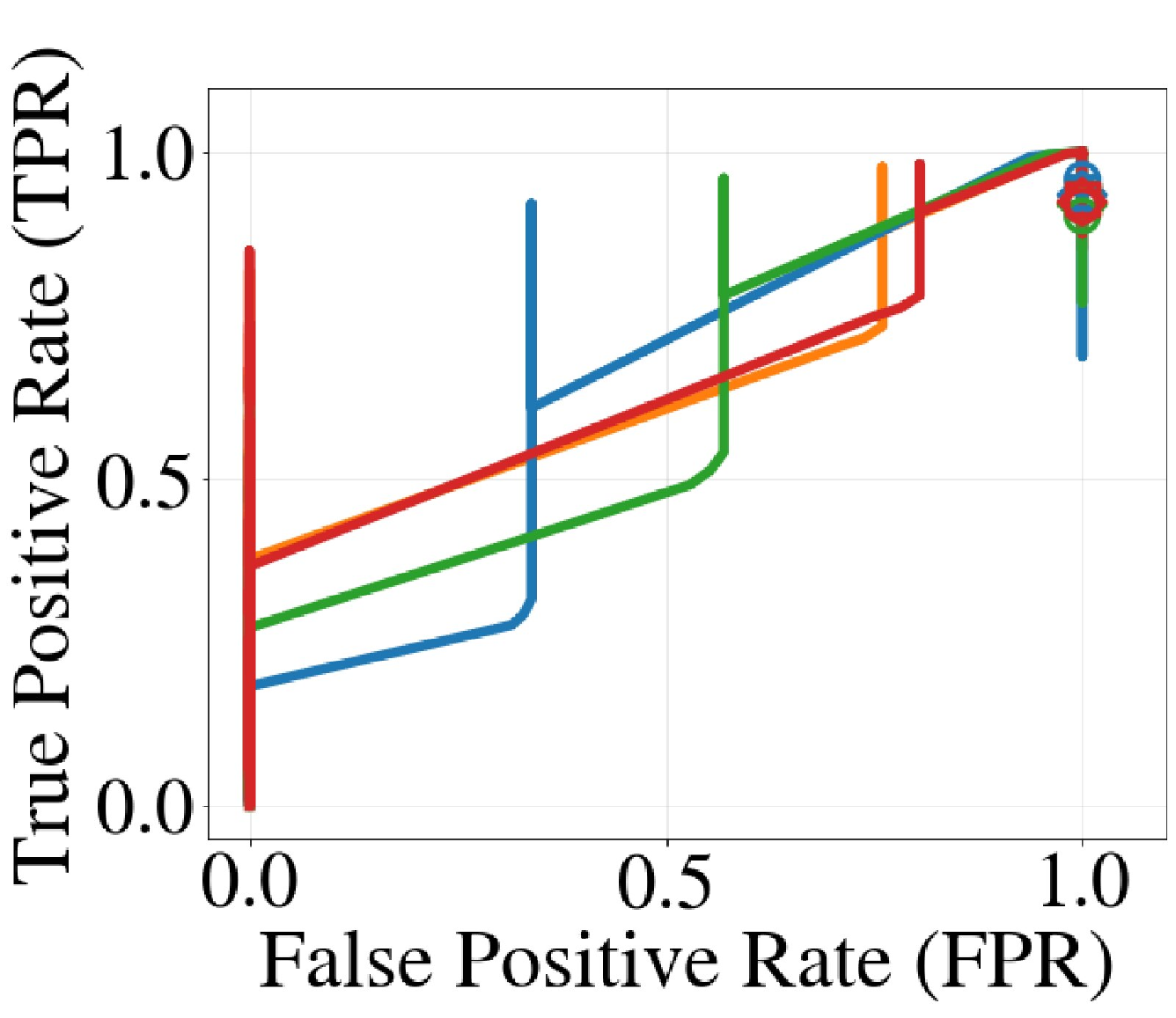}
        \subcaption{Post-strategic ROC (strategic firms)}
        \label{fig:roc_str}
    \end{subfigure}
    ~~~
        \begin{subfigure}[t]{0.27\textwidth}
        \centering
        \includegraphics[width=\linewidth]{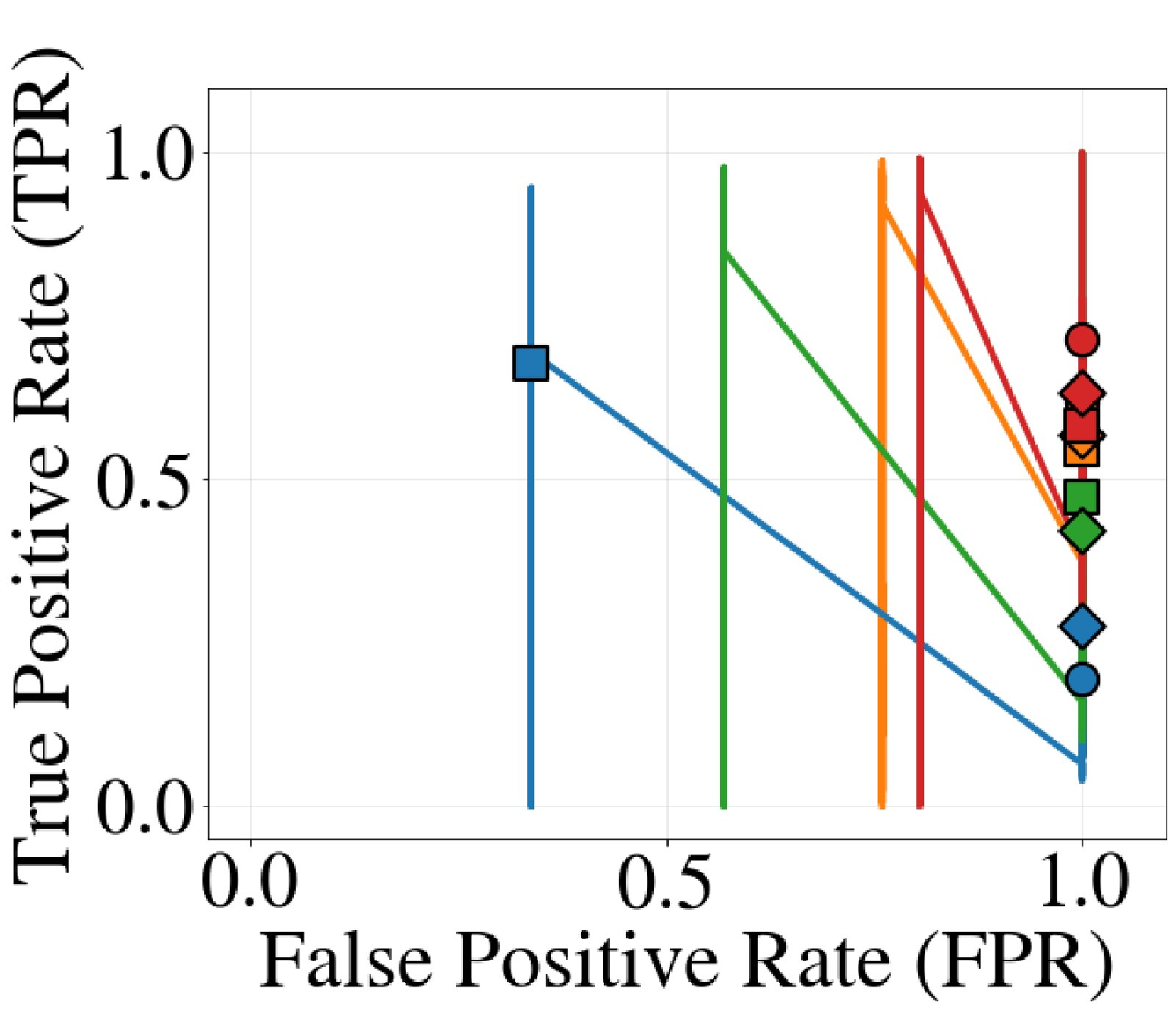}
        \subcaption{Post-strategic ROC (non-strategic firms)}
        \label{fig:roc_non}
    \end{subfigure}
    ~~~
    \begin{subfigure}[t]{0.39\textwidth}
        \centering
        \includegraphics[width=\linewidth]{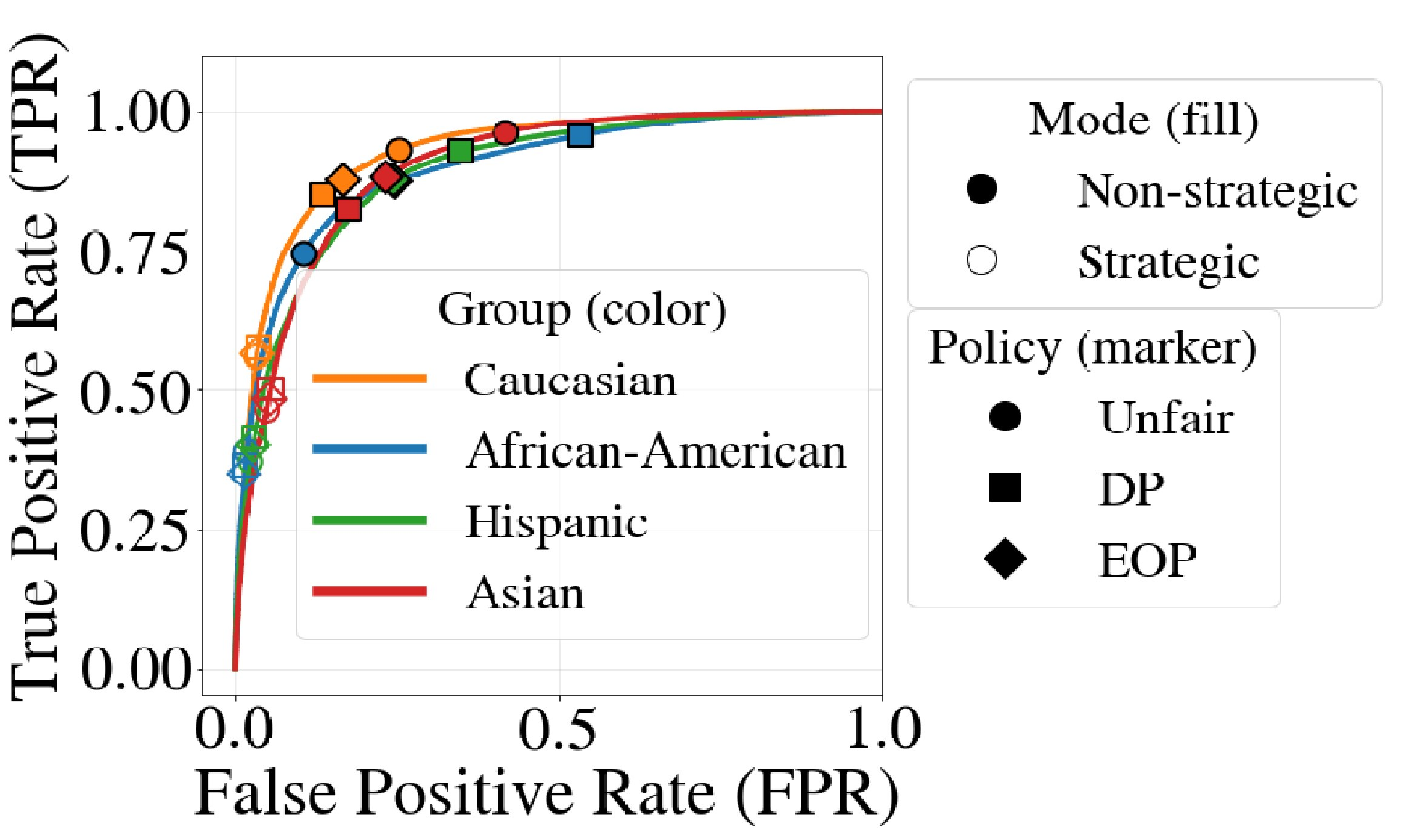}
        \subcaption{Pre-strategic rates (both firms)}
        \label{fig:roc_fixed_x_y}
    \end{subfigure}
    \caption{ROC curves under pre- and post-strategic statistics. Figures~(a)--(b) report ROC based on post-strategic statistics, reflecting the actual fairness performance of strategic and non-strategic decision policies, respectively. Figure~(c) reports the ROC based on pre-strategic statistics, reflecting the intended fairness performance under the non-strategic decision policy.}
    \label{fig:ROC_cureves}
\end{figure}

\paragraph{Strategic vs. non-strategic fair firms.}
First, the ROC curves computed using post-strategic statistics differ substantially from those based on pre-strategic statistics, reflecting the fact that both features and labels are endogenous to the firm’s threshold policy once strategic behavior is taken into account.

Figure~\ref{fig:roc_str} shows that a strategic firm can simultaneously satisfy both DP and EOP fairness constraints without requiring substantial deviations from the threshold used under the unfair policy, consistent with the trends observed in the previous experiment (See Figure~\ref{fig:fair_vis_Type1_c_aa_2}). This is observed by the nearly identical rates indicated by the square, circle, and diamond markers in Figure~\ref{fig:roc_str}. Notably, the TPRs and FPRs for all groups approach 1. This occurs because the strategic firm accepts all previously qualified agents and induces a large proportion of initially unqualified agents across all groups to invest in improvement, after which they are accepted. Consequently, all groups become overwhelmingly qualified, with qualification rates approaching 1, consistent with the trends observed in the previous experiment (See Figure~\ref{fig:c_aa_1}). As a result, the TPRs converge to 1. Moreover, only a negligible number of unqualified agents remain, so although the FPRs approach 1, indicating acceptance of nearly all remaining unqualified agents, this corresponds to a very small proportion. Together, these dynamics explain how both DP and EOP fairness constraints are satisfied in this setting.

Figure~\ref{fig:roc_non} demonstrates that a non-strategic firm fails to satisfy the fairness constraint when evaluated on true (post-strategic) statistics; in particular, EOP is violated, as evidenced by the differing TPRs indicated by the diamond markers. In contrast, Figure~\ref{fig:roc_fixed_x_y} shows that the same non-strategic firm \emph{appears} to satisfy EOP when evaluated using pre-strategic statistics (see the filled diamond markers). This discrepancy highlights a key distinction between the \emph{intended} fairness guarantees under the firm’s assumed behavioral model and the \emph{realized} fairness outcomes once agents respond strategically.

Moreover, we observe FPRs of approximately 1 across all groups under the non-strategic firm’s policies, while TPRs vary within and across groups, with the exception of the AA group under the DP fair non-strategic policy, where TPR $\approx 0.6$ and FPR $\approx 0.3$. 

This pattern arises because the non-strategic firm, in an attempt to satisfy the fairness constraint, raises the decision thresholds for the advantaged groups C, A, and H. This incentivizes agents in these groups to behave strategically. However, due to the higher thresholds, not all initially qualified agents and unqualified agents who choose to improve are ultimately accepted. As a result, TPRs under fair policies are lower than those under the unfair policy for these groups. At the same time, among unqualified agents who opt to manipulate, together with those who are accepted by default, are ultimately accepted, yielding an FPR close to 1. Unlike the strategic firm case, the proportion of unqualified agents in this accepted set is higher (see Figure~\ref{fig:c_aa_1}).

This results in greater acceptance of qualified agents (with or without effort), as well as acceptance of some initially unqualified agents who choose to improve and of remaining unqualified agents who opt to manipulate. Consequently, TPRs increase under the fair policy, with a more pronounced increase under DP. The FPR is approximately 1 under most policies, except under DP, where it decreases to about 0.3.

This is because satisfying DP requires a larger reduction in the threshold, which allows even more qualified agents to be accepted. At the same time, although the lower threshold also encourages some unqualified agents to attempt manipulation (cheaper actions at a risk of rejection), not all of these agents are ultimately accepted. Nevertheless, the AA group contains a larger proportion of post-strategic unqualified agents overall (see Figure~\ref{fig:c_aa_1}), which shapes the observed outcomes.

\subsection{FICO dataset-Based Type 3 Experitment}\label{app:type3_fico_experiments}
In this appendix, we present numerical experiments to illustrate the Type~3 equilibrium analysis discussed in Section~\ref{sec:effect-strateg-prediction}, particularly Proposition~\ref{prop:firm_impact_comp} and Corollary~\ref{cor:fair-policies}. We examine the impact of accounting for agents’ strategic behavior—limited to manipulation under this equilibrium—by comparing the optimal policies and responses of strategic versus non-strategic firms. We further evaluate the effects of fairness interventions under both Equality of Opportunity (EOP) and Demographic Parity (DP) constraints. In this appendix, unless otherwise noted, we consider settings that ensure a Type~3 equilibrium (manipulation-only) strategic response. Action costs are set to $C_{M,s}=0.1$ and $C_{I,s}=0.9$, and boost distributions follow truncated Gaussian forms: $\tau^y_{M,s}(b) \sim \text{TruncNorm}([0.35, 0.70], 0.525, 0.22)$, $\tau^0_{I,s}(b) \sim \text{TruncNorm}([0.45, 0.85], 0.65, 0.15)$, and $\tau^1_{I,s}(b) \sim \text{TruncNorm}([0.57, 0.97], 0.77, 0.15)$.

\paragraph{Experimental setup for two demographic groups.}
We begin by comparing the behavior of policies without fairness constraints across two demographic groups—African American (AA) and Caucasian (C). For each group and for each target qualification level $\alpha_s$, we simulate outcomes using the same two-group protocol as in Section~\ref{sec:numerical-exp}. The target rate $\alpha_s$ is swept over $\{0.1, 0.2, \ldots, 0.9\}$, and for every setting we perform 50 independent trials, reporting the empirical mean across runs. 

\paragraph{Strategic vs.\ non-strategic (unfair) firms.}
The trends in Figure~\ref{fig:fair_vis_Type3_fico} follow the theoretical predictions. Figure~\ref{fig:type3_c,aa_theta} shows that strategic firms systematically adopt higher thresholds than non-strategic firms for all values of $\alpha_s$, in agreement with Proposition~\ref{prop:firm_impact_comp}. Correspondingly, Figure~\ref{fig:type3_c,aa_utility} demonstrates larger equilibrium utility under strategic decision-making. Figure~\ref{fig:type3_c,aa_psi1} and Figure~\ref{fig:type3_c,aa_psi0} reveal how this threshold gap shapes agent responses: qualified agents engage more in beneficial manipulation, while unqualified agents reduce costly manipulation, both of which increase firm utility—again matching the comparative statics established in Proposition~\ref{prop:firm_impact_comp} (parts (ii) and (iii)). It is worth noting that for $\alpha_s \leq 0.25$, the strategic firm induces slightly lower manipulation among qualified agents compared to the non-strategic firm. This aligns with its higher threshold policy, which reduces the incentive for qualified manipulation (see Figure~\ref{fig:type3_c,aa_psi0}). Conversely, for high $\alpha_s \geq 0.75$, the strategic firm permits slightly more manipulation by unqualified agents, as it lowers the threshold to avoid excessive rejection of qualified applicants. In doing so, the firm tolerates limited manipulation by a small fraction of unqualified agents to preserve acceptance for the much larger pool of qualified ones. 

\begin{figure}[hbt]
    \centering
    \begin{subfigure}{0.2\textwidth}
        \centering
        \includegraphics[width=\textwidth,trim={0 0 0 0},clip]{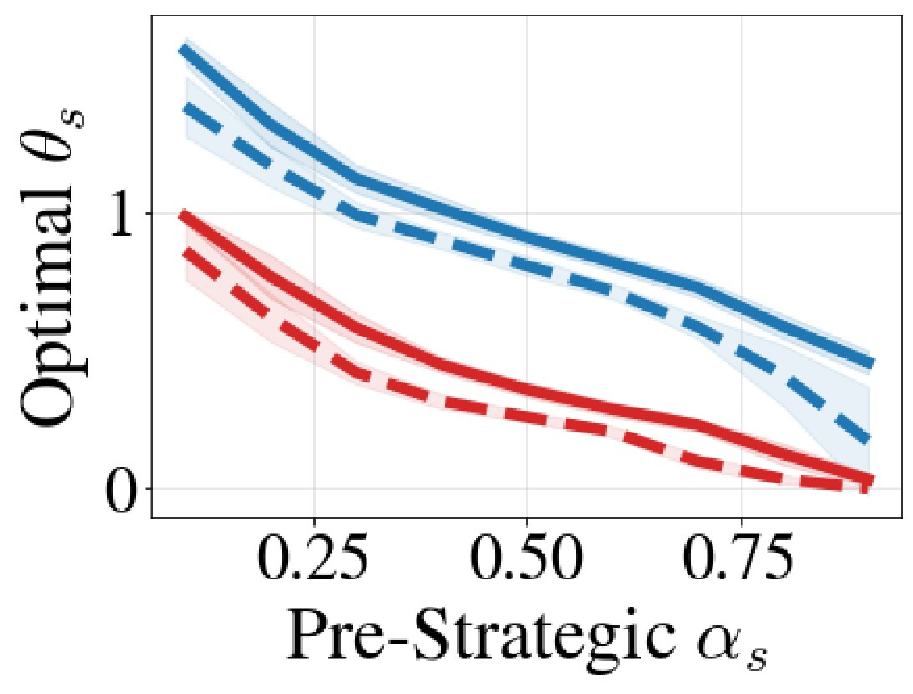}
        \caption{Optimal thresholds}
        \label{fig:type3_c,aa_theta}
    \end{subfigure}
    \begin{subfigure}{0.215\textwidth}
        \centering \includegraphics[width=\textwidth,trim={0 0 0 0},clip]{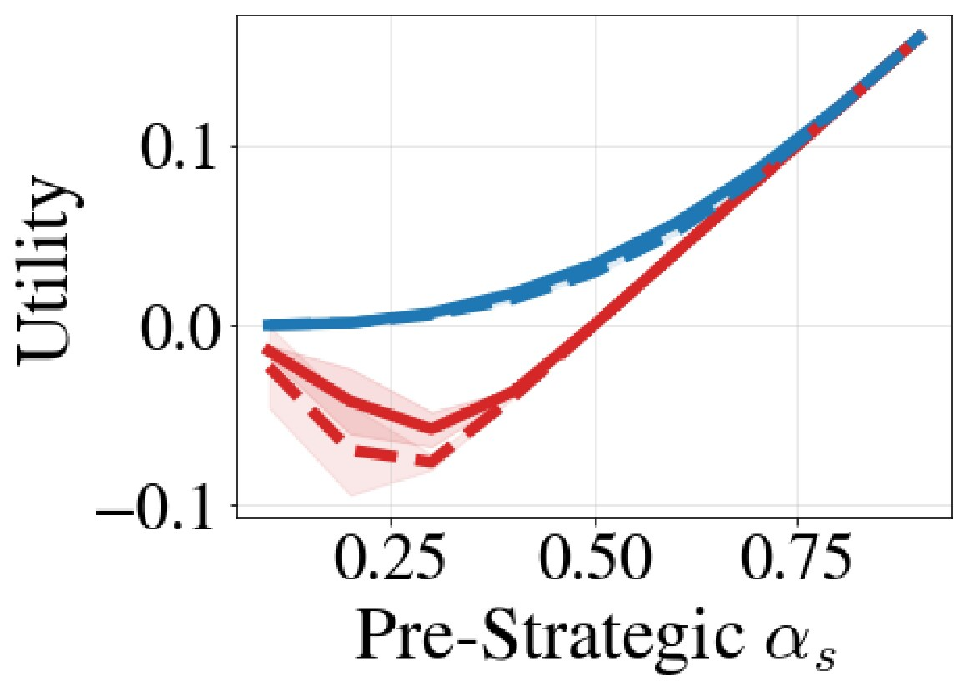}
        \caption{Optimal utility}
        \label{fig:type3_c,aa_utility}
    \end{subfigure}
    \begin{subfigure}{0.215\textwidth}
        \centering \includegraphics[width=\textwidth,trim={0 0 0 0},clip]{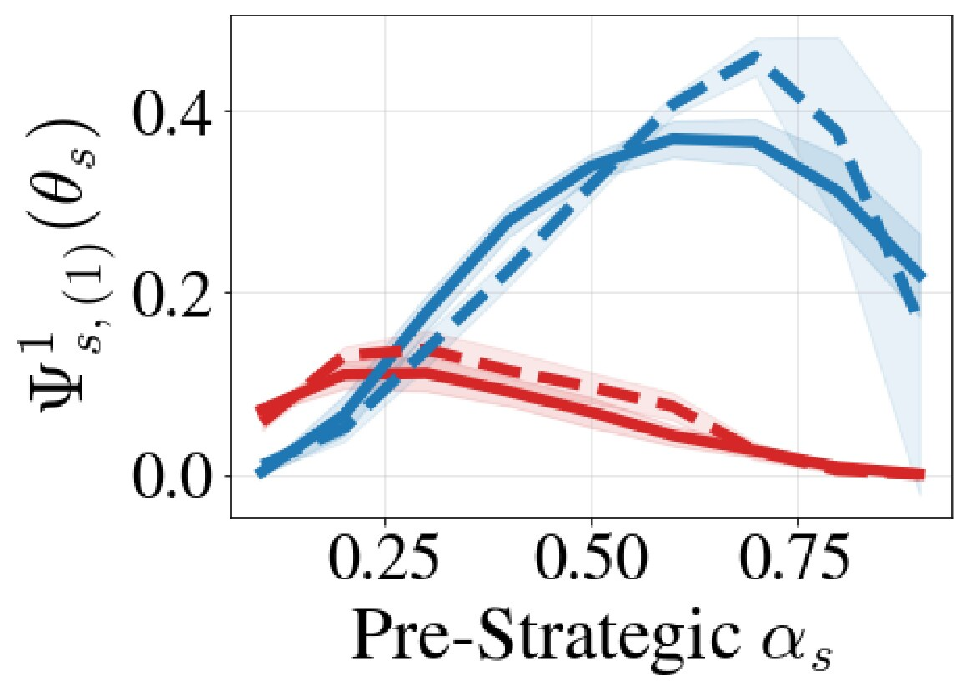}
        \caption{$\Psi^1_{a,(1)}$ (Manipulation)}
        \label{fig:type3_c,aa_psi1}
    \end{subfigure}
    \begin{subfigure}{0.35\textwidth}
        \centering \includegraphics[width=\textwidth,trim={0 0 0 0},clip]{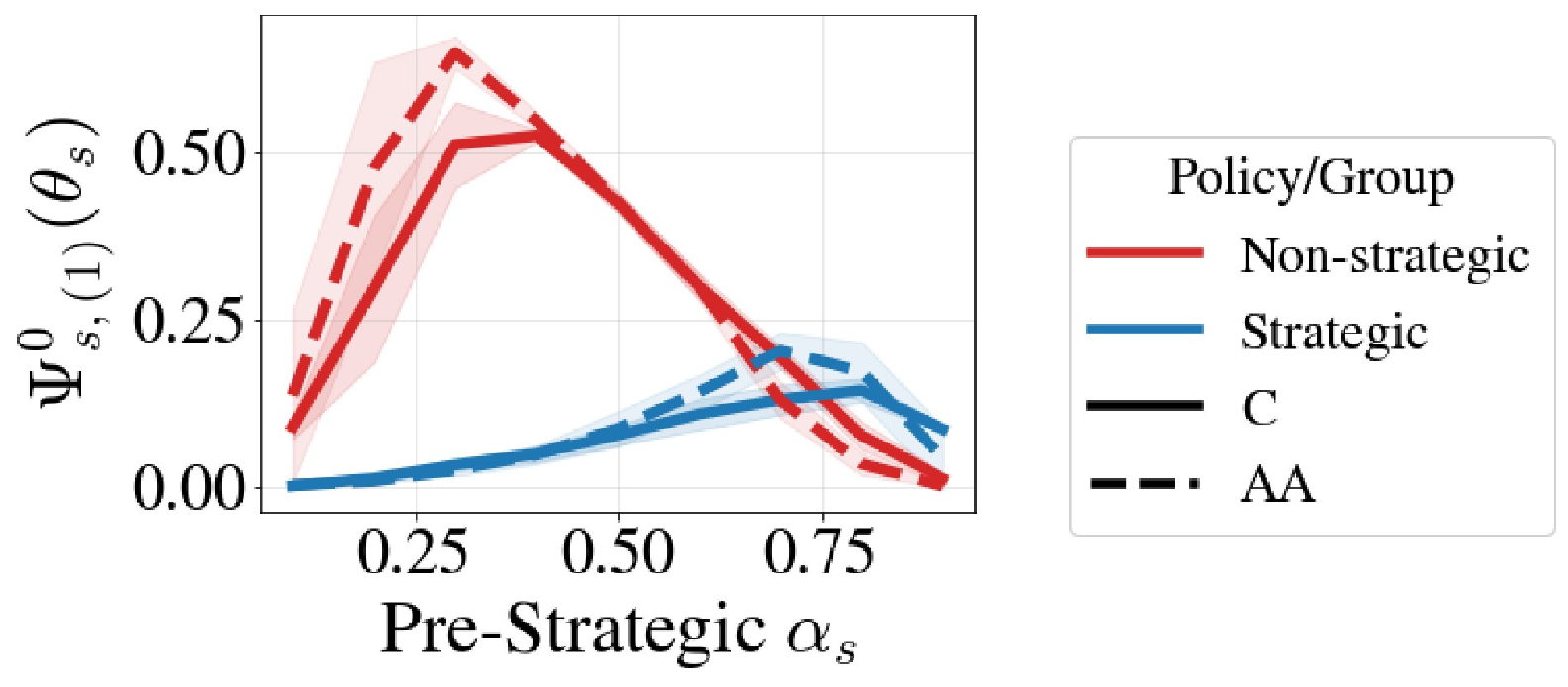}
        \caption{$\Psi^0_{a,(1)}$ (Manipulation)}
        \label{fig:type3_c,aa_psi0}
    \end{subfigure}
    \caption{Unfair policies: non-strategic vs. strategic policies' comparison when varying pre-strategic $\alpha_s$ under Type 3 equilibrium.} 
    \label{fig:fair_vis_Type3_fico}
\end{figure}

\paragraph{Two-groups experiment setup.}
For comparison, we also consider a two-group setting consisting of a majority-qualified (Caucasian) group as group A and a majority-unqualified (African-American) group as group B. As in the four-group case, we use empirical histograms for the feature distributions and assume symmetric access to strategic actions.

\paragraph{Strategic vs. non-strategic fair firms.}
Figure~\ref{fig:fair_vis_Type3_c_aa} provides contour visualizations of the firm’s utility surface, together with the corresponding utilities under DP and EOP parity constraints, as the group-specific thresholds $\theta_C$ and $\theta_{AA}$ vary. The optimal thresholds for the unfair, DP-fair, and EOP-fair cases are also indicated in the figure. These same thresholds, along with the resulting manipulation effects on firm utility, are reported in Figure~\ref{fig:c_aa_type3_impact}. A key distinction is that the strategic firm evaluates fairness over post-strategic statistics, whereas the non-strategic firm evaluates fairness over \emph{pre}-strategic statistics. Consequently, the fairness desiderata of a non-strategic firm are not satisfied ex-post, even though they appear satisfied under its own (incorrect) evaluation.

As seen in Type~1 equilibrium, Figure~\ref{fig:fair_vis_Type3_c_aa} shows that both firms raise (resp. lower) thresholds for the majority-qualified (resp. majority-unqualified) group to meet fairness constraints, per Corollary\ref{cor:fair-policies} and prior literature, with the strategic firm setting higher thresholds for both groups, consistent with Proposition~\ref{prop:firm_impact_comp}.

\begin{figure}[ht]
    \centering
        \begin{subfigure}{0.45\textwidth}
        \centering
        \includegraphics[width=\textwidth,trim={0 0 0 0},clip]{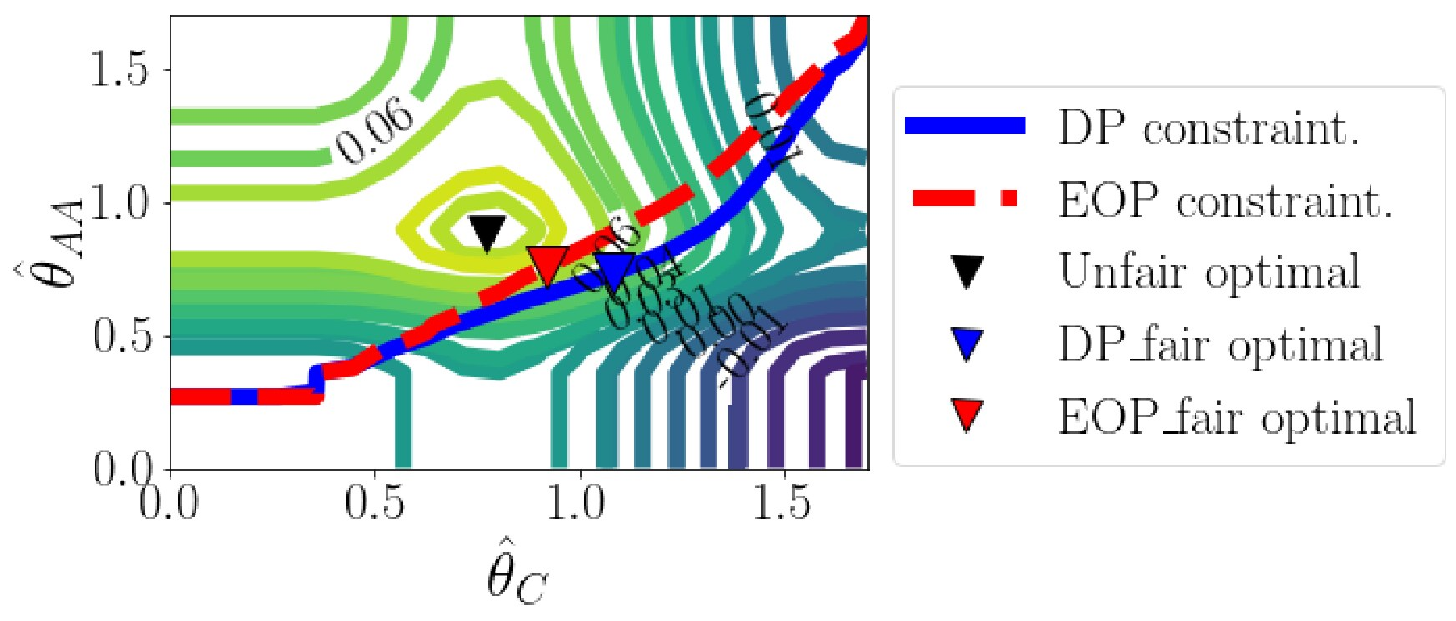}
        \caption{Strategic firm's utility}
        \label{fig:c_aa_str_type3}
    \end{subfigure}
    \begin{subfigure}{0.45\textwidth}
        \centering
        \includegraphics[width=\textwidth,trim={0 0 0 0},clip]{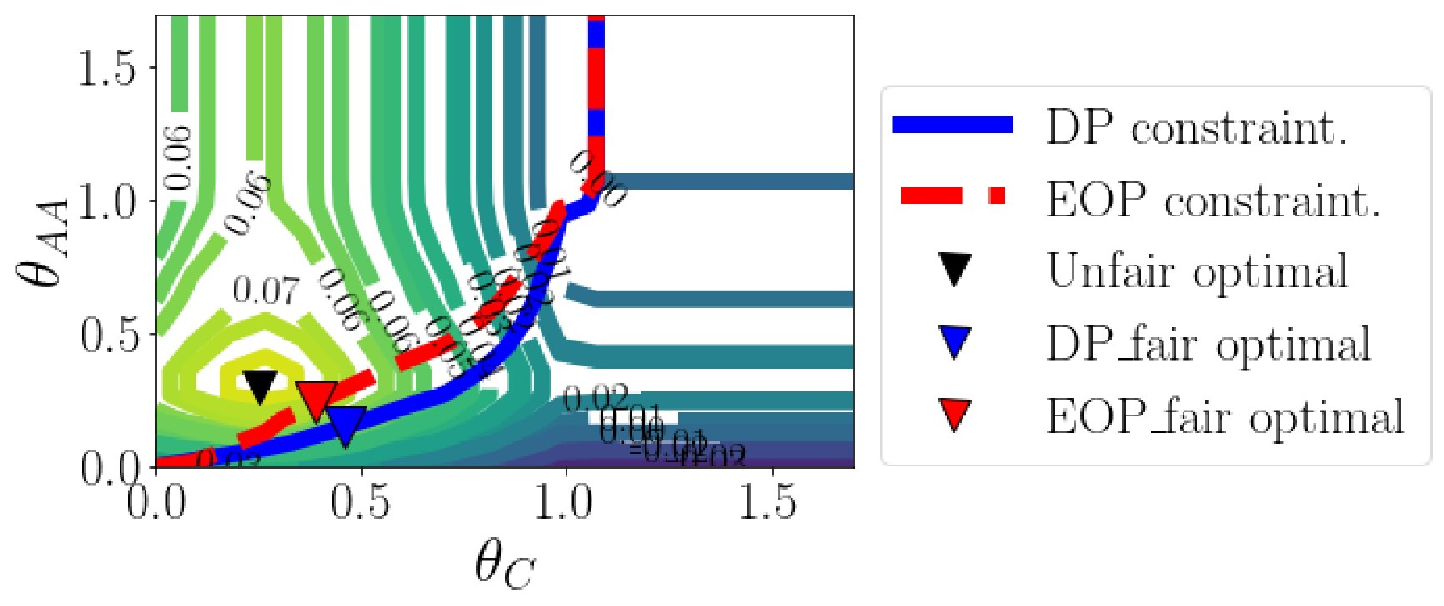}
        \caption{Non-strategic firm's utility}
        \label{fig:c_aa_nonstr_type3}
    \end{subfigure}
    \caption{ Caucasian vs. African American: Firm's utility and optimal thresholds when imposing fairness constraints under Type 3 equilibruim.}
    \label{fig:fair_vis_Type3_c_aa}
\end{figure}

Unlike the Type 1 setting, the comparison of Demographic Parity (DP) and Equality of Opportunity (EOP) constraints in the Type 3 equilibrium reveals pronounced shifts in decision thresholds for strategic firms relative to their non-strategic counterparts (see Figure~\ref{fig:fair_vis_Type3_c_aa}). Because no \emph{improvement} actions are available, the overall qualification rate $\alpha_s$ remains fixed, compelling strategic firms to substantially lower $\theta_{AA}$ and raise $\theta_C$ to satisfy fairness constraints. In this case, strategic behavior alone cannot achieve selection parity. 

For group AA, both the DP- and EOP-fair strategic firms produce higher levels of unqualified manipulation, increasing $\boldsymbol{\Psi}^0_{AA,(1)}(\hat{\theta}^{\mathcal{C}}{AA})$ relative to the unfair policy’s higher threshold that had previously suppressed manipulation (see the solid red lines in the first Figure of Figure~\ref{fig:c_aa_type3_impact}). The non-strategic firm also relaxes $\theta{AA}$ but exhibits a reduction in manipulation (Figure~\ref{fig:c_aa_type3_impact}, second Figure) because unqualified agents already engage heavily in such behavior, and many are now accepted by default. Overall, the absence of improvement actions forces fairness to be achieved through larger threshold adjustments under strategic awareness, resulting in lower firm utility.

\begin{figure}[htb]
    \centering
    \includegraphics[width=\textwidth,trim={0 0 0 0},clip]{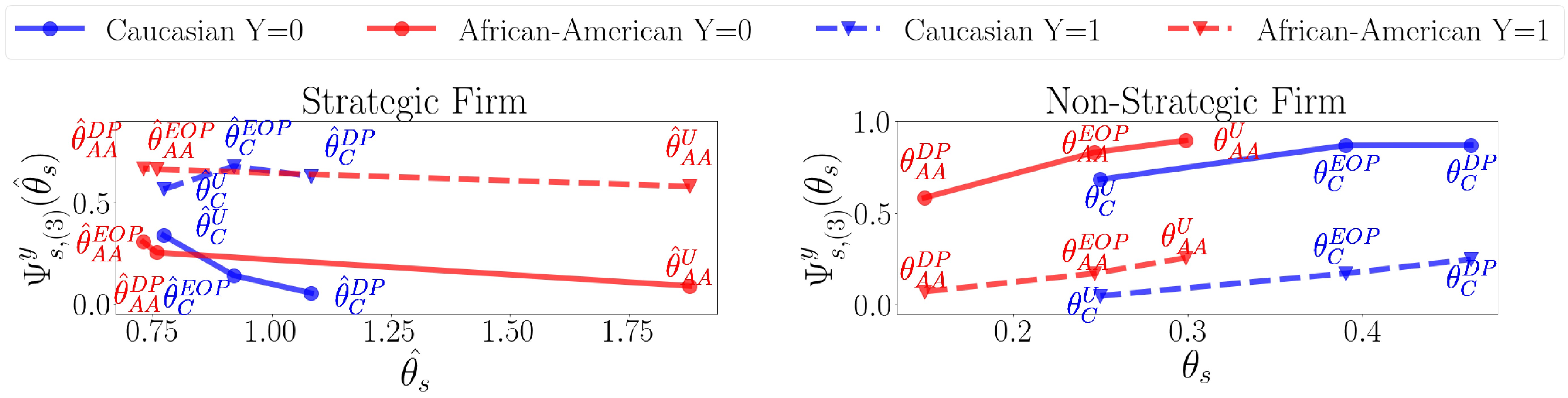}
    \caption{Caucasian vs. African American: Strategic behavior's impact comparison: fair vs unfair policies}
    \label{fig:c_aa_type3_impact}
\end{figure}

\paragraph{Four-groups experiment setup}
Building on the four-group setup in Section~\ref{sec:numerical-exp}, we follow the same experimental design and evaluation protocol, and highlight only the differences here. Here, we adopt different cost and boost parameters corresponding to a Type~3 equilibrium. As outlined at the beginning of this section, this setting permits only manipulation actions, which affect observable features but not the true qualification label~$y$. We again evaluate unfair and fair policies (DP and EOP) under both strategic and non-strategic firms using ROC curves, following the same methodology as in Section~\ref{sec:numerical-exp}. However, note that labels are now fixed and remain unchanged under strategic manipulation.

\paragraph{Strategic vs. non-strategic fair firms.}
First, as in Section~\ref{sec:numerical-exp}, although agents now manipulate only their features, ROC curves computed using post-strategic statistics continue to differ from those based on pre-strategic statistics. However, in contrast to the Type~1 equilibrium studied in Section~\ref{sec:numerical-exp}, Figure~\ref{fig:roc_str_type3} shows that under the Type~3 (manipulation-only) setting, a strategic firm attains fairness by adjusting its decision thresholds relative to those used under the unfair policy, as indicated by the differing rates among the square, circle, and diamond markers.

The observed shifts in TPR–FPR values from circles to diamonds and squares across all groups can be explained as follows. To achieve fairness, the strategic firm increases the decision thresholds for the C, A, and H groups, which have higher qualification rates ($>0.5$). This encourages agents in these groups to manipulate their features in order to meet the higher thresholds and gain acceptance. In contrast, the AA group—being the most disadvantaged group in our setting ($\alpha_s < 0.4$) faces a reduced threshold under the fair policy. Previously, this threshold was set high to curb manipulation by unqualified agents. Lowering it incentivizes manipulation, particularly among the large fraction of unqualified AA agents who were previously rejected, some of whom may now obtain acceptance. Finally, we note that the strategic firm only weakly satisfies the EOP constraint, due to the uncertainty in manipulation outcomes for the AA group. 

By comparison, Figure~\ref{fig:roc_non_type3} shows that a non-strategic firm \emph{fails} to satisfy the EOP constraint, even though Figure~\ref{fig:roc_fixed_x_y_type3} suggests apparent fairness when performance is evaluated using pre-strategic data (see the filled diamond markers). This divergence, as noted previously, highlights the gap between the firm’s \emph{intended} fairness guarantees and the \emph{realized} outcomes once agents respond strategically.

Moreover, the group-wise non-strategic policies exhibit consistent FPRs across groups (markers of the same color in Figure~\ref{fig:roc_non_type3}). Although the non-strategic firm attempts to achieve fairness by adjusting its thresholds, doing so primarily incentivizes qualified agents to manipulate, with higher thresholds assigned for C, A, and H groups; more among those may not be accepted, leading to lower TPRs relative to the unfair policy. However, for the AA group, reducing the already high threshold primarily accepts more qualified agents who may or may not choose to manipulate, leading to higher TPRs relative to the unfair policy. Fair non-strategic firm, however, leaving incentives for unqualified agents largely unchanged. This illustrates a negative consequence of enforcing fairness without accounting for agents’ strategic responses. 

\begin{figure}[htbp!]
    \centering
    \begin{subfigure}[t]{0.28\textwidth}
        \centering
        \includegraphics[width=\linewidth]{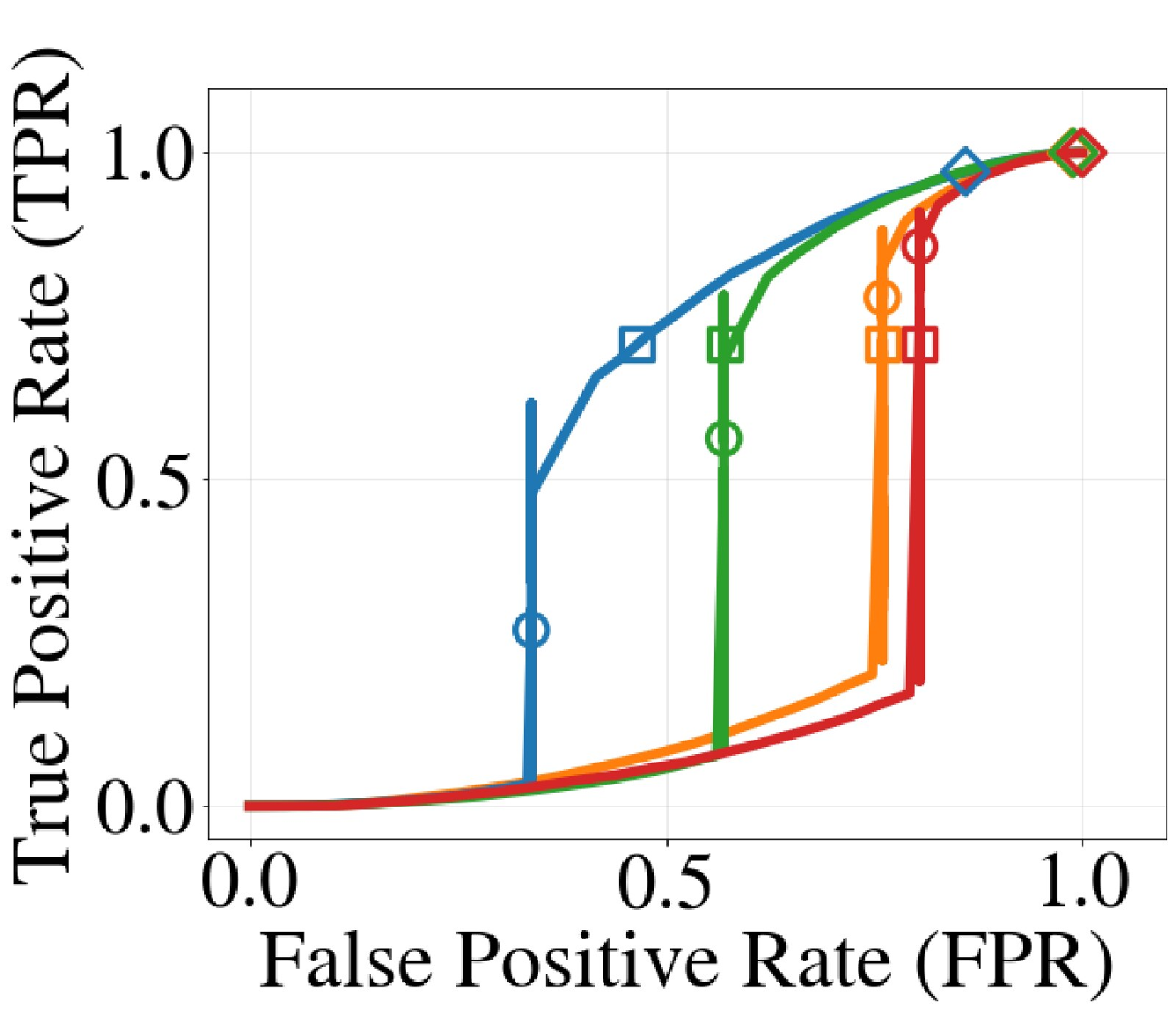}
        \subcaption{Post-strategic ROC (strategic firm)}
        \label{fig:roc_str_type3}
    \end{subfigure}
        \begin{subfigure}[t]{0.28\textwidth}
        \centering
        \includegraphics[width=\linewidth]{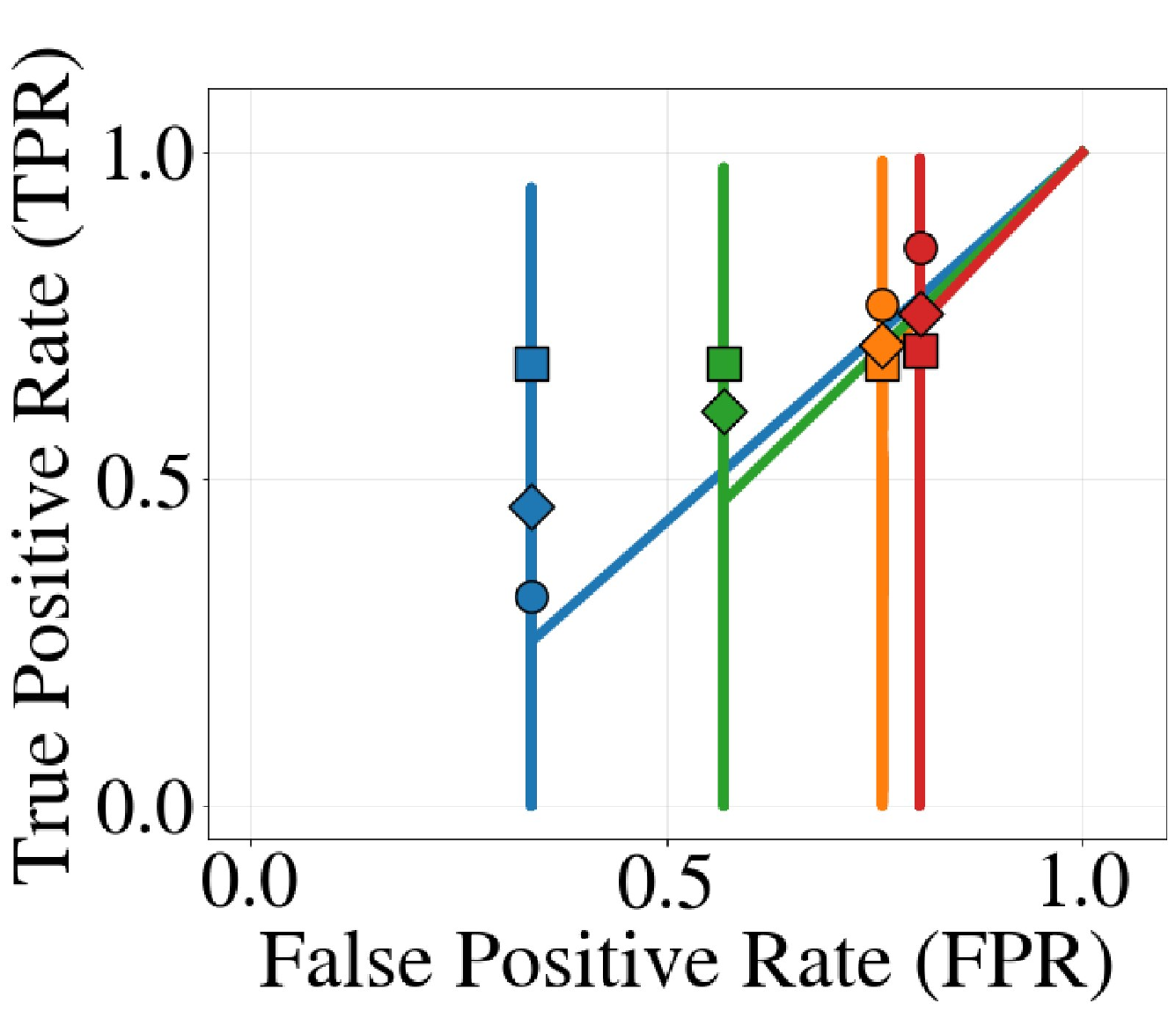}
        \subcaption{Post-strategic ROC (non-strategic firm)}
        \label{fig:roc_non_type3}
    \end{subfigure}
    \begin{subfigure}[t]{0.4\textwidth}
        \centering
        \includegraphics[width=\linewidth]{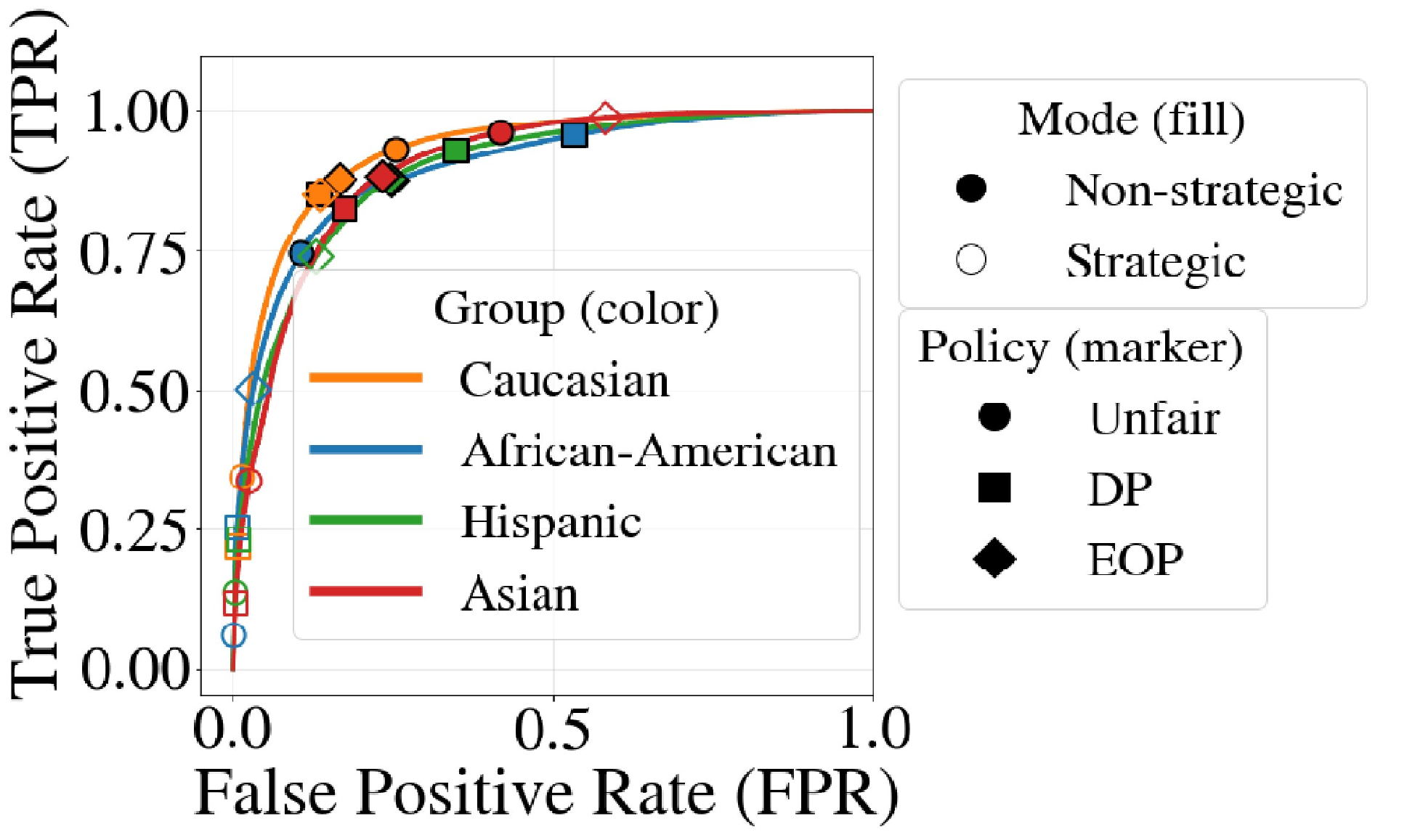}
        \subcaption{Pre-strategic ROC (both firms}
        \label{fig:roc_fixed_x_y_type3}
    \end{subfigure}
    \caption{ROC curves under pre- versus post-strategic statistics. Figures~(a)--(b) report ROC based on post-strategic statistics, reflecting the actual fairness performance of strategic and non-strategic decision policies, respectively.  Figure~(c) reports ROC based on pre-strategic statistics, reflecting the intended fairness performance under the non-strategic decision policy.}
    \label{fig:ROC_cureves_type3}
\end{figure}

\section{Synthetic Data–Based Experiments}\label{app:add_NE}
We present additional numerical experiments on synthetic datasets. We will start with a population consisting of a single demographic group to highlight the main observations from Proposition~\ref{prop:firm_impact_comp} (including under some relaxations of its assumptions). 
We will then present experiments on a population with two distinct demographic groups, placing emphasis on the impacts of fairness interventions, contrasting the induced strategic behavior in each demographic group, and analyzing the induced post-strategic qualification rates across groups at equilibrium. We first conduct these experiments under a Type~1 equilibrium (Appendix~\ref{app:type1_experiments}) and then extend the analysis to a Type~3 equilibrium (Appendix~\ref{app:type3_experiments}).

\subsection{Synthetic dataset Type 1 Experiments}\label{app:type1_experiments}
\paragraph{Single-group experiment setup.} 
We start with a population of agents from a single group $s$, with truncated Gaussian feature distributions. In particular, we let $G^0_s \sim \text{TruncNorm}([20, 60], 40, 15^2) \linebreak[4] \text{ and } G^1_s = \text{TruncNorm}([53, 113], 83, 15^2)$. We vary the (pre-strategic) qualification rate $\alpha_s$ in $(0,1)$. For the characteristics of the strategic actions, we set $C_{M,s}=0.1$, and $C_{I,s}=0.6$. The boost distributions are truncated Gaussian distributions, parameterized as follows: $\tau^y_{M,s}(b) \sim \text{TruncNorm}([10, 50], 30, 22^2),\linebreak[4] \tau^0_{I,s}(b) \sim \text{TruncNorm}([37, 79], 58, 15^2)$, and $\tau^1_{I,s}(b) \sim \text{TruncNorm}([40, 80], 60, 15^2).$ {We draw $n=1000$ total agents from the feature distributions, with the number drawn from in each label determined from the selected $\alpha_s$. For each $\alpha_s$, we run the experiment $50$ times, and report the average across these runs.} 

\paragraph{Strategic vs. non-strategic (unfair) firms.} In Figure~\ref{fig:type1_unfair_vis_all}, we explore the impacts of anticipating agents' strategic behavior (in the absence of any fairness interventions). Figure~\ref{fig:type1_comp_theta} shows that the strategic optimal thresholds are greater than the optimal non-strategic threshold at all levels of $\alpha_s$; this is consistent with Proposition~\ref{prop:firm_impact_comp}. Figure~\ref{fig:type1_comp_utility} supports this by showing that a strategic firm's utility is greater than the non-strategic one's. Lastly, consistent with parts (ii)-(iv) of Proposition~\ref{prop:firm_impact_comp}, we observe that the anticipation of agents' strategic behavior reduces the impacts of manipulation by the unqualified agents (Figure~\ref{fig:type1_comp_Psi0}) while increasing improvement by them (Figure~\ref{fig:type1_comp_Phi0}). Simultaneously, this choice leads to both more manipulation (Figure~\ref{fig:type1_comp_Psi1}) and (weakly) more improvement (Figure~\ref{fig:type1_comp_Phi1}) by the qualified agents, both of which also benefit the firm. Figure~\ref{fig:alpha_post_com_unfair} further supports these claims, by showing that post-strategic qualification rates $\hat{\alpha}_s$ are higher under the strategic firm's policy compared to the non-strategic firm's policy. In other words, strategic firms succeed at incentivizing agents to opt for improvement, with the impacts being more pronounced when the group is initially majority unqualified (low $\alpha_s$). Finally, we also note that all these observations hold even though the assumption $\overline{x}^0_s +(\mathbb{T}^y_{M,s})^{-1}(C_{I,s}-C_{M,s})  < \mu^0_s + \underline{b}^0_{I,s}$
is not satisfied in the current experiment setup, hinting that Proposition~\ref{prop:firm_impact_comp} may hold beyond its current assumptions. 

\begin{figure}[ht]
    \centering
    \begin{subfigure}{0.32\textwidth}
        \centering
        \includegraphics[width=\textwidth,trim={0 0 0 0},clip]{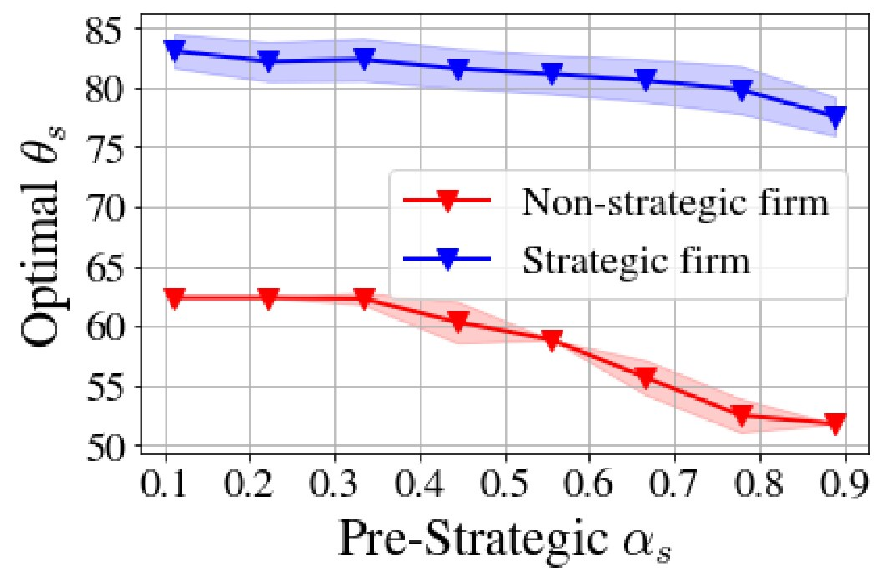}
        \caption{Optimal thresholds comparison}
        \label{fig:type1_comp_theta}
    \end{subfigure}
    \begin{subfigure}{0.32\textwidth}
        \centering
        \includegraphics[width=\textwidth,trim={0 0 0 0},clip]{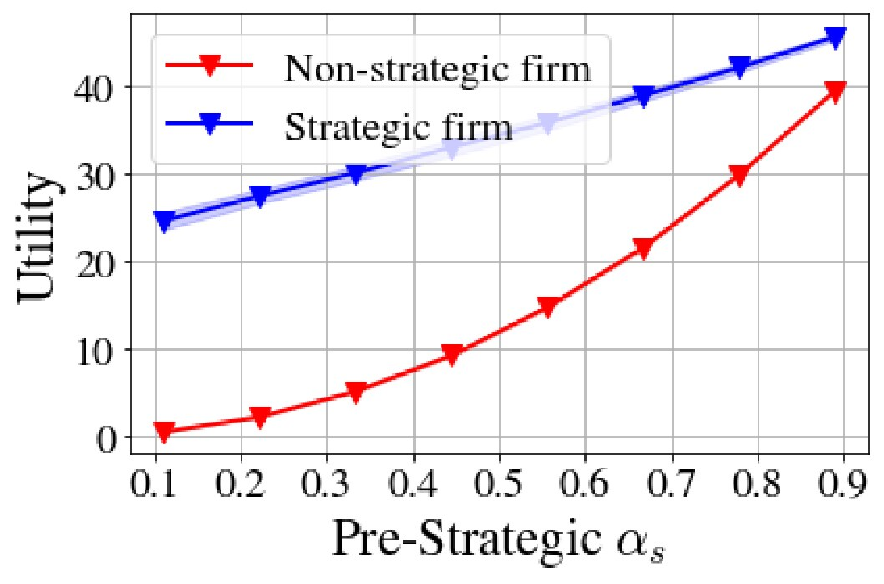}
        \caption{Optimal utility comparison}
        \label{fig:type1_comp_utility}
    \end{subfigure}
    \begin{subfigure}{0.32\textwidth}
        \centering
        \includegraphics[width=\textwidth,trim={0 0 0 0},clip]{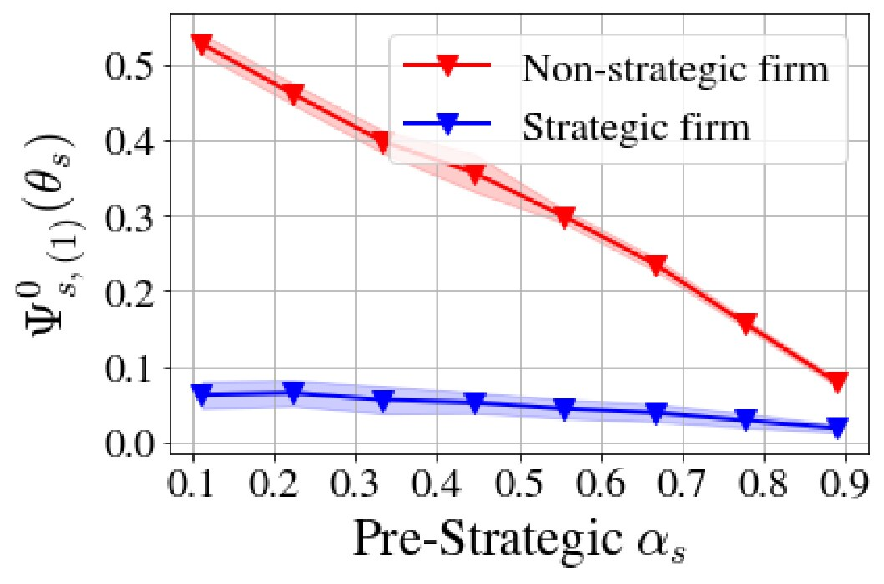}
        \caption{$\boldsymbol{\Psi}^0_{s,(1)}$ (Manipulation) comparison}
        \label{fig:type1_comp_Psi0}
    \end{subfigure}
    \begin{subfigure}{0.32\textwidth}
        \centering
        \includegraphics[width=\textwidth,trim={0 0 0 0},clip]{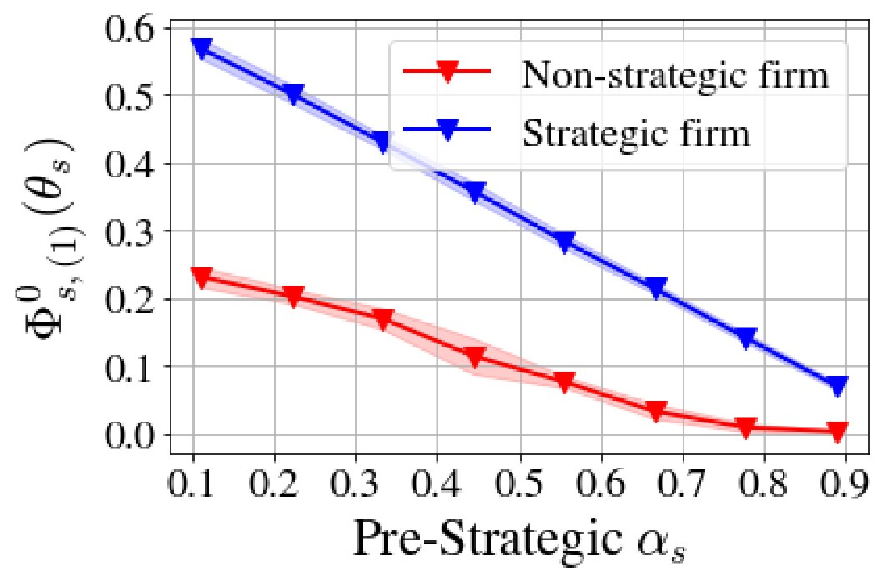}
        \caption{$\boldsymbol{\Phi}^0_{s,(1)}$ (Improvement) comparison}
        \label{fig:type1_comp_Phi0}
    \end{subfigure}
    \begin{subfigure}{0.32\textwidth}
        \centering
        \includegraphics[width=\textwidth,trim={0 0 0 0},clip]{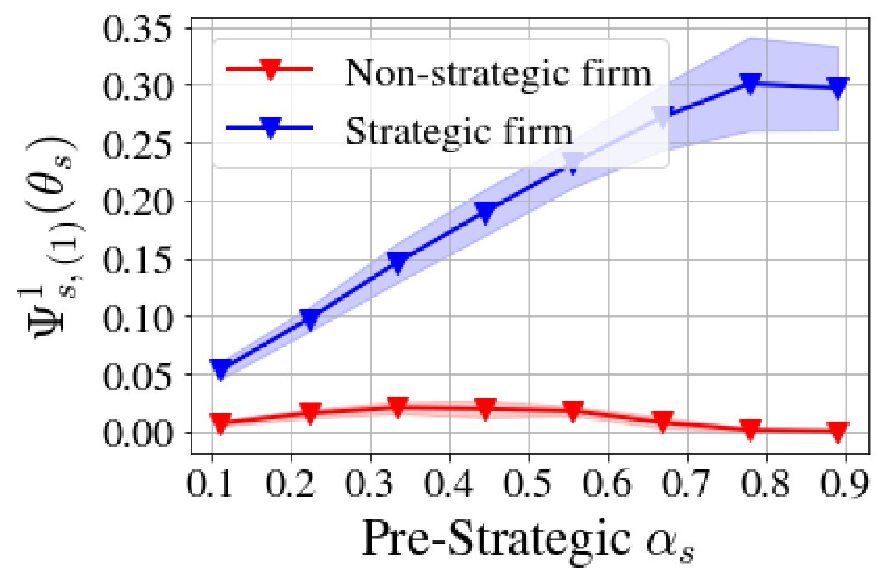}
        \caption{$\boldsymbol{\Psi}^1_{s,(1)}$ (Manipulation) comparison}
        \label{fig:type1_comp_Psi1}
    \end{subfigure}
    \begin{subfigure}{0.33\textwidth}
        \centering
        \includegraphics[width=\textwidth,trim={0 0 0 0},clip]{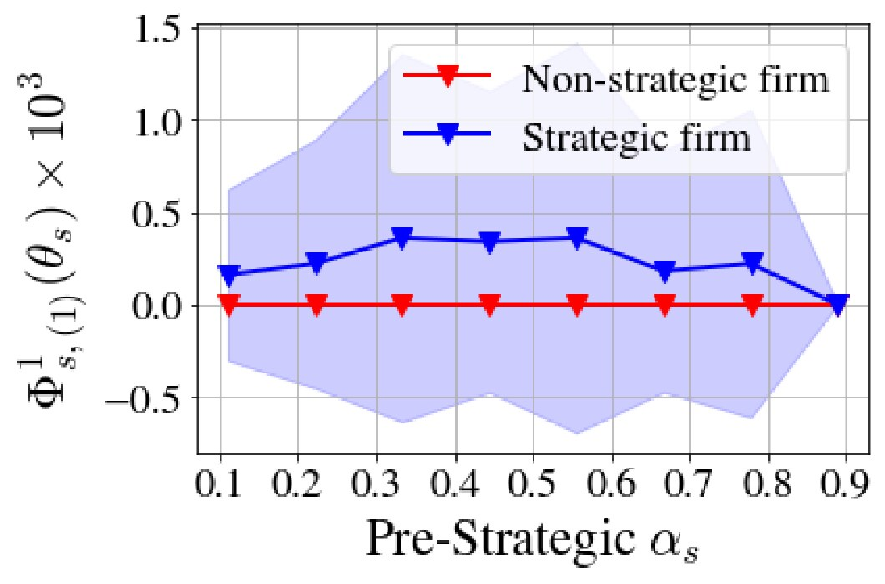}
        \caption{$\boldsymbol{\Phi}^1_{s,(1)}$ (Improvement) comparison }
        \label{fig:type1_comp_Phi1}
    \end{subfigure}
    \caption{Unfair optimal non-strategic vs. strategic policies' comparison when varying pre-strategic $\alpha_s$.}
    \label{fig:type1_unfair_vis_all}
\end{figure}
\begin{figure}[htb]
    \centering
    \includegraphics[width=0.32\linewidth]{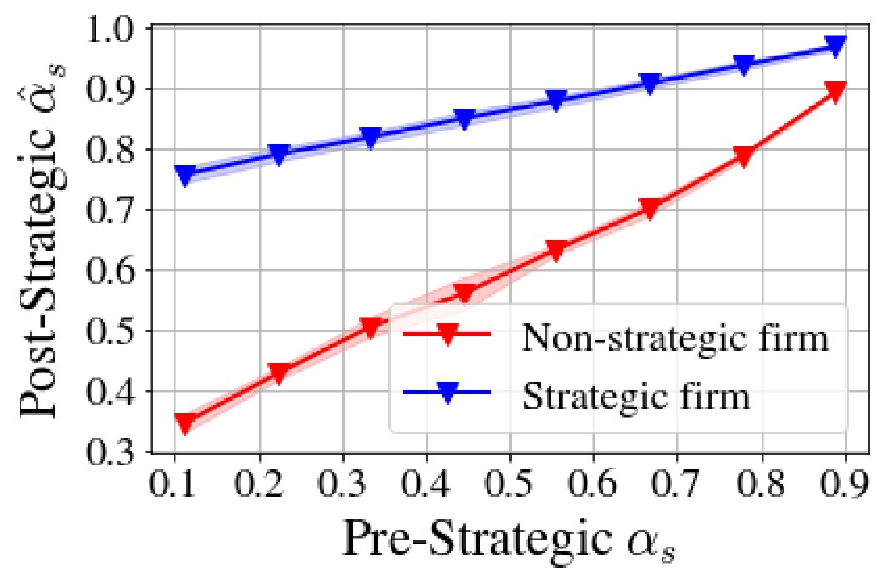}
    \caption{Unfair non-strategic vs. strategic policies' impact on post-strategic $\hat{\alpha}_s$ across pre-strategic $\alpha_s$.}
    \label{fig:alpha_post_com_unfair}
\end{figure}

\paragraph{Two-group numerical illustration setup.} We next consider a population consisting of two equal-size groups $a$ and $b$, with group $a$ being majority-qualified ($\alpha_a=0.7$) and group $b$ being majority-unqualified ($\alpha_b=0.2$). The feature distributions for both groups follow \emph{truncated} Gaussian distributions, with $G^0_a \sim \text{TruncNorm}([20, 110], 65, 15^2)$, $G^0_b \sim \text{TruncNorm}([0, 90], 45, 15^2)$, $G^1_a\sim \text{TruncNorm}([98, 188], 143, 15^2)$, and $G^1_b\sim \text{TruncNorm}([78, 168], 123, 15^2)$. We assume both groups have access to the same strategic actions, with $C_{M,s}=0.2$, $C_{I,s}=0.3$, and the boost distributions $\tau^y_{M,s}(b) \sim \text{TruncNorm}([20, 70], 45, 22^2)$, $ \tau^0_{I,s}(b) \sim \text{TruncNorm}([75, 115], 95, 15^2)$, and $\tau^1_{I,s}(b) \sim \text{TruncNorm}([72, 118], 97, 15^2)$. Note that with these assumptions, the disparities between groups are assumed to be \emph{historical} (in qualification rates and feature distributions) but not \emph{current} (in access to strategic resources).

\paragraph{Strategic vs. non-strategic fair firms.} Figure~\ref{fig:fair_vis_Type1} uses contour plots to illustrate the firms' utility, and the fairness-constrained utilities under two types of constraints (Demographic Parity and True Positive Rate parity) when varying the group thresholds $\theta_a$ and $\theta_b$. It also indicates the optimal thresholds for each case (unfair, DP-fair, and TPR-fair). Figure~\ref{fig: Type_1_fair_comp} also shows these thresholds, and the manipulation and improvement impacts on the firm's utility under each choice. It is worth noting that the strategic firm is aware of post-strategic statistics, and uses these when evaluating and imposing the fairness constraints. In contrast, due to unawareness of strategic behavior, the non-strategic firm evaluates its fairness constraints over the \emph{pre-strategic} constraints, which in turn means that its fairness desiderata are in fact \emph{not} satisfied ex-post. 

We first observe that both strategic and non-strategic firms set higher (resp. lower) thresholds on the majority-qualified (resp. majority-unqualified) group in order to meet a fairness constraint; this is consistent with existing literature and our assumption in Corollary~\ref{cor:fair-policies}. We also note that the strategic firm (whether fair or not) sets higher thresholds on both groups that the non-strategic firm; this is also consistent with  Proposition~\ref{prop:firm_impact_comp}. 

The more notable and newly observed distinction is when we compare the impact of the \emph{type} of fairness intervention (DP vs. TPR) on the difference between strategic and non-strategic firms. In particular, there is a more significant change in DP-fair thresholds when the firm is non-strategic. To see why, consider the DP-Fair policy selected by a non-strategic firm (Figure~\ref{fig:fair_Non_utlity_type1}). Such a firm does not account for agents' strategic responses, simply lowering the decision threshold on group $b$ and increasing it on group $a$, in order to satisfy the fairness constraint. 
For group $b$, this DP-Fair policy is in fact so low that it accepts many agents (by default) and accepts others with even lower features who opt for manipulation. That is why $\boldsymbol{\Phi}^0_{b,(1)}(\theta^{\mathcal{C}}_b)$ (resp. $\boldsymbol{\Psi}^0_{b,(1)}(\theta^{\mathcal{C}}_b)$) decreases (resp. increases) compared to those under the unfair non-strategic policy, as illustrated by the solid red lines in Figures~\ref{fig: Type_1_fair_comp} the second and the fourth subplots. In contrast, a strategic firm lowers its fair thresholds on group $b$ much less drastically (Figure~\ref{fig:fair_Str_utlity_type1}), as it realizes that parity between selection rates (as required by DP) can be achieved by a combination of adjusting the thresholds \emph{and} driving agents' best-responses. In fact, we can observe that fairness desiderata are achieved with very small changes to the decision thresholds when accounting for agents' strategic responses, therefore leading to far less loss in utility for the firm. 

\begin{figure}[ht]
    \centering
    \begin{subfigure}{0.45\textwidth}
        \centering
        \includegraphics[width=\textwidth,trim={0 0 0 0},clip]{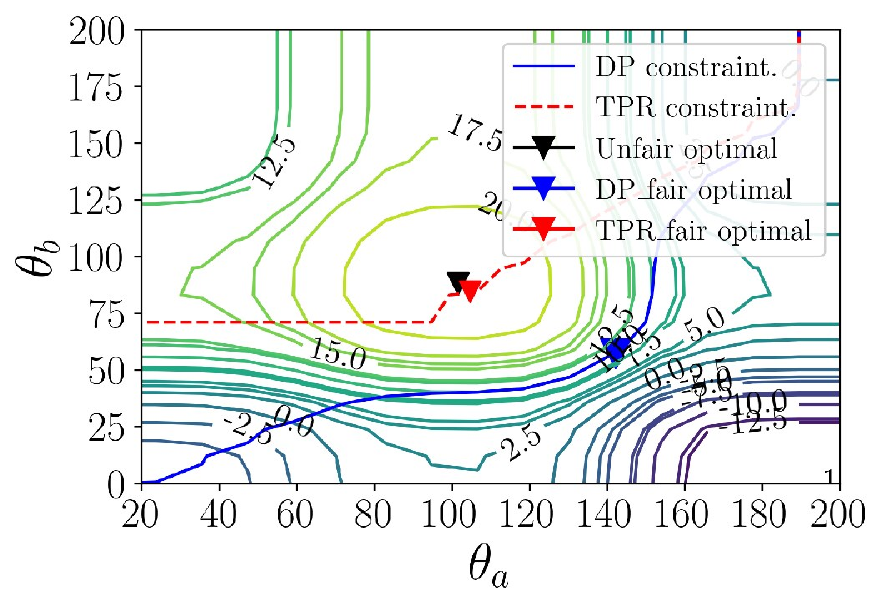}
        \caption{Non-strategic firm's utility}
        \label{fig:fair_Non_utlity_type1}
    \end{subfigure}
    \hfill
    \begin{subfigure}{0.45\textwidth}
        \centering
        \includegraphics[width=\textwidth,trim={0 0 0 0},clip]{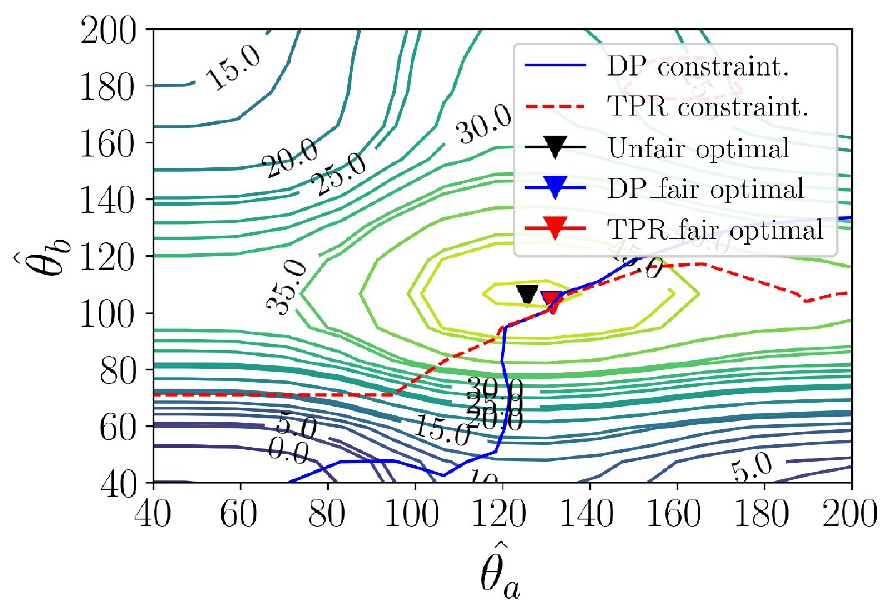}
        \caption{Strategic firm's utility}
        \label{fig:fair_Str_utlity_type1}
    \end{subfigure}
    \caption{Firm's utility and optimal thresholds when imposing fairness constraints.}
    \label{fig:fair_vis_Type1}
\end{figure} 

\begin{figure}
    \centering
    \includegraphics[width=1\linewidth]{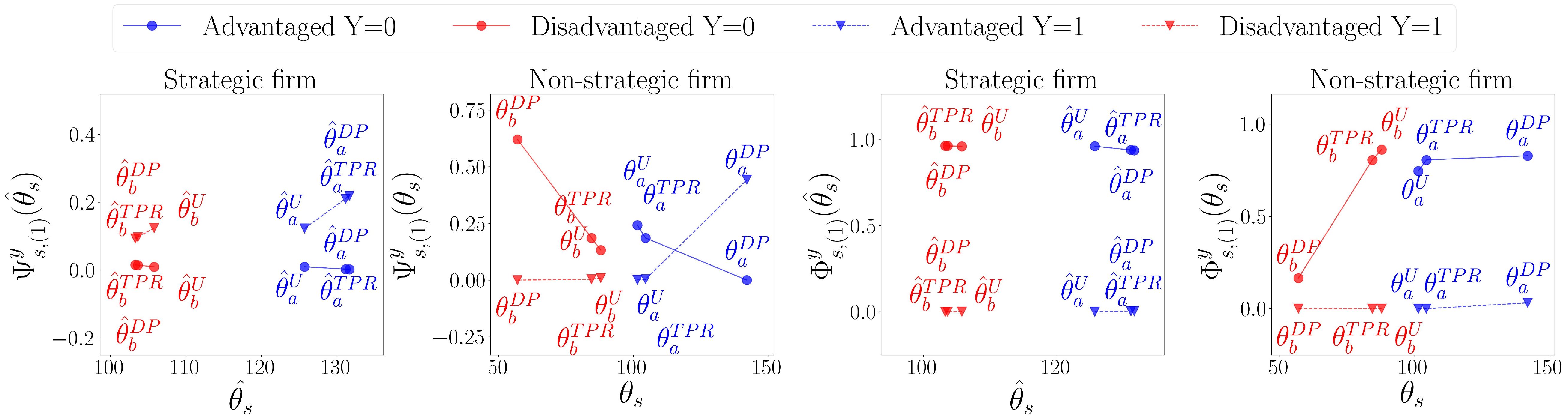}
    \caption{Strategic behavior's impact comparison: fair vs unfair policies}
    \label{fig: Type_1_fair_comp}
\end{figure}

\paragraph{Impacts on agents' post-strategic qualification rates.} To further explore the impacts of a strategic policy on agents' behavior under different fairness interventions (unfair, DP, and TPR), in Figure~\ref{fig:alpha_post_com_fair}, we contrast the post-strategic $\hat{\alpha}_s$ qualification rates for each group. 
First, as also noted earlier, we note that the strategic firm is more successful at incentivizing improvement, as evidenced by the higher $\hat{\alpha}_s$ in both groups $a$ and $b$ compared to the non-strategic firm (in Figure~\ref{fig:alpha_post_com_fair}, blue bars are consistently higher than red bars). This increase is even more significant for group $b$, as this group has a lower qualification rate to begin with. 

Perhaps more interestingly, we observe that if a firm does not account for strategic responses by agents,  fairness interventions can have a negative implication as they \emph{decrease} improvement by the disadvantaged group $b$ (in Figure~\ref{fig:alpha_post_com_fair}, the red bars are lower under the fair policies compared to the unfair policy). This is because fairness interventions typically lower the threshold for group $b$ compared to the unfair policy, which in turn drives the adoption of manipulation over improvement decisions. In comparison, a strategic firm recognizes that fairness goals can be met by a combination of lowering the threshold and incentivizing improvement actions, ultimately leading to higher improvement rates in the disadvantaged group $b$. 

\begin{figure}[ht]
    \centering
    \includegraphics[width=0.6\linewidth]{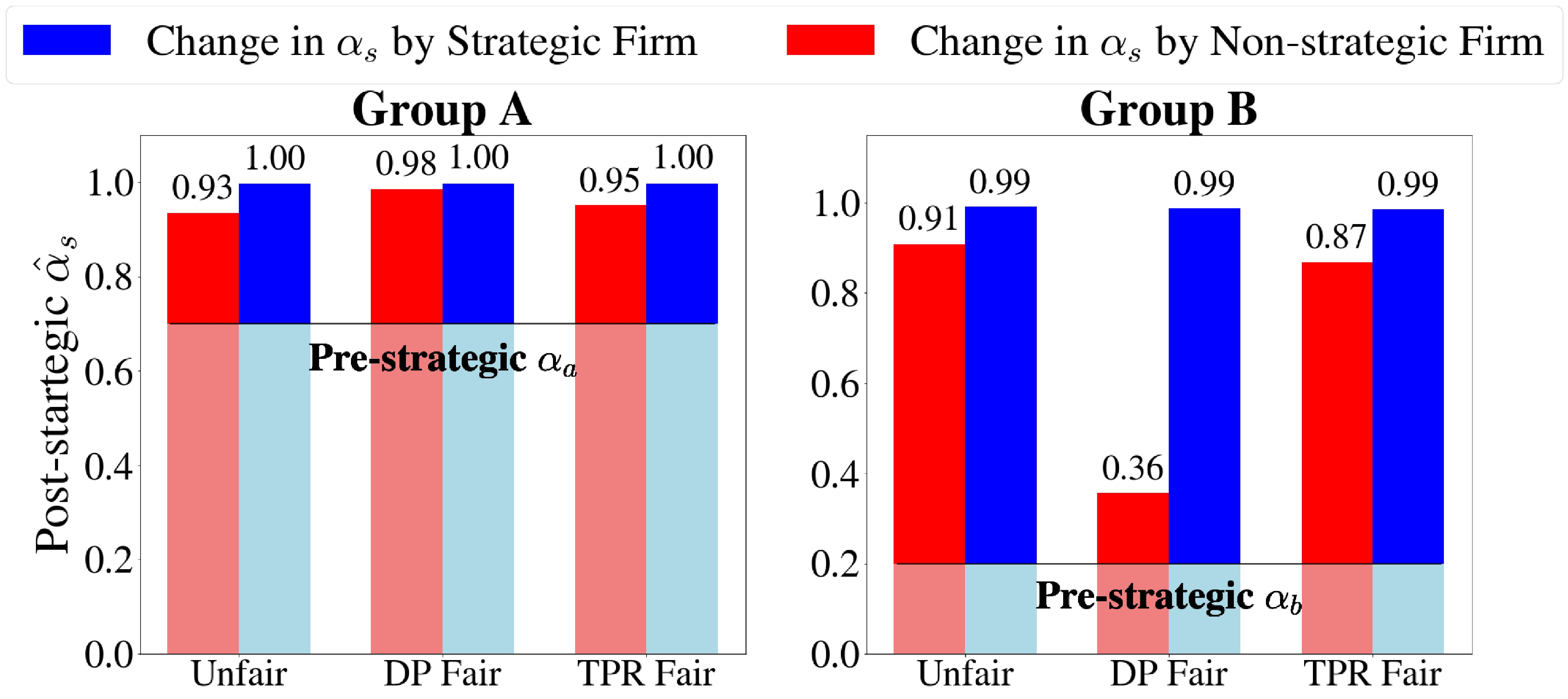}
    \caption{Fair vs. unfair, non-strategic vs. strategic policies' impact on post-strategic $\hat{\alpha}_s$.}
    \label{fig:alpha_post_com_fair}
\end{figure}

\subsection{Synthetic dataset Type 3 Experiments}\label{app:type3_experiments}
\paragraph{Single-group experiment setup.}
We model a population of agents from group $s$ with truncated Gaussian feature distributions: $G^0_s \sim \text{TruncNorm}([0, 90], 45, 15^2)$ and $G^1_s \sim \text{TruncNorm}([78, 168], 123, 15^2)$. The pre-strategic qualification rate $\alpha_s \in (0, 1)$ varies, with strategic action costs $C_{M,s}=0.2$ and $C_{I,s}=0.8$. Boost distributions are truncated Gaussians: $\tau^y_{M,s}(b) \sim \text{TruncNorm}([20, 70], 45, 22) \linebreak[4] \tau^0_{I,s}(b) \sim \text{TruncNorm}([40, 82], 61, 15)$, and $\tau^1_{I,s}(b) \sim \text{TruncNorm}([42, 82], 62, 15)$. We sample $n=1000$ agents based on $\alpha_s$, run 50 experiments per $\alpha_s$, and report averages.

\paragraph{Strategic vs. non-strategic (unfair) firms.}
In Figure~\ref{fig:all_analysis_Type3}, the leftmost picture shows strategic optimal thresholds (in the absence of any fairness interventions) exceeding non-strategic ones across all $\alpha_s$ levels, aligning with Proposition~\ref{prop:firm_impact_comp}, while the second-left picture confirms higher utility for a strategic firm. The fourth-left picture indicates reduced manipulation by unqualified agents due to anticipating strategic behavior, whereas the third-left picture shows increased manipulation by qualified agents, benefiting the firm, consistent with Proposition~\ref{prop:firm_impact_comp} parts (ii) and (iii).

\begin{figure}
    \centering
    \includegraphics[width=1\linewidth]{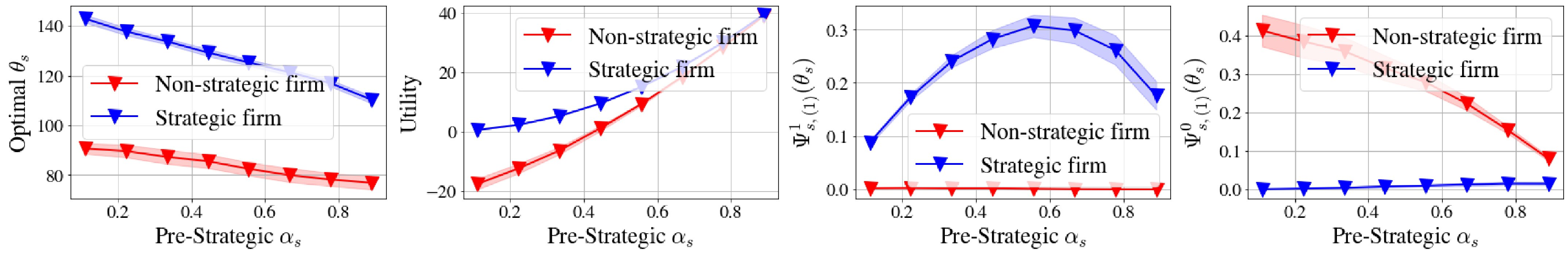}
    \caption{Proposition~\ref{prop:firm_impact_comp} under Type 3 equilibrium}
    \label{fig:all_analysis_Type3}
\end{figure} 

\paragraph{Two-group numerical illustration setup.}
We next consider a population consisting of two equal-size
groups a and b, with group a being majority-qualified ($\alpha_a$ = 0.7) and group b being majority-unqualified ($\alpha_b$ = 0.2). The feature distributions for both groups follow \emph{truncated} Gaussian distributions, with $G^0_a\sim \text{TruncNorm}([20, 110],65,15^2), G^0_b \sim \text{TruncNorm}([0, 90],45,15^2), G^1_a\sim \text{TruncNorm}([98, 188],143,15^2)$ and $G^1_b\sim \text{TruncNorm}([78, 168],123,15^2)$. We assume both groups have access to the same strategic actions, with $C_{M,s}=0.2$, $C_{I,s}=0.8$, and the same boost distribution as in the "Single-group experiment".
   
\paragraph{Strategic vs. non-strategic fair firms.}
Figure~\ref{fig:fair_vis_Type3} uses contour plots to display firm utility and fairness-constrained utilities (DP and TPR parity) as thresholds $\theta_a$ and $\theta_b$ vary, marking optimal thresholds (unfair, DP-fair, TPR-fair). Figure~\ref{fig:fair_analysis_type3} shows these thresholds and manipulation effects on utility. The strategic firm, using post-strategic statistics, enforces fairness constraints, unlike the non-strategic firm, which relies on pre-strategic constraints.

As seen in Type 1 equilibrium, both firms raise (resp. lower) thresholds for the majority-qualified (resp. majority-unqualified) group to meet fairness constraints, per Corollary~\ref{cor:fair-policies} and prior literature, with the strategic firm setting higher thresholds for both groups, consistent with Proposition~\ref{prop:firm_impact_comp}.

\begin{figure}[hbt]
    \centering
    \begin{subfigure}{0.35\textwidth}
        \centering
        \includegraphics[width=\textwidth,trim={0 0 0 0},clip]{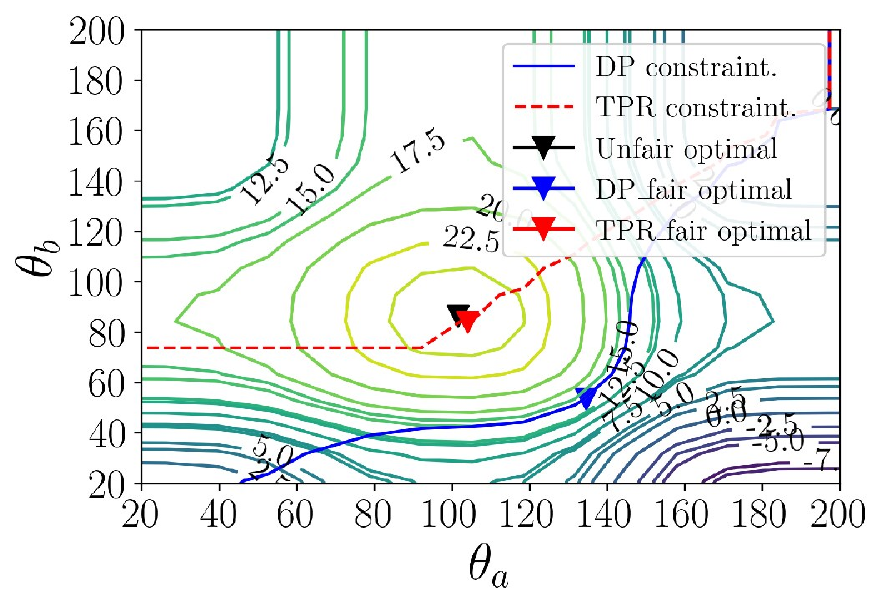}
        \caption{Nonstrategic utility and fairness constraints}
        \label{fig:fair_Non_utlity_type3}
    \end{subfigure}
    \begin{subfigure}{0.35\textwidth}
        \centering \includegraphics[width=\textwidth,trim={0 0 0 0},clip]{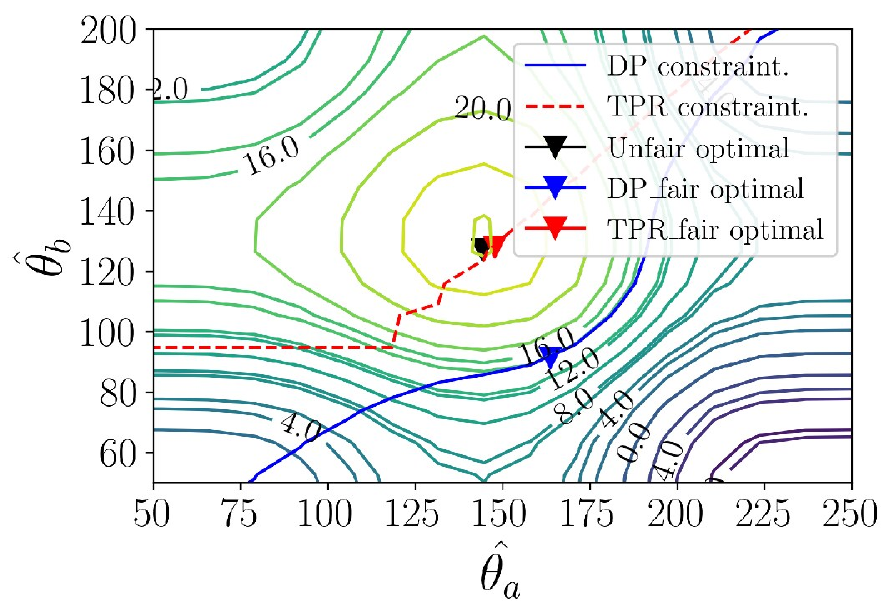}
        \caption{Strategic utility and fairness constraints}
        \label{fig:fair_Str_utlity_type3}
    \end{subfigure}
    \caption{Firm's Utility under agents' best response Type 3 with the fairness constraint}
    \label{fig:fair_vis_Type3}
\end{figure}

Unlike Type 1, comparing DP vs. TPR fairness impacts reveals greater threshold changes under DP for both firms, as seen in Figure~\ref{fig:fair_vis_Type3}. Without \emph{improvement} actions, $\alpha_s$ remains unchanged, forcing firms to lower $\theta_b$ and raise $\theta_a$ to meet DP constraints, as strategic responses can’t achieve selection parity. For group $b$, the DP-fair strategic firm increases unqualified manipulation more, raising $\boldsymbol{\Psi}^0_{b,(1)}(\hat{\theta}^{\mathcal{C}}_b)$ compared to the unfair policy’s high threshold (curb many unqualified manipulation), shown by solid red lines in Figure~\ref{fig:fair_analysis_type3}’s first plot. The non-strategic firm lowers $\theta_b$ similarly, but with less manipulation increase (Figure~\ref{fig:fair_analysis_type3}’s second plot) due to already high unqualified manipulation. Thus, without improvement actions, fairness requires larger threshold shifts under strategic responses, reducing firm utility.

\begin{figure}[tb]
    \centering
    \includegraphics[width=1\linewidth]{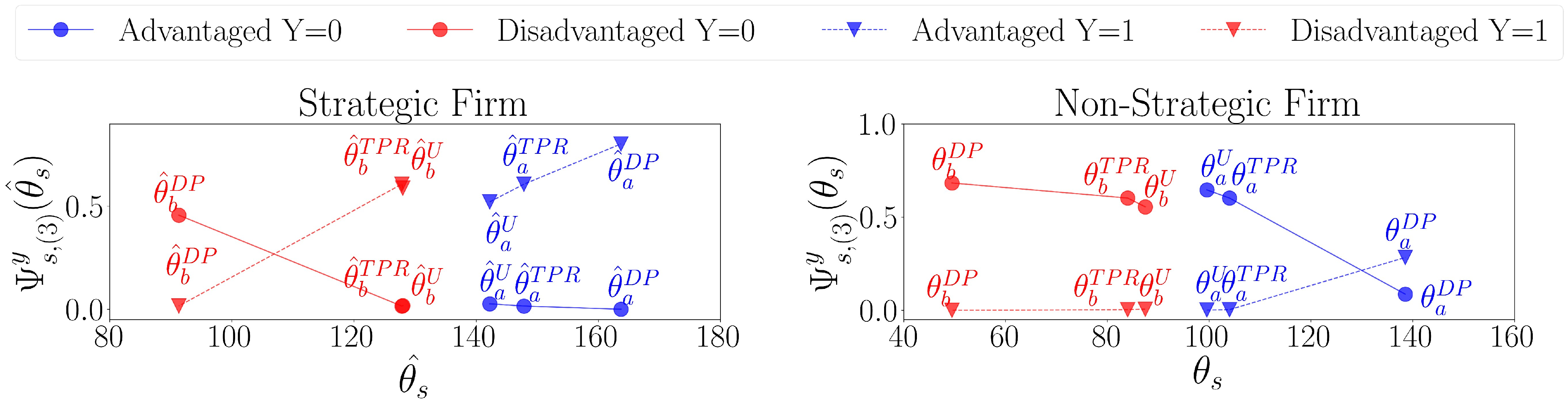}
    \caption{Strategic behavior’s impact comparison: fair vs unfair policies under Type 3 equilibrium}
    \label{fig:fair_analysis_type3}
\end{figure}

\end{document}